\documentclass{article}

\usepackage[utf8]{inputenc} %
\usepackage[T1]{fontenc}    %
\usepackage{url}            %
\usepackage{booktabs}       %
\usepackage{algorithm, algpseudocode}
\usepackage[normalem]{ulem}
\usepackage{amsfonts}       %
\usepackage{nicefrac}       %
\usepackage{microtype}      %
\usepackage[table]{xcolor}
\usepackage{multirow}
\usepackage{enumitem}
\usepackage{arydshln}
\usepackage{rotating}
\usepackage{siunitx}
\usepackage{caption}
\usepackage{makecell}

\newcommand{\valunc}[2]{%
  \textnormal{#1}%
  \if\relax\detokenize{#2}\relax
  \else
    {\scriptsize$\pm$\textnormal{#2}}%
  \fi
}

\newcommand{\texthl}[2]{%
  \begingroup
    \setlength{\fboxsep}{0pt}%
    \colorbox{#1}{\strut #2}%
  \endgroup
}

\usepackage{amsmath}
\allowdisplaybreaks %
\usepackage{amssymb}
\usepackage{mathtools}
\usepackage{amsthm}
\usepackage{nicefrac}
\usepackage{thmtools}
\usepackage{bbm}
\usepackage{cancel}
\usepackage{centernot}
\usepackage{siunitx}
\usepackage{tikz}
\usetikzlibrary{snakes,arrows,shapes}

\usepackage{xspace}
\makeatletter
\DeclareRobustCommand\onedot{\futurelet\@let@token\@onedot}
\def\@onedot{\ifx\@let@token.\else.\null\fi\xspace}

\makeatother

\declaretheorem[name=Theorem]{theorem}

\declaretheorem[name=Lemma]{lemma}
\declaretheorem[name=Definition]{definition}
\declaretheorem[name=Example]{example}
\declaretheorem[name=Proposition]{proposition}
\declaretheorem[name=Remark,style=remark]{remark}

\newcommand{\E}[2]{\mathbb{E}_{#1}{\left[#2\right]}}

\newcommand{\R}{\mathbb{R}}
\newcommand{\C}{\mathbb{C}}
\newcommand{\M}{\mathcal{M}}
\renewcommand{\S}{\mathcal{S}}
\newcommand{\A}{\mathcal{A}}

\newcommand*{\argmax}{\mathop{\mathrm{argmax}}}
\newcommand{\defeq}{\mathrel{\mathop:}=}
\newcommand{\LHS}{\mathrm{LHS}}
\newcommand{\RHS}{\mathrm{RHS}}

\newcommand{\D}{\mathcal{D}}
\newcommand{\rmax}{r_{\max}}
\renewcommand{\Pr}{\mathbb P}
\newcommand{\supp}{\mathrm{supp}}

\newcommand{\Hc}{\mathcal{H}}

\newcommand{\Mf}{\mathfrak{M}}

\newcommand{\term}{f_\mathrm{term}}
\newcommand{\ensemble}{\mathbf{m}_{\boldsymbol{\theta}}}
\newcommand{\unc}{U_{\boldsymbol{\theta}}}
\newcommand{\uncthres}{\mathcal U(\zeta)}

\usepackage{wrapfig}
\definecolor{citecolor}{rgb}{0.3,0.3,0.3}
\definecolor{linkcolor}{rgb}{0,0,0.5}
\usepackage[backref=page,colorlinks=true,citecolor=citecolor,linkcolor=linkcolor]{hyperref}
\usepackage[capitalize,noabbrev]{cleveref}

\usepackage[accepted]{icml2026}

\icmltitlerunning{Long-Horizon Model-Based Offline Reinforcement Learning Without Explicit Conservatism}

\begin{document}

\newcommand{\algo}{\textsc{Neubay}\xspace}

\twocolumn[
  \icmltitle{Long-Horizon Model-Based Offline Reinforcement Learning \\Without Explicit Conservatism}

  \icmlsetsymbol{equal}{*}

  \begin{icmlauthorlist}
    \icmlauthor{Tianwei Ni}{mila,udem}
    \icmlauthor{Esther Derman}{mila,udem}
    \icmlauthor{Vineet Jain}{mila,mcgill}
    \icmlauthor{Vincent Taboga}{mila,udem}\\
    \icmlauthor{Siamak Ravanbakhsh}{mila,mcgill}
    \icmlauthor{Pierre-Luc Bacon}{mila,udem}
  \end{icmlauthorlist}

  \icmlaffiliation{mila}{Mila - Quebec AI Institute}
  \icmlaffiliation{udem}{Université de Montréal}
  \icmlaffiliation{mcgill}{McGill University}

  \icmlcorrespondingauthor{Tianwei Ni}{twni2016@gmail.com}

  \icmlkeywords{Machine Learning, ICML}

  \vskip 0.3in
]

\printAffiliationsAndNotice{}  %

\begin{abstract}
Popular offline reinforcement learning (RL) methods rely on \textit{explicit conservatism}, penalizing out-of-dataset actions or restricting rollout horizons. We question the universality of this principle and revisit a complementary \textit{Bayesian} perspective for test-time adaptation. By modeling a posterior over world models and training a history-dependent agent to maximize expected return, the Bayesian approach directly addresses epistemic uncertainty without explicit conservatism. We first illustrate in a bandit setting that Bayesianism excels on low-quality datasets where conservatism fails. Scaling to realistic tasks, we find that \textit{long-horizon rollouts} are essential to control value overestimation once conservatism is removed. We introduce design choices that enable learning from long-horizon rollouts while mitigating compounding model errors, yielding our algorithm, \algo, grounded in the \textsc{neu}tral \textsc{Bay}esian principle. On D4RL and NeoRL benchmarks, \algo is competitive with leading conservative algorithms, achieving new state-of-the-art on 7 datasets with rollout horizons of several hundred steps. Finally, we characterize datasets by quality and coverage to identify when \algo is preferable to conservative methods. Code is available at \href{https://github.com/twni2016/neubay}{github.com/twni2016/neubay}.
\end{abstract}

\vspace{-1em}
\section{Introduction}
\label{sec:introduction}

\begin{figure}[h]
    \centering
    \includegraphics[width=\linewidth]{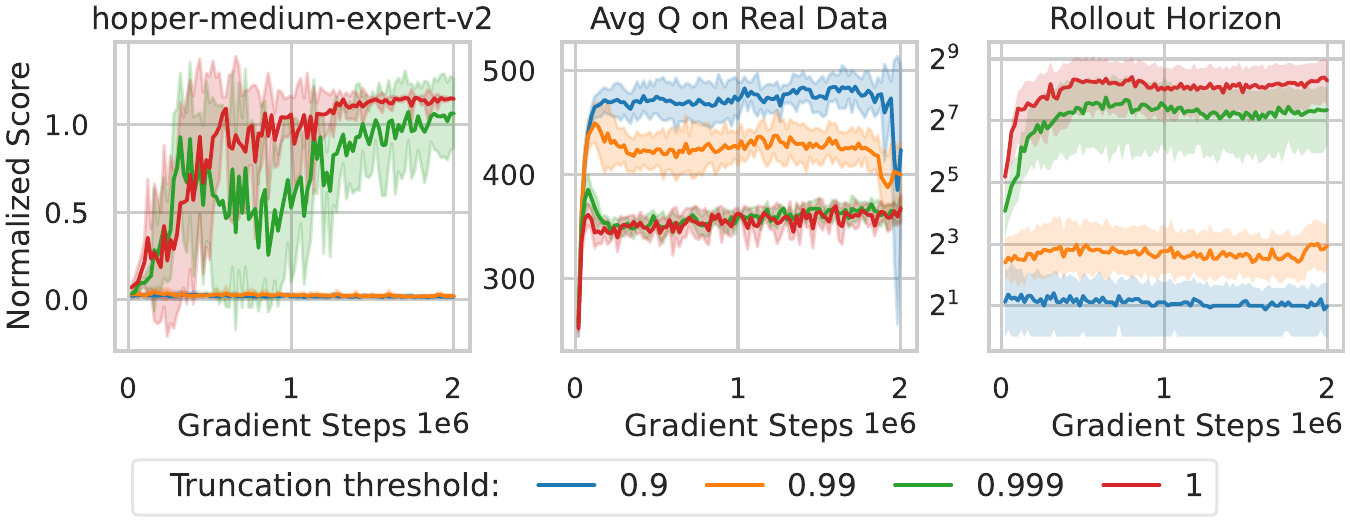}
    \vspace{-1.5em}
    \caption{\footnotesize\textbf{Adaptive long-horizon rollouts} improve performance (left) with lower estimated Q-value on the offline dataset (middle), \textbf{without explicit conservatism}. We vary the rollout truncation threshold $\zeta \in \{0.9, 0.99, 0.999, 1.0\}$ and report horizons (median with 25\%--75\% band) over 100 training-time rollouts (right). \textbf{Full results} are shown in \autoref{app:unc_thres}.\looseness=-1}
    \vspace{-1em}
    \label{fig:horizon_main}
\end{figure}

Reinforcement learning (RL) often assumes direct interaction with the environment, which we refer to as online RL~\citep{sutton2018reinforcement}. While successful in simulation, deploying it in real-world settings such as robotics~\citep{dulac2021challenges}, recommendation systems~\citep{chen2024opportunities}, and language reasoning~\citep{ma2026learning} is limited by expensive or risky data collection. 
A more practical alternative is offline RL~\citep{lange2012batch}, which learns from pre-collected datasets (e.g., from human demonstrations or prior agents) without further environment access~\citep{levine2020offline, fu2020d4rl, gulcehre2020rl}. This decoupling enables safe, scalable training and the potential to outperform the behavior policies that produced the data.\looseness=-1

Most offline RL algorithms adopt a conservative principle\footnote{In this paper, ``explicit conservatism'' refers to the deliberate injection of pessimism into offline policy learning. Unless otherwise noted, we use ``conservatism'' as shorthand for this notion. This excludes generic anti-overestimation methods proposed in the online RL literature.}
by penalizing the policy and value function on out-of-dataset state-action pairs~\citep{levine2020offline,prudencio2023survey} and using short rollout horizons~\citep{lu2021revisiting}.
In theory, these algorithms enforce robustness, either strict~\citep{jin2021pessimism,uehara2021pessimistic} or soft~\citep{zhang2024soft}, over the uncertainty set of possible MDPs consistent with the dataset. 
The trade-off is clear: conservatism reduces value overestimation and unsafe extrapolation, but can also suppress average-case performance and limit adaptation, since policies are discouraged from exploring potentially high-reward but underrepresented actions.\looseness=-1

Bayesian RL optimizes average-case performance under epistemic uncertainty~\citep{duff2002optimal}. Its application to offline RL was pioneered by \citet{ghosh2022offline}, who formalized the problem as an \textit{epistemic POMDP}, where partial observability arises from limited coverage. This formulation enables test-time adaptation through history-dependent Bayes-optimal policies.
In this work, we first revisit and extend this Bayesian principle through the lens of \emph{data quality}. Using a two-armed bandit with skewed data, we show that conservative algorithms, with sufficient uncertainty penalty, are guaranteed to commit to the seen arm regardless of test-time conditions. In contrast, Bayesian algorithms can adaptively explore and commit to the better arm at test time, a clear advantage in low-quality datasets.\looseness=-1

However, scaling the Bayesian principle to realistic tasks is challenging, as it requires solving an \emph{approximate} epistemic POMDP. 
We identify three key challenges: (1) \textit{value overestimation}~\citep{fujimoto2019off}, which becomes central once Bayesian RL abandons explicit conservatism and removes penalties for out-of-dataset actions; (2) \textit{compounding error} in world models~\citep{lambert2022investigating}, where inaccuracies grow rapidly with horizon; and (3) training agents with \textit{long-term memory}~\citep{ni2023transformers} to enable test-time adaptation.
Each aligns with a major research area: offline RL, model-based RL, and partially observable RL.
This helps explain why prior Bayesian-inspired algorithms often \textit{reintroduce} explicit conservatism through uncertainty penalties and short horizons~\citep{chen2021offline,jeong2022conservative} or reduce to model-free RL~\citep{ghosh2022offline}.\looseness=-1

We build a practical algorithm, \algo, to address these challenges. We first show that \textit{long-horizon rollouts} can themselves reduce value overestimation once explicit conservatism is removed. We then introduce design choices that make such long-horizon rollouts feasible. To control compounding errors, we truncate rollouts adaptively using epistemic uncertainty as a threshold~\citep{frauenknecht2024trust,zhan2021model}, an alternative use of uncertainty beyond penalties, and apply layer normalization~\citep{ba2016layer} within the world model. To stabilize learning from long horizons, we leverage recent advances in online recurrent RL with long-term memory~\citep{morad2024recurrent,luo2024efficient}.\looseness=-1

We evaluate \algo on the D4RL~\citep{fu2020d4rl} and NeoRL~\citep{qin2022neorl} benchmarks, covering 33 datasets. Overall, \algo is on par with leading conservative algorithms and outperforms other Bayesian-inspired methods, establishing new state-of-the-art results on 7 datasets. On realistic tasks, \algo performs best on low-quality datasets and on medium-quality datasets with moderate coverage. 
Our sensitivity study validates a key insight: \textbf{adaptive long horizon} is a primary driver of \algo's success. Whereas dominant model-based RL practice favors short horizons, our Bayesian approach uncovers a new role for long horizons: they suppress value overestimation. \algo routinely plans \textbf{64-512 steps} (e.g., \autoref{fig:horizon_main}), while short-horizon variants fail due to severe overestimation.
These results position \algo as a practical direction for model-based and offline RL from a Bayesian perspective, and future advances in world modeling can push the limits further.\looseness=-1

\vspace{-0.5em}
\section{Background on Offline RL}
\label{sec:prelim}
\vspace{-0.5em}

In the standard offline RL setting, a static dataset $\D$ is collected by interacting with an MDP $\M^*$, which we refer to as the \textit{true MDP}.
Formally, $\mathcal M^* = (\S, \A,\gamma, T,\term, m^*, \rho)$ where $\S$ and $\A$ are state and action spaces, $\gamma \in (0,1)$ is the discount factor, $T\in \mathbb N$ is the maximum episode length, and $\term: \S\times\A\times\S \to \{0,1\}$ is the terminal function. 
These components are assumed to be \textbf{known}~\citep{puterman2014markov, yu2020mopo}. 
The \textit{joint reward-transition function} is $m^*: \S \times \A \to \Delta( [-\rmax, \rmax]\times\S )$, consisting of reward function $R^*: \S \times \A \to \Delta([-\rmax, \rmax])$ (here, $\rmax$ is a positive constant) and dynamics $P^*: \S \times \A \to \Delta(\S)$, both \textbf{unknown} to the agent and learned in model-based methods.
The initial state distribution $\rho \in \Delta(\S)$ is also unknown, but we do not model it explicitly since initial states can be directly sampled from $\D$~\citep{janner2019trust}.\looseness=-1

The static dataset of trajectories\footnote{Although offline RL datasets are often stored as transitions, trajectories are available in common benchmarks and used by history-dependent methods~\citep{chen2021decision}.} $\D = \{\tau^i\}_{i=1}^{\mathrm{num\_traj}}$ is collected by an unknown (possibly) history-dependent behavior policy $\pi_\beta:\mathcal H_t \to \Delta(\A)$, where $\mathcal H_t$ is the space of state-action-reward sequences up to timestep $t$. Define $h_t = (s_{0:t},a_{0:t-1},r_{1:t})\in \mathcal H_t$ for $t\ge 1$ with the convention that $h_0 = s_0$.
Each trajectory $\tau = (s_0,a_0,r_1,d_1, s_1,a_1,\dots)$ is generated by:
$$s_0\sim \rho,a_t\sim \pi_\beta(h_t), (r_{t+1}, s_{t+1})\sim m^*(s_t,a_t),$$   
and $d_{t+1}=\term(s_t,a_t,s_{t+1})$. A trajectory ends either when $d_t = 1$ (termination) or when $t = T$ (truncation).
We stress this distinction: \textit{termination} implies absorbing states with zero future rewards, whereas \textit{truncation} preserves continuation and thus allows bootstrapping.\looseness=-1

\textit{The ideal objective} in offline RL is to find a possibly history-dependent policy $\pi: \mathcal H_t \to \Delta(\A)$ that maximizes the expected discounted return under the true MDP $\mathcal M^*$: $$\max_\pi J(\pi, m^*) \defeq \E{}{\sum_{t=0}^\infty \gamma^t r_{t+1}},$$where $s_0\sim \rho,  a_t \sim \pi(h_t), (r_{t+1},s_{t+1})\sim m^*( s_t,a_t)$. The defining constraint of offline RL is that the agent cannot interact with $m^*$, making it intractable to direct optimize. This leads to the following discussion about epistemic uncertainty on $m^*$.\looseness=-1 

\textbf{Empirical model and epistemic uncertainty.} From the agent's view, knowledge of $m^*$ is well-defined only on the state-action support of the dataset $\D$: $\supp_{\S\times \A}(\D) \defeq \{(s,a) \mid (s,a,r,s')\in \D\}$. Let $\Mf_{\text{in}}$ denote a model class whose domain is restricted to $\supp_{\S\times \A}(\D)$. 
The \emph{empirical model}~\citep{fujimoto2019off} is then obtained by maximum likelihood estimation (MLE):
$$m_{\D} = \argmax_{m \in\Mf_{\text{in}}} \, \E{(s,a,r,s')\sim \D}{\log m(r, s' \mid s,a)}.$$
Thus, $m_\D$ is uniquely determined in-support by empirical frequencies, but remains \emph{undefined} for $(s,a)\notin\supp_{\S\times\A}(\D)$, giving rise to substantial epistemic uncertainty~\citep{gal2016uncertainty}. 
A common way to formalize this uncertainty is through an \emph{uncertainty set} $\Mf_\D$, the set of plausible models on $\S\times\A$ that agree with $\D$ on $\supp_{\S\times\A}(\D)$ (i.e, close to $m_\D$).
Offline policy learning then incorporates $\Mf_\D$ into optimization, typically via two paradigms: (explicit) \emph{conservatism}, which optimizes against (near) worst-case models in $\Mf_\D$, or \emph{Bayesianism}, which leverages a posterior distribution over $\Mf_\D$.\looseness=-1

\begin{figure*}[t]
    \centering
 \includegraphics[width=\linewidth]{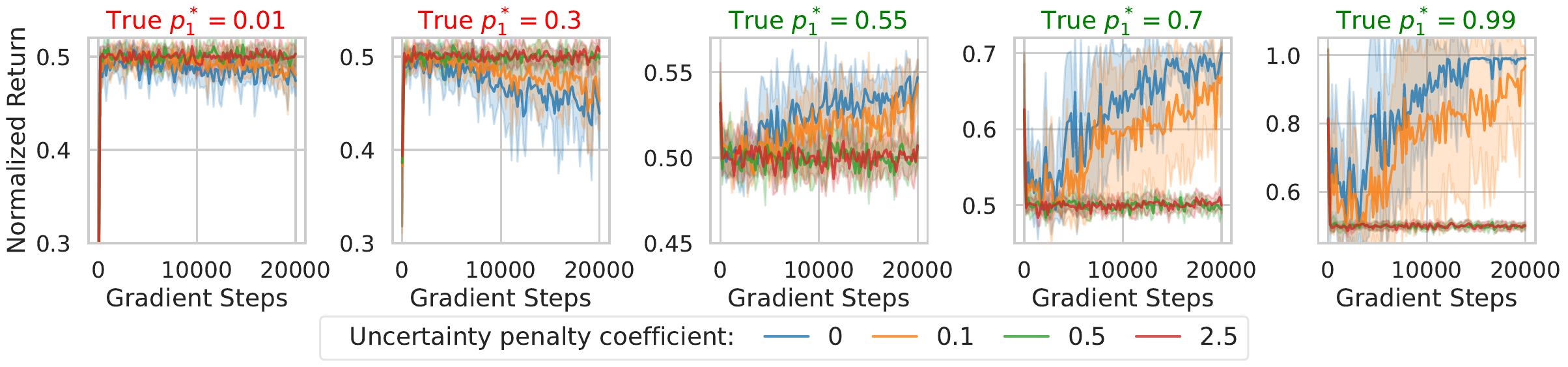}
 \vspace{-1.5em}
    \caption{\footnotesize Average return (normalized by $T$) on \textbf{test-time bandits} with different $p^*_1$. Since the observed arm has $p^*_0=0.5$, cases with $p^*_1<0.5$ are \textit{worse} and those with $p^*_1>0.5$ are \textit{better}. Zero penalty: Bayesian agent; nonzero penalty: conservative agent.}
    \label{fig:bandit}
    \vspace{-1.5em}
\end{figure*}

\textbf{Conservative principle: robust MDPs.}  Conservative RL methods commonly optimize return under the \textit{worst} model of the uncertainty set $\Mf_\D$~\citep{uehara2021pessimistic}:
\begin{equation}
\label{eq:robust_obj}
\max_\pi J(\pi; \Mf_\D) \defeq \max_\pi \min_{m\in \Mf_\D} J(\pi, m).
\end{equation}
This robust MDP formulation~\citep{wiesemann2013robust} covers both model-free~\citep{fujimoto2019off} and model-based algorithms~\citep{rigter2022rambo,yu2020mopo}, differing only in the choice of uncertainty set $\Mf_\D$ and the degree of robustness~\citep{zhang2024soft}. We provide formal connections between conservatism and robustness in \autoref{app:bg_conservatism}.\looseness=-1

\textbf{Bayesian principle: epistemic POMDPs.} 
Bayesianism instead treats the true model $m^*$ as a random variable~\citep{ghavamzadeh2015bayesian}, maintaining a posterior distribution $\Pr_\D(m)$ over plausible models. Bayesian offline RL~\citep{ghosh2022offline,uehara2021pessimistic,fellows2025sorel} then optimizes the \emph{expected} return under this posterior, known as \emph{ambiguity-neutrality}~\citep{ellsberg1961risk}:\looseness=-1
\begin{align}
\label{eq:bayes_obj}
\max_\pi J(\pi;\Pr_\D) \defeq  \max_\pi \E{m \sim \Pr_{\D}}{J(\pi, m)}. 
\end{align}
The posterior naturally induces an uncertainty set via its support. If $\Mf_\D = \supp(\Pr_\D)$, the Bayesian objective (\autoref{eq:bayes_obj}) is less pessimistic than the conservative one (\autoref{eq:robust_obj}), since for all policies $\pi$, $J(\pi; \Pr_\D) \ge J(\pi; \Mf_\D)$. 
\autoref{eq:bayes_obj} is equivalent to solving a Bayes-adaptive MDP (BAMDP)~\citep{duff2002optimal}, also known as an epistemic POMDP; see \autoref{app:bg_bayesian} for a connection with POMDPs.\looseness=-1

\textbf{Approximate posterior and uncertainty quantifier.} To instantiate the Bayesian objective (\autoref{eq:bayes_obj}), we require a tractable approximation
$\hat \Pr_\D$ to the posterior over models.  
Following common practice~\citep{chua2018deep,yu2020mopo}, we approximate $\Pr_\D$ with a finite ensemble of learned models 
$\ensemble=\{m_{\theta^n}\}_{n=1}^N$. 
The ensemble induces the empirical posterior
$
\hat \Pr_\D(m)=\frac{1}{N}\sum_{n=1}^N \mathbbm{1}(m=m_{\theta^n})
$ that captures modes of true posterior~\citep{fort2019deep,wilson2020bayesian}.
Given a state-action pair $(s,a)$, we adopt the disagreement among the \textit{mean} predictions~\citep{lakshminarayanan2017simple} as uncertainty quantifier, 
$$
\unc(s,a) = \mathrm{std}(\{\mu_{\theta^{n}}(s,a)\}_{n=1}^N),   
$$which is shown to be empirically robust under dataset shift~\citep{ovadia2019can}. We defer architectural and training details of the ensemble to \autoref{app:ensemble}.\looseness=-1

\vspace{-0.5em}
\section{Illustrative Example for Bayesianism}
\label{sec:bandit}
\vspace{-0.5em}

\textbf{The bandit dataset.} We consider a sequential two-armed bandit with $\A = \{0,1\}$. Each arm $a \in \A$ yields a Bernoulli reward $r\sim R^*(a)=\mathcal{B}(p^*_a)$, so the true MDP $m^*$ is specified by reward parameters $p^* \in [0,1]\times [0,1]$. We fix $p^*_0 = 0.5$ in our setup.
The optimal policy $\pi^*$ in $m^*$ is deterministic and memoryless, always choosing $\argmax_{a\in \A} p^*_{a}$.
To highlight out-of-support challenges, we construct a dataset that only covers arm $0$. Specifically, the dataset  $\D = \{(a_{0:T-1}^i,r_{1:T}^i)\}_{i}$ is collected by a deterministic behavior policy $\pi_\beta(a) = \mathbbm 1 (a =0)$. 
Thus, for each $t<T=100$, $a_t = 0$, $r_{t+1}\sim \mathcal{B}(p^*_{a_t})$, and $\D$ contains no data on $a=1$.\looseness=-1

\begin{figure}[H]
\vspace{0.5em}
\centering
\includegraphics[width=0.5\linewidth]{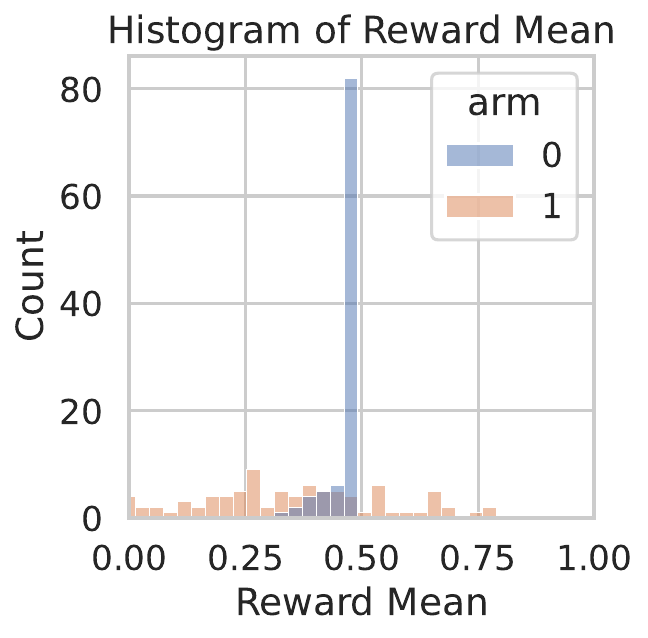}
\vspace{-0.5em}
\caption{\footnotesize Histogram of estimated reward means $p_0,p_1$ across ensemble members.\looseness=-1}
\label{fig:reward_ensemble}
\end{figure}

This skewed dataset $\D$ leaves the true reward parameter $p^*_1$ for arm $1$ \textit{completely} unobserved, inducing substantial epistemic uncertainty on $p^*_1$.
We approximate the posterior by fitting an ensemble of reward models $\ensemble$ using Gaussian outputs.
As shown in \autoref{fig:reward_ensemble}, after training on $\D$, the ensemble predictions concentrate around the observed arm $0$ but sharply disagree on the unobserved arm $1$, with estimated uncertainty about $10\times$ larger than on arm $0$.
Next, we test the Bayesian method from \autoref{eq:bayes_obj}, where planning begins at $t=0$ and truncates at $t=T$. Dynamics are not modeled due to the bandit structure, and there is no concern about compounding error.\looseness=-1

\textbf{Uncertainty penalty hurts generalization, but Bayesian agents adapt to worse \emph{and} better test cases.} 
We study explicit conservatism by penalizing the reward with an uncertainty-based term $\lambda\, \unc(a)$. Such a design is common in empirical algorithms~\citep{yu2020mopo}, and can be viewed as  soft robustness~\citep{zhang2024soft} (see \autoref{app:bg_conservatism}).
Since the estimated uncertainty on arm $1$ is much larger than arm $0$ (i.e, $\unc(1) \gg \unc(0)$), a sufficiently large coefficient $\lambda$ forces the resulting algorithm to always select arm $0$, regardless of the test-time environment.
In contrast, the Bayesian agent ($\lambda=0$) trains on a posterior $\Pr_\D(p_1)$ that spans both worse and better values relative to $p^*_0$ (see \autoref{fig:reward_ensemble}).
It adopts a history-dependent policy that explores the unseen arm at test time before committing, similarly to meta-RL~\citep{zintgraf2020varibad}. 
Thus, as shown in \autoref{fig:bandit}, the Bayesian agent’s return is slightly worse than conservatism when $p^*_1 < p^*_0=0.5$, because it briefly explores arm $1$. When $p^*_1 > 0.5$, Bayesianism is a clear win, as it identifies and exploits the better unseen action, whereas heavy conservatism remains stuck on arm $0$.
This leads to our main insight: \textit{Bayesianism excels on \textbf{low-quality} datasets or when optimal actions are scarce by leveraging test-time adaptation, while remaining competitive on \textbf{high-quality} datasets at the cost of exploration}.\looseness=-1

We theoretically confirm this empirical insight in \autoref{app:theory} by comparing robust-optimal and Bayes-optimal policies under the true MDP. For skewed datasets such as the bandits above, we show that Bayes-optimal policies outperform conservative ones by a provable margin. This gap is characterized by two data-quality measures: the Bayesian- and robust-coverage ratios. The robust-coverage ratio can grow arbitrarily large, while posterior averaging bounds the Bayesian-coverage ratio, explaining the performance difference.\looseness=-1

\vspace{-0.5em}
\section{A Practical Bayesian-Principled Algorithm}
\label{sec:algorithm}
\vspace{-0.5em}

The Bayesian agent is conceptually simple and, as seen in the bandit example, clearly outperforms conservative agents on low-quality data and enables test-time adaptation. Extending this principle from bandits to realistic MDPs, however, is challenging:
the full-horizon rollouts implied by the Bayesian objective~(\autoref{eq:bayes_obj}) can suffer from severe compounding model errors~\citep{luo2024reward, lin2025any}.
While this may suggest restricting to short-horizon rollouts, we first show that \textbf{long-horizon rollouts} are in fact necessary for Bayesian agents, as longer horizons reduce value
overestimation, a central challenge once explicit conservatism is removed (\autoref{sec:why_longhorizon}).
We then introduce a set of design choices for both performing and learning from long-horizon rollouts (\autoref{sec:how_longhorizon}--\autoref{sec:recurrent_rl}), yielding our practical algorithm, \algo (\autoref{sec:overall_algorithm}).\looseness=-1

\vspace{-0.5em}
\subsection{Why Do We Need Long-Horizon Rollouts?}
\label{sec:why_longhorizon}
\vspace{-0.5em}

Ideally, under a correct posterior, the Bayesian principle calls for full-horizon rollouts. In practice, however, compounding errors make full-horizon rollouts unreliable. 
To mitigate this issue, prior work in offline RL typically adopts a \textit{combination} of short-horizon rollouts and conservatism: short horizons limit compounding errors, while explicit conservatism (e.g., uncertainty penalties) significantly reduces value overestimation caused by the extrapolation error~\citep{fujimoto2019off,kumar2019stabilizing}.\looseness=-1

In contrast, our focus is on Bayesian agents that aim to adapt beyond the dataset, as illustrated in \autoref{sec:bandit}. This requires abandoning conservatism, but doing so removes an explicit mechanism for controlling value overestimation.
We show that, perhaps surprisingly, \textit{long-horizon rollouts} can themselves serve this role by actively \textbf{reducing overestimation} compared to short-horizon rollouts.\looseness=-1

To illustrate, starting from a real history $h_t \in \mathcal D$, we sample a model $m_\theta$ from the ensemble $\ensemble$ and generate a trajectory of length $H$, where $\hat a_{t+j} = \pi(\hat h_{t+j})$ comes from a deterministic policy, $(\hat r_{t+j+1}, \hat s_{t+j+1}) \sim m_\theta(\hat s_{t+j}, \hat a_{t+j})$, and $\hat h_t = h_t$.
Applying one-step Bellman backups on the Bayesian value function \textit{along this rollout}, i.e., 
{\small
$$Q^{\text{Bayes}}(\hat h_{t+j}, \hat a_{t+j}) \gets \hat r_{t+j+1} + \gamma Q^{\text{Bayes}} (\hat h_{t+j+1}, \pi(\hat h_{t+j+1})),$$
}for $0 \le j < H$, we obtain an $H$-step TD target on $Q^{\text{Bayes}}(h_t,\hat a_t)$:\looseness=-1
\begin{equation}
\label{eq:overestimation}
\sum\nolimits_{j=0}^{H-1} \gamma^{j} \underbrace{\hat r_{t+j+1}}_{\text{lower bias}} + \underbrace{\gamma^{H}}_{\text{discounted}} \underbrace{Q^{\text{Bayes}}(\hat h_{t+H}, \pi(\hat h_{t+H}))}_{\text{higher bias}}.
\end{equation}
This decomposition highlights the bias trade-off: imagined rewards can be low-bias if the model generalizes, while the bootstrapped term -- more susceptible to overestimation~\citep{kumar2019stabilizing,sims2024edge} -- is exponentially discounted with $H$. We formalize this effect in \autoref{app:proof_bootstrap}, extending the analysis of \citet{sims2024edge}.
Importantly, this role of long-horizon rollouts in mitigating overestimation is distinct from their use for data augmentation or exploration~\citep{young2023benefits}. Rather than relying on conservatism, long horizons absorb overestimation risk by
shifting estimation from bootstrapping toward model-generated rewards. The cost of longer horizons is increased compounding error, which we address in \autoref{sec:how_longhorizon} and \autoref{sec:ensemble_arch}.\looseness=-1

\vspace{-0.5em}
\subsection{How Do We Perform Long-Horizon Rollouts?}
\label{sec:how_longhorizon}
\vspace{-0.5em}

\textbf{Where to start planning?}
From the Bayesian objective (\autoref{eq:bayes_obj}), it is natural to initiate planning rollouts by sampling initial states $s_0 \sim \rho_\D$, where $\rho_\D$ is the empirical initial-state distribution. But this underrepresents states appearing \textit{later} in real trajectories, as compounding errors make them harder to reach.
Following MBPO-style branching mechanism~\citep{janner2019trust}, we sample starting states $s_t \sim \D$ from any timestep $t$ to maintain coverage. Since the environment appears as an epistemic POMDP to the \textit{agent}, we also provide the corresponding history $h_t = (s_{0:t}, a_{0:t-1}, r_{1:t}) \in \D$.\looseness=-1

\begin{figure}[h]
\centering
\vspace{0.5em}
\includegraphics[width=0.6\linewidth]{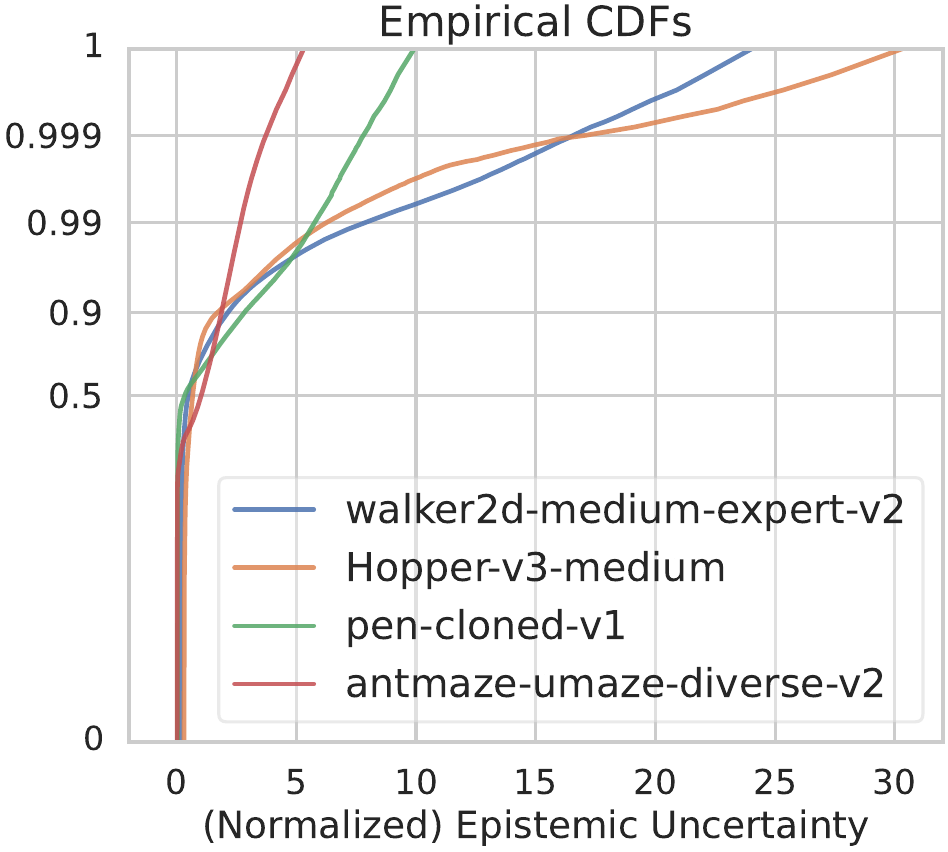}
\vspace{-0.5em}
\caption{\footnotesize \textbf{Empirical CDFs} of epistemic uncertainty $\unc$ over in-dataset $(s,a)$ (logit-scaled y-axis). Uncertainties are normalized by the dataset mean, so $1$ is the average value.\looseness=-1}
\label{fig:cdf}
\end{figure}

\textbf{How to truncate long-horizon rollouts?} 
Since model errors depend on specific $(s,a)$ pairs, a key question is not simply \emph{when} to truncate rollouts, but \emph{where}.
A natural criterion is the model’s uncertainty estimate $\unc(s,a)$, which correlates with prediction error. 
In practice, epistemic uncertainty is highly \textit{non-uniform} even within $\D$: frequently visited $(s,a)$ pairs yield low uncertainty, while rarely seen ones result in high uncertainty, reflecting uncertainty about the true model $m^*$ rather than $m_\D$.\footnote{Standard concentration bounds~\citep[Section D.3]{kumar2020conservative} imply that, with probability $1-\delta$, for $(s,a) \in \supp_{\S \times \A}(\D)$, $\textsc{tv}(m_\D(s,a), m^*(s,a)) \le c_{m^*,\delta}/\sqrt{n_\D(s,a)}$, where $n_\D(s,a)$ is the visitation count in $\D$ and $c_{m^*,\delta}$ is a constant.} As shown in \autoref{fig:cdf}, most empirical CDFs are long-tailed, with dataset-dependent skewness reflecting the underlying visitation distribution.

Prior work has also used uncertainty threshold to limit model error, together with short horizon caps~\citep{pan2020trust,zhan2021model,zhang2023uncertainty,frauenknecht2024trust}. At a high level, these methods and ours share the same principle: using the longest reliable rollout before model error becomes harmful. Our setting differs in how this reliable horizon is used. Once explicit conservatism is removed, short horizons can cause severe overestimation due to excessive reliance on bootstrapping. We therefore remove any fixed cap and use uncertainty alone to determine truncation adaptively, allowing rollouts to remain long wherever the model is reliable.

\textbf{Uncertainty threshold $\uncthres$ for rollout truncation (design choice 1).} Let $\zeta \in [0,1]$, we define
\begin{equation*}
\uncthres \defeq F^{-1}_{Y}(\zeta), \quad Y \defeq \unc(s,a), \, (s,a) \sim \D, 
\end{equation*}
where $F^{-1}_Y$ is the quantile function of $Y$. 
Quantile-based thresholds adapt naturally to different datasets and uncertainty scales. 
Setting $\zeta=1.0$ yields the maximum in-dataset uncertainty and encourages the longest possible rollouts; thresholds beyond this are avoided as they correspond to out-of-distribution uncertainty.\looseness=-1

\vspace{-0.5em}
\subsection{World Architecture for Long-Horizon Rollouts}
\label{sec:ensemble_arch}
\vspace{-0.5em}

To support long rollouts, we introduce simple and effective architectural choices for the world model ensemble that mitigate compounding errors and improve posterior fidelity.\looseness=-1 

\textbf{Larger ensemble size $N$ (design choice 2).}
While deep ensembles approximate Bayesian posteriors, their fidelity relies on member diversity. Under long-horizon rollouts, compounding errors amplify posterior inaccuracies, rendering small ensembles (e.g., $N=5$ in MBPO~\citep{janner2019trust}) inadequate. In our main experiments, we therefore use a larger ensemble size ($N=100$), which remains computationally feasible due to parallelization.\looseness=-1 

\textbf{Layer normalization in the world model (design choice 3).}
We further find that applying layer normalization (LN)~\citep{ba2016layer} within the world model significantly mitigates compounding errors for long-horizon rollouts.
This mirrors its role in reducing extrapolation error in model-free RL~\citep{ball2023efficient}. 
Concretely, we structure the world model as a delta predictor and apply LN to its hidden features. The model predicts the next state as $$\E{}{\hat s'} = s + \mathbf W^\top \mathrm{ReLU}(\mathrm{LN}(\psi(s,a)),$$with features $\psi(s,a) \in \R^{k}$ and output weights $\mathbf W \in \R^{k \times |\S|}$. For LN without affine parameters, under $\ell_2$ norm:\looseness=-1
\begin{equation}
\label{eq:LN}
\begin{aligned}
\| \E{}{\hat s'} - s\|  &\le \| \mathbf W\| \| \mathrm{ReLU}(\mathrm{LN}(\psi(s,a))\|  \\
&\le \| \mathbf W\| \|\mathrm{LN}(\psi(s,a))\| = \sqrt{k} \| \mathbf W\|,
\end{aligned}
\end{equation}
using that fact that $\|\mathrm{LN}(x)\| = \sqrt{k}$ for any $x\in \R^k$.
Applying the triangle inequality over an $H$-step imagined trajectory yields a linear compounding bound, $\|\E{}{\hat s_H} - s_0\| \le H \sqrt{k} \| \mathbf W\|$. Therefore, by controlling $\|\E{}{\hat s_H}\|$, we can upper bound the compounding error: $\|\E{}{\hat s_H} - s_H\| \le \|\E{}{\hat s_H}\| + \|s_H\|$.\looseness=-1

\vspace{-0.5em}
\subsection{Stable Recurrent RL Using Long-Horizon Rollouts}
\vspace{-0.5em}
\label{sec:recurrent_rl}

Lastly, we describe how we train policies from long-horizon rollouts. Following MBPO~\citep{janner2019trust}, we apply model-free RL to a mixture of real and model-generated data. Since real data is off-policy and $\hat \Pr_\D$ induces a POMDP, we use recurrent off-policy RL with a recurrent actor $\pi_\nu(a_t\mid h_t)$ and critic $Q_\omega(h_t,a_t)$, each equipped with a separate RNN encoder ($\nu_\phi(h_t)$ and $\omega_\phi(h_t)$) for stability~\citep{ni2022recurrent}.
Long rollouts and test-time adaptation require memory spanning entire episodes (up to $1000$ steps in our tasks), beyond the capacity of \textit{standard} RNNs~\citep{ni2023transformers}. We therefore extend memoroid~\citep{morad2024recurrent} to actor-critic architectures using linear recurrent units (LRUs)~\citep{orvieto2023resurrecting}, enabling long-term memory with parallel optimization. 
Since memoroid was originally designed for online POMDPs, this mismatch motivates additional design choices below.\looseness=-1

\textbf{Balancing real and imagined data with $\kappa \in (0,1)$ (design choice 4).} Following MBPO, we introduce a mixing ratio $\kappa$ to weight real versus imagined transitions in the RL loss (see \autoref{app:algorithm}). Prior work~\citep{zhang2023uncertainty} shows that when real data is of higher quality, a larger $\kappa$ is preferable.\looseness=-1

\textbf{Small context encoder learning rate $\eta_\phi$ (design choice 5).} 
Borrowing insights from online recurrent RL~\citep{luo2024efficient}, we use a much smaller learning rate for recurrent encoders, since history representations can diverge exponentially with history length even under tiny parameter changes.\looseness=-1

\newcommand{\novel}[1]{\textcolor{magenta}{#1}}
\newcommand{\module}[1]{\textcolor{blue}{#1}}
\vspace{1em}
\begin{algorithm}[h]
\footnotesize
\captionof{algorithm}{\algo: Full training loop}\label{algo:full}
\begin{algorithmic}[1]
    \Require Offline dataset $\D$, Online buffer $\mathcal B \gets \emptyset$. 
    \module{World ensemble}: $\ensemble$, \module{Recurrent actor}: $\pi_\nu$, \module{Recurrent critic}: $Q_\omega$. 
    \State Train $\ensemble$ on $\D$ until convergence
    \While{gradient steps $\le$ max gradient steps}
        \For{$k=1$ to $K$ rollouts} \Comment{Parallelized in practice}
            \State \novel{Sample $m_\theta \sim \ensemble$, fixed for the rollout}
            \State \novel{Sample initial history $h_t \sim \D$} 
            \State Append \Call{Rollout}{$h_t$, $m_\theta$, $\pi_\nu$} to $\mathcal B$
        \EndFor
        \For{$g=1$ to $G$ gradient steps}
            \State Optimize RL loss $L(Q_\omega, \pi_\nu; \tau, \kappa)$ with $\tau \sim \mathcal B$
        \EndFor
    \EndWhile
\end{algorithmic}
\end{algorithm}
\vspace{1em}
\begin{algorithm}[h]
\footnotesize
\captionof{algorithm}{\algo: Rollout function}\label{algo:rollout}
\begin{algorithmic}[1]
    \Require Terminal function $\term$, Max episode length $T$, Uncertainty threshold $\uncthres$, \novel{\sout{Rollout horizon $H$}}
    \Function{Rollout}{$h_t, m_\theta, \pi_\nu $} 
    \State Set $\hat h_t = h_t$ and $\mathrm{done}$ to False
    \While{$\mathrm{done}$ is False} 
        \State $\hat a_t \sim \pi_\nu(\hat h_t), (\hat r_{t+1}, \hat s_{t+1}) \sim m_\theta(\hat s_t, \hat a_t)$ 
        \State $\hat d_{t+1} = \term(\hat s_t,\hat a_t, \hat s_{t+1})$
        \State \novel{$\mathrm{trunc} = (\unc(\hat s_t, \hat a_t) > \uncthres)$}
        \State \novel{$\mathrm{done} = \hat d_{t+1} \lor \mathrm{trunc} \lor (t+1 \ge T)$}
        \State $\hat h_{t+1} = (\hat h_t, \hat a_t, \hat r_{t+1}, \hat s_{t+1})$, $t \gets t+1$
    \EndWhile
    \State \Return Trajectory $\tau$ (i.e., $h_t$ plus imagined rollout)
    \EndFunction
\end{algorithmic}
\end{algorithm}

\vspace{-0.5em}
\subsection{Overall Algorithm}
\vspace{-0.5em}
\label{sec:overall_algorithm}

We now present the full method in \autoref{algo:full}, integrating all design choices with the core rollout subroutine (\autoref{algo:rollout}). 
\algo remains conceptually simple, using single-step models and no uncertainty penalty, while addressing the challenges of Bayesian offline RL. 
We \novel{highlight the lines} where \algo differs from prior work. In particular, to respect the Bayesian objective, we fix a single model per rollout, rather than randomizing the model index at each step as in MBPO-style methods.
For completeness, \autoref{app:algorithm} provides details on RL loss and stopping criteria, although they are not needed to follow the experiments below.\looseness=-1

\vspace{-1em}
\section{Experiments}
\label{sec:experiments}
\vspace{-0.5em}

\textbf{Benchmarks.}
We evaluate on two widely used offline RL suites for continuous control: D4RL~\citep{fu2020d4rl} and NeoRL~\citep{qin2022neorl}. 
From \textbf{D4RL}, we use the \textbf{locomotion} benchmark (\textbf{12} datasets) spanning three MuJoCo environments: halfcheetah (hc), hopper (hp), and walker2d (wk). Each one provides four data regimes: random, medium-replay, medium, and medium-expert. 
From \textbf{NeoRL}, we use its \textbf{locomotion} benchmark (\textbf{9} datasets) built on the same environments, but with data collected from more deterministic policies, resulting in narrower coverage and different data regimes: Low, Medium, High. 
We also include the \textbf{Adroit} benchmark (\textbf{6} datasets) from D4RL, where a 28-DoF robotic arm manipulates objects to solve tasks (pen, door, hammer) with data collected from human demonstrations (human) and behavior cloning (cloned). The difficulty of Adroit lies in its high-dimensional control, small data size, and low-data quality. 
Lastly, we perform experiments on \textbf{AntMaze} (\textbf{6} datasets), where an 8-DoF ant robot must reach a goal position in a maze. Tasks span maze layouts of different sizes (umaze, medium, large) and start distributions (play, diverse), and are challenging for model-based RL due to navigation under sparse rewards. See more details in \autoref{app:dataset}. \looseness=-1

\textbf{Baselines.}
We compare against offline RL baselines grouped into three categories, including \textbf{conservative model-free RL} (behavior cloning (BC), CQL~\citep{kumar2020conservative}, IQL~\citep{kostrikov2021offline}, EDAC~\citep{an2021uncertainty}, ReBRAC~\citep{tarasov2023revisiting}),
\textbf{conservative model-based RL} (MOPO~\citep{yu2020mopo}, COMBO~\citep{yu2021combo}, RAMBO~\citep{rigter2022rambo}, ARMOR~\citep{bhardwaj2023adversarial}, MOBILE~\citep{sun2023model}, LEQ~\citep{park2025model}, MoMo~\citep{srinivasan2024offline}, ADMPO~\citep{lin2025any}, SUMO~\citep{qiao2025sumo}, VIPO~\citep{chen2025vipo}, ScorePen~\citep{liu2025leveraging}, ROMI~\citep{qiao2026modelbased}), and \textbf{Bayesian-inspired RL} (APE-V~\citep{ghosh2022offline}, MAPLE~\citep{chen2021offline}, CBOP~\citep{jeong2022conservative}, MoDAP~\citep{choi2024diversification}).\looseness=-1

\textbf{Algorithm setup.}
In our main experiments (\autoref{sec:benchmarking}), we adopt the following \textbf{default} design choices: ensemble size $N=100$, layer normalization in the world models, uncertainty threshold $\uncthres$ of $\zeta = 1.0$.
The learning rates of the actor and critic MLP heads are fixed to \num{1e-4}, while their RNN encoders share a tied learning rate $\eta_\phi$ swept over $[\num{3e-7}, \num{1e-4}]$, the exact range being benchmark-dependent. Similarly, the real data ratio $\kappa$ is swept within $[0.05, 0.95]$, also depending on the benchmark. Note that this yields \textit{two} per-dataset hyperparameters, consistent with common practice in conservative model-based RL~\citep{yu2020mopo,rigter2022rambo,sun2023model}.

Following common practice in offline RL benchmarking, including the baselines we compare against, we tune these hyperparameters based on final policy performance in the true environment. We report the best results for each dataset. 
Implementation details are provided in \autoref{app:implementation}. In \autoref{sec:analysis}, we conduct sensitivity and ablation studies on design choices.\looseness=-1

\vspace{-0.5em}
\subsection{Benchmarking Results} 
\label{sec:benchmarking}
\vspace{-0.5em}

\definecolor{hlcolor}{RGB}{225,230,255}
\renewcommand{\highlight}[1]{\cellcolor{hlcolor}{#1}}

\begingroup
\setlength{\tabcolsep}{2pt}       %
\renewcommand{\arraystretch}{1.1} %
\begin{table*}[h]
\caption{\footnotesize Comparison of offline RL methods on the \textbf{D4RL locomotion} benchmark. We report mean normalized scores for all baselines, with {\scriptsize $\pm$std} for competitive baselines. The \textbf{best mean score} is bolded, and \texthl{hlcolor}{marked} methods are statistically similar under a $t$-test. Our results use 6 seeds, each evaluated at the final step with 20 episodes.\looseness=-1}
\vspace{-0.5em}
\label{tab:d4rl_loco}
\centering
\resizebox{\linewidth}{!}{%
\begin{tabular}{lrrrrrrrrrrrrrrrr}
\toprule
 & \multicolumn{2}{c}{Model-free} & \multicolumn{9}{c}{Conservative model-based} & \multicolumn{4}{c}{Bayesian-inspired} & \multicolumn{1}{c}{Ours} \\
\cmidrule(lr){2-3} \cmidrule(lr){4-12} \cmidrule(lr){13-16} \cmidrule(lr){17-17}
Dataset & CQL & EDAC & {\scriptsize COMBO} & {\scriptsize RAMBO} & MOPO & LEQ & MOBILE & MoMo & ADMPO & SUMO & VIPO & APE-V & MAPLE & MoDAP & CBOP & \algo \\
\midrule
hc-random  & \valunc{31.3}{} & \valunc{28.4}{} & \valunc{38.8}{} & \valunc{40.0}{} & \valunc{38.5}{} & \valunc{30.8}{3.3} & \valunc{39.3}{3.0} & \valunc{39.6}{3.7} & \highlight{\valunc{\textbf{45.4}}{2.8}} & \valunc{34.9}{2.1} & \valunc{42.5}{0.2} & \valunc{29.9}{} & \valunc{41.5}{} & \valunc{36.5}{1.8} & \valunc{32.8}{0.4} & \valunc{37.0}{3.3} \\
hp-random   & \valunc{5.3}{}  & \valunc{25.3}{} & \valunc{17.9}{} & \valunc{21.6}{} & \highlight{\valunc{31.7}{}} & \highlight{\valunc{32.4}{0.3}} & \highlight{\valunc{31.9}{0.6}} & \valunc{18.3}{2.8} & \highlight{\valunc{32.7}{0.2}} & \valunc{30.8}{0.9} & \highlight{\valunc{\textbf{33.4}}{1.9}} & \valunc{31.3}{} & \valunc{10.7}{} & \valunc{8.9}{1.1} & \highlight{\valunc{31.4}{0.0}} & \highlight{\valunc{24.5}{28.5}} \\
wk-random   & \valunc{5.4}{}  & \valunc{16.6}{} & \valunc{7.0}{}  & \valunc{11.5}{} & \valunc{7.4}{}  & \valunc{21.5}{0.1} & \valunc{17.9}{6.6} & \valunc{26.8}{3.3} & \valunc{22.2}{0.2} & \valunc{27.9}{2.0} & \valunc{20.0}{0.1} & \valunc{15.5}{} & \valunc{22.1}{} & \valunc{23.1}{1.6} & \valunc{17.8}{0.4} & \highlight{\valunc{\textbf{34.1}}{6.8}} \\
\midrule
hc-med-rep & \valunc{45.3}{} & \valunc{61.3}{} & \valunc{55.1}{} & \highlight{\valunc{\textbf{77.6}}{}} & \valunc{72.1}{} & \valunc{65.5}{1.1} & \valunc{71.7}{1.2} & \valunc{72.9}{1.8} & \valunc{67.6}{3.4} & \highlight{\valunc{76.2}{1.3}} & \highlight{\valunc{77.2}{0.4}} & \valunc{64.6}{} & \valunc{69.5}{} & \valunc{67.3}{3.4} & \valunc{66.4}{0.3} & \valunc{72.1}{2.4} \\
hp-med-rep  & \valunc{86.3}{} & \valunc{101.0}{} & \valunc{89.5}{} & \valunc{92.8}{} & \valunc{103.5}{} & \valunc{103.9}{1.3} & \valunc{103.9}{1.0} & \valunc{104.0}{1.8} & \valunc{104.4}{0.4} & \highlight{\valunc{109.9}{1.4}} & \highlight{\valunc{109.6}{0.9}} & \valunc{98.5}{} & \valunc{85.0}{} & \valunc{94.2}{4.8} & \valunc{104.3}{0.4} & \highlight{\valunc{\textbf{110.6}}{0.7}} \\
wk-med-rep  & \valunc{76.8}{} & \valunc{87.1}{} & \valunc{56.0}{} & \valunc{86.9}{} & \valunc{85.6}{} & \highlight{\valunc{98.7}{6.0}} & \highlight{\valunc{89.9}{1.5}} & \highlight{\valunc{90.4}{7.7}} & \highlight{\valunc{95.6}{2.1}} & \valunc{78.2}{1.5} & \highlight{\valunc{98.4}{0.3}} & \valunc{82.9}{} & \valunc{75.4}{} & \highlight{\valunc{88.4}{4.2}} & \highlight{\valunc{92.7}{0.9}} & \highlight{\valunc{\textbf{99.3}}{19.3}} \\
\midrule
hc-medium  & \valunc{46.9}{} & \valunc{65.9}{} & \valunc{54.2}{} & \valunc{68.9}{} & \valunc{73.0}{} & \valunc{71.7}{4.4} & \valunc{74.6}{1.2} & \valunc{77.1}{0.9} & \valunc{72.2}{0.6} & \highlight{\valunc{\textbf{84.3}}{2.4}} & \valunc{80.0}{0.4} & \valunc{69.1}{} & \valunc{48.5}{} & \valunc{77.3}{1.1} & \valunc{74.3}{0.2} & \valunc{78.6}{1.6} \\
hp-medium   & \valunc{61.9}{} & \valunc{101.6}{} & \valunc{97.2}{} & \valunc{96.6}{} & \valunc{62.8}{} & \valunc{103.4}{0.3} & \valunc{106.6}{0.6} & \highlight{\valunc{\textbf{110.8}}{2.3}} & \valunc{107.4}{0.6} & \valunc{104.8}{2.1} & \valunc{107.7}{1.0} & \textemdash & \valunc{44.1}{} & \valunc{106.6}{1.9} & \valunc{102.6}{0.1} & \valunc{54.2}{7.2} \\
wk-medium   & \valunc{79.5}{} & \highlight{\valunc{92.5}{}} & \valunc{81.9}{} & \valunc{85.0}{} & \valunc{84.1}{} & \valunc{74.9}{26.9} & \highlight{\valunc{87.7}{1.1}} & \highlight{\valunc{95.0}{1.4}} & \highlight{\valunc{93.2}{1.1}} & \highlight{\valunc{94.1}{2.5}} & \highlight{\valunc{93.1}{1.8}} & \highlight{\valunc{90.3}{}} & \valunc{81.3}{} & \valunc{81.1}{6.5} & \highlight{\valunc{95.5}{0.4}} & \highlight{\valunc{\textbf{106.4}}{23.0}} \\
\midrule
hc-med-exp & \valunc{95.0}{} & \valunc{106.3}{} & \valunc{90.0}{} & \valunc{93.7}{} & \valunc{90.8}{} & \valunc{102.8}{0.4} & \highlight{\valunc{108.2}{2.5}} & \valunc{107.9}{1.9} & \valunc{103.7}{0.2} & \valunc{106.6}{2.4} & \highlight{\valunc{\textbf{110.0}}{0.4}} & \valunc{101.4}{} & \valunc{55.4}{} & \valunc{103.4}{4.3} & \valunc{105.4}{1.6} & \highlight{\valunc{109.5}{8.7}} \\
hp-med-exp  & \valunc{96.9}{} & \valunc{110.7}{} & \valunc{111.1}{} & \valunc{83.3}{} & \valunc{81.6}{} & \valunc{109.4}{1.8} & \valunc{112.6}{0.2} & \valunc{109.1}{0.4} & \valunc{112.7}{0.3} & \valunc{107.8}{0.7} & \valunc{113.2}{0.1} & \valunc{105.7}{} & \valunc{95.3}{} & \valunc{94.5}{7.8} & \valunc{111.6}{0.2} & \highlight{\valunc{\textbf{114.8}}{0.5}} \\
wk-med-exp  & \valunc{109.1}{} & \valunc{114.7}{} & \valunc{103.3}{} & \valunc{68.3}{} & \valunc{112.9}{} & \valunc{108.2}{1.3} & \valunc{115.2}{0.7} & \valunc{118.4}{0.9} & \valunc{114.9}{0.3} & \highlight{\valunc{\textbf{122.8}}{0.4}} & \valunc{117.7}{1.0} & \valunc{110.0}{} & \valunc{107.0}{} & \valunc{112.2}{2.8} & \valunc{117.2}{0.5} & \valunc{120.7}{1.3} \\
\midrule
AVG   & \valunc{61.6}{} & \valunc{76.0}{} & \valunc{66.8}{} & \valunc{68.9}{} & \valunc{70.3}{} & \valunc{76.9}{} & \valunc{80.0}{} & \valunc{80.9}{} & \valunc{81.0}{} & \valunc{81.5}{} & \valunc{\textbf{83.6}}{} & \textemdash & \valunc{61.3}{} & \valunc{74.5}{} & \valunc{79.3}{} & \valunc{80.1}{} \\
\bottomrule
\end{tabular}
}
\end{table*}
\endgroup

We present results for D4RL locomotion in \autoref{tab:d4rl_loco}, NeoRL locomotion in \autoref{tab:neorl_loco}, D4RL Adroit in \autoref{tab:d4rl_adroit}, and D4RL AntMaze in \autoref{tab:d4rl_antmaze}. 
Overall, \algo is competitive with leading conservative methods. The average normalized scores are 80.1 vs. 83.6 for the best baseline (VIPO) in D4RL locomotion, 64.7 vs. 73.3 for the best baseline (VIPO) in NeoRL locomotion, 21.1 vs. 28.1 for the best model-based baseline (MoMo) in D4RL Adroit, and 28.8 vs. 64.9 for the best model-based baseline (LEQ) in D4RL AntMaze. 
\algo establishes \textbf{new state-of-the-art} (SOTA) mean performance on 7 of the 33 datasets, with 3 gains statistically significant. In terms of magnitude, it advances walker-random-v2 (27.9 $\to$ 34.1), walker-medium-v2 (95.5 $\to$ 106.4), pen-cloned-v1 (74.1 $\to$ 91.3), and hammer-cloned-v1 (5.0 $\to$ 14.4).
It also surpasses prior Bayesian-inspired methods on most tasks, despite their reliance on explicit conservatism.\footnote{An exception is MoDAP, which avoids explicit conservatism by fine-tuning model during policy learning.}
To ensure fair comparison with conservative approaches, we also evaluate a design-matched conservative variant of \algo by incorporating uncertainty penalties while keeping all other design choices identical (see \autoref{sec:analysis}).

\textbf{When is \algo better than conservative methods?}
We group the 33 datasets into 3 categories:\looseness=-1
\begin{itemize}[noitemsep, topsep=0pt, leftmargin=*]
\item \textbf{Low-quality datasets (11 tasks):} 3 D4RL random, 3 NeoRL Low, 2 Adroit door, 2 Adroit hammer, and AntMaze umaze-diverse. Consistent with the bandit insight (\autoref{sec:bandit}), \algo ranks among the best methods in 5 of 11 tasks and achieves reasonable performance in 10 of 11.\looseness=-1
\item \textbf{Medium-quality, moderate coverage (7 tasks):} 3 D4RL medium-replay, 3 D4RL medium-expert\footnote{D4RL expert and NeoRL High are treated as medium-quality, as their behavior policies are not fully optimal.}, and Adroit pen-cloned.
\algo ranks among the best methods in 5 of 7 tasks, without any failures.\looseness=-1
\item \textbf{Medium-quality, narrow coverage (15 tasks):} 3 D4RL medium, 3 NeoRL Medium, 3 NeoRL High, Adroit pen-human, 5 AntMaze tasks. \algo is weaker, ranking among the best methods in only 3 of 15 tasks. 
In AntMaze, narrow coverage and contact-rich dynamics make world modeling difficult, while sparse rewards hinder exploration (see \autoref{app:antmaze} for analysis). Adding uncertainty penalties in \algo further degrades performance (\autoref{tab:ablation}), suggesting that AntMaze remains challenging for many model-based RL methods.\looseness=-1
\end{itemize}

\vspace{-0.5em}
\subsection{What Matters in \algo?}
\label{sec:analysis}

\newcommand{\degL}[1]{\cellcolor{red!10}#1}   %
\newcommand{\degM}[1]{\cellcolor{red!35}#1}   %
\newcommand{\degD}[1]{\cellcolor{red!60}#1}   %
\newcommand{\impL}[1]{\cellcolor{green!15}#1} %
\newcommand{\impM}[1]{\cellcolor{green!35}#1} %
\newcommand{\impD}[1]{\cellcolor{green!60}#1} %

\begingroup
\setlength{\tabcolsep}{4pt}       %
\renewcommand{\arraystretch}{1.1} %
\begin{table*}[h]
\caption{\footnotesize 
\textbf{Sensitivity and ablation results averaged across benchmarks.} 
The \texthl{hlcolor}{highlighted} setting ($N{=}100$, $\lambda{=}0.0$, $\zeta{=}1.0$, using the entire history as agent input) is the main result. Ablations vary one hyperparameter at a time, except for the Markov agent, where we sweep the real data ratio $\kappa$ for a fair comparison. 
Shading shows degradation level: \texthl{red!10}{light} (3–10), \texthl{red!35}{medium} (10–30), \texthl{red!60}{dark} ($>$30).
Full per-dataset results appear in \autoref{tab:ablation_full}.
}
\vspace{-0.5em}
\label{tab:ablation}
\centering
\resizebox{0.9\linewidth}{!}{%
\begin{tabular}{l|rrr|rrrrr|rrrr|rr}
\toprule
& \multicolumn{3}{c|}{Ensemble size $N$} 
& \multicolumn{5}{c|}{Unc.\ penalty coef.\ $\lambda$} 
& \multicolumn{4}{c|}{Truncation threshold $\zeta$} & \multicolumn{2}{c}{Agent input} \\
\cmidrule(lr){2-4} \cmidrule(lr){5-9} \cmidrule(lr){10-13} \cmidrule(lr){14-15}
Benchmark
& \highlight{100} & 20 & 5
& \highlight{0.0} & 0.04 & 0.2 & 1.0 & 5.0
& \highlight{1.0} & 0.999 & 0.99 & 0.9 
& \highlight{Hist.} & Mark. \\
\midrule
D4RL Locomotion (12 tasks)
& \valunc{80.1}{}             %
& \degL{\valunc{71.6}{}}      %
& \degM{\valunc{65.9}{}}      %
& \valunc{80.1}{}             %
& \valunc{79.8}{}             %
& \valunc{80.4}{}             %
& \degL{\valunc{73.1}{}}      %
& \degM{57.4}
& \valunc{80.1}{}             %
& \degM{\valunc{69.6}{}}      %
& \degD{\valunc{45.1}{}}      %
& \degD{\valunc{22.5}{}}      %
& \valunc{80.1}{}
& \degL{75.8}
\\
NeoRL Locomotion (9 tasks)
& \valunc{64.7}{}
& \degL{\valunc{60.3}{}}      %
& \degL{\valunc{59.8}{}}      %
& \valunc{64.7}{}
& \degL{\valunc{61.5}{}}      %
& \degL{\valunc{59.8}{}}      %
& \degL{\valunc{58.3}{}}      %
& \degM{42.5}
& \valunc{64.7}{}
& \degL{\valunc{58.8}{}}      %
& \degM{\valunc{36.3}{}}      %
& \degD{\valunc{16.7}{}}      %
& \valunc{64.7}{}
& 66.8 
\\
D4RL Adroit (3 non-zero tasks)
& \valunc{42.2}{}
& \degM{\valunc{30.0}{}}      %
& \degL{\valunc{32.6}{}}      %
& \valunc{42.2}{}
& \degL{\valunc{33.7}{}}      %
& \degL{\valunc{34.2}{}}      %
& \degL{\valunc{38.1}{}}      %
& \degL{\valunc{34.1}{}}
& \valunc{42.2}{}
& \degM{\valunc{17.8}{}}      %
& \degM{\valunc{16.0}{}}      %
& \degD{\valunc{1.9}{}}       %
& \valunc{42.2}{}
& \degL{36.4}
\\
D4RL AntMaze (4 non-zero tasks) 
& \valunc{43.2}{}
& \degM{33.2}
& \degM{28.8}
& \valunc{43.2}{}
& \degD{5.0}
& \degD{3.2}
& \degD{4.2}
& \degD{12.0}
& \valunc{43.2}{}
& \degL{35.8}
& \degL{38.4}
& \degM{32.7}
& \valunc{43.2}{}
& \degD{1.3}
\\
\bottomrule
\end{tabular}
}
\end{table*}
\endgroup

To understand which components are critical to the success of \algo, we conduct sensitivity studies on the design choices introduced in \autoref{sec:algorithm} and ablation studies on the uncertainty penalty and Markov agent (shown in \autoref{app:ablation}), using the best hyperparameters identified in \autoref{sec:benchmarking}.\looseness=-1

\begin{figure*}[h]
    \centering
    \includegraphics[width=\linewidth]{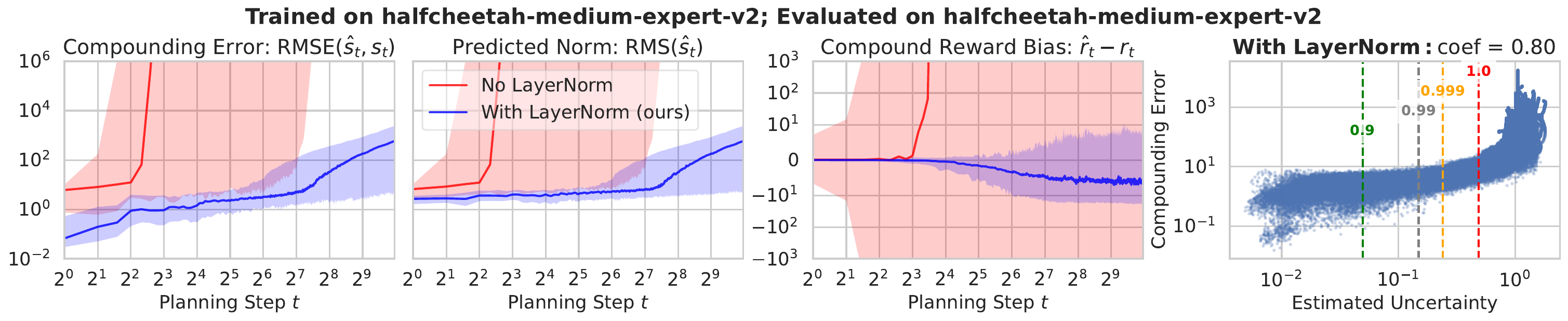}
    \vspace{-2em}
    \caption{\footnotesize\textbf{Effect of LayerNorm in world models} trained and evaluated on halfcheetah-medium-expert-v2.  
    We collect 200 rollouts and truncate only on \texttt{float32} overflow, without using an uncertainty threshold.
    For each metric, we plot the \textbf{median} (solid line) together with the \textbf{5-95\% percentile band} across rollouts. 
    The rightmost scatter plot show the Spearman's rank coefficient in the with-LayerNorm setting; vertical lines mark uncertainty thresholds $\zeta \in\{0.9, 0.99,0.999, 1.0\}$. 
    \textbf{Full results} and plotting setup are shown in \autoref{app:LN}.\looseness=-1 
    }
    \vspace{-1em}
    \label{fig:compound_main}
\end{figure*}

\textbf{Adaptive long-horizon rollout (\autoref{sec:how_longhorizon}) is the \textit{decisive} factor for success.} 
We vary the uncertainty quantile $\zeta \in \{0.9, 0.99, 0.999, 1.0\}$ and find that the most aggressive choice, $\zeta=1.0$, yields the best performance on almost all datasets.
As shown in \autoref{fig:horizon_main} (with full results in \autoref{app:unc_thres} and a summary in \autoref{tab:ablation}), $\zeta=1.0$ yields \textbf{64-512 steps} for 75th-percentile horizon and \textbf{256-1000 steps} for maximum horizon, in 21 out of 23 tasks with $T=1000$.
These horizons are far longer than typically used in model-based RL, contradicting the conventional view that long rollouts are harmful when horizons are fixed~\citep{janner2019trust}.
With \emph{adaptive} horizons, compounding errors remain controlled, allowing rollouts to extend deep into the episode.\looseness=-1

In contrast, performance often collapses to \textit{near zero} (i.e., scores $\le 5$) with smaller thresholds such as $\zeta = 0.9$ (16 datasets) or $0.99$ (6 datasets). The consistent failure mode behind these cases is \textbf{severe value overestimation}, where estimated Q-values ($\E{(h_t,a_t)\sim \D}{Q_\omega(h_t,a_t)}$) are much higher (2nd column in \autoref{fig:horizon_main}) compared to the actual performances.
These results support our analysis in \autoref{sec:why_longhorizon}: without explicit conservatism, Bayesian RL mitigates overestimation by relying more on imagined rewards over long horizons, which are lower-biased (reward panel in \autoref{fig:compound_main}). Adaptive horizons ensure that the model is trusted only within its in-distribution confidence region, enabling effective long-horizon rollouts.\looseness=-1

\textbf{LayerNorm in the world model (\autoref{sec:ensemble_arch}) and uncertainty truncation (\autoref{sec:how_longhorizon}) jointly enables long-horizon rollouts.}
\autoref{fig:compound_main} (see \autoref{app:LN} for full results) shows that without LN, predicted state norms diverge quickly (around 10 steps here), driving exploding compounding errors.
With LN, state norms remain bounded, which suppresses state error growth and stabilizes reward predictions. 
This matches our intuition based on \autoref{eq:LN}: by normalizing features at each step, LN constrains prediction magnitudes and thus compounding error.
Moreover, the rightmost scatter plot shows what uncertainty-based truncation \emph{would} accomplish: although we display full rollouts to reveal their divergence, applying thresholds would cut them off before entering high-error regions. Thus, LayerNorm prevents error explosion, while uncertainty cutoff provides a complementary safeguard.\looseness=-1

\textbf{Larger ensemble size (\autoref{sec:ensemble_arch}) improves performance.} 
As summarized in \autoref{tab:ablation} and detailed in \autoref{tab:ablation_full}, reducing the ensemble size $N$ to $20$ or $5$ degrades performance, though often moderately. This indicates that $N=100$ is close to the practical limit of what ensembling can offer for these tasks, while smaller ensembles remain viable.\looseness=-1

\textbf{Context encoder learning rate and real data ratio (\autoref{sec:recurrent_rl}) have to be tuned per dataset.} 
The best values of $\eta_\phi$ and $\kappa$ are reported in \autoref{tab:sweep_d4rl_loco}--\autoref{tab:sweep_adroit}, with selective learning curves shown in \autoref{fig:enclr}--\autoref{fig:realw}.
We find the optimal encoder learning rates are generally smaller than those used in online POMDPs~\citep{luo2024efficient}; for example, \num{1e-6} and \num{3e-7} yield the best results on 7 datasets. Intuitively, in Bayesian offline RL, smaller learning rates help curb overestimation by slowing down learning.
For the real data ratio, $\kappa = 0.05$, widely used in prior conservative algorithms~\citep{yu2020mopo}, often yields poor performance. Large $\kappa$ acts as a \textit{softened} regularizer, limiting overtrust in the model while avoiding explicit penalties.\looseness=-1

\textbf{Ablation: Introducing explicit conservatism to \algo helps some tasks, but not on average.} 
We study explicit conservatism by penalizing imagined rewards with $\lambda \frac{\unc(\hat s,\hat a)}{\E{(s,a)\sim\D}{\unc(s,a)}}$, normalized by the dataset average so that $\lambda$ is more comparable across datasets.
As summarized in \autoref{tab:ablation}, strong penalties ($\lambda=1.0$ or $5.0$) generally hurt performance, while a small penalty ($\lambda=0.04$) performs comparably to \algo ($\lambda=0$). Effect remains dataset-dependent, as detailed in \autoref{tab:ablation_full}.
Consistent with the bandit intuition, heavy penalties significantly worsen performance on 6 of 8 low-quality datasets, while leaving \verb|*|-medium-expert datasets largely unaffected.
In contrast, some tasks benefit substantially: hopper-random-v2 (24.5 $\to$ 48.2, a new SOTA), hopper-medium-v2 (54.2 $\to$ 105.8), and pen-human-v1 (20.8 $\to$ 35.9), but each requires a different $\lambda$, highlighting the need for tuning. However, penalties are not a universal remedy for narrow data: Walker2d-v3-Medium and halfcheetah-medium-v2 still fall short of the best baselines, and in AntMaze, penalties consistently harm performance.\looseness=-1

\vspace{-0.5em}
\subsection{Computation Costs}
\label{sec:computation}

The total training cost of \algo consists of two phases: \textit{world-model training} and \textit{agent training}. In practice, the world model is trained once per seed and its ensemble checkpoint is reused when tuning agents, making world-model training a \textit{one-time cost}, while agent training dominates the overall computation. All experiments were run on a single NVIDIA L40S GPU (48 GB) with 3 seeds in parallel, using a fully vectorized JAX implementation.\looseness=-1

For halfcheetah-medium-expert-v2 with 2M gradient steps and ensemble size $N=100$, world-model training costs 6 hours per seed and 10.7 GB memory, while recurrent agent training costs 4.4 hours per seed and 2.6 GB memory. As detailed in \autoref{app:compute}, the rollout inference cost is minimal, and increasing ensemble size $N$ or rollout size $K$ has only a minor effect on runtime.\looseness=-1

\vspace{-0.5em}
\section{Discussion}
\label{sec:concl}

\textbf{Practical guidelines for \algo.} Based on our experiments, we suggest the following guidelines when applying \algo to new tasks or after modifying key modules:

\begin{itemize}[noitemsep, topsep=0pt, leftmargin=*]
    \item \textit{When to use \algo?} \algo is best suited for low-quality or moderate-coverage datasets and settings that permit \textit{test-time adaptation}. It may underperform on narrow-coverage ones possibly due to current limits in posterior modeling. In safety-critical domains, uncertainty penalties can be reintroduced to prioritize safety.
    \item \textit{Defaults first.} Use a large ensemble size, LayerNorm in the world model, and uncertainty threshold $\zeta=1.0$.
    \item \textit{Encoder learning rate and real data ratio.} Tune $\eta_\phi$ within a wide range, typically $3$–$300\times$ smaller than the MLP head’s learning rate. Adjust $\kappa$ within $[0.5,0.95]$.
    \item \textit{Monitor overestimation.} Mitigate it by reducing the discount factor $\gamma$ or lowering the learning rate $\eta_\phi$, as suggested by our analysis in \autoref{sec:why_longhorizon}.
\end{itemize}

\textbf{Conclusion.} 
We revisit (explicit) conservatism as the dominant principle in offline RL and show that it is not universally optimal. Instead, we advance a Bayesian perspective that trains history-dependent agents to maximize expected rewards over a posterior of world models. Building on this principle, we show that \textit{long-horizon rollouts} play a critical role in reducing value overestimation once explicit conservatism is removed, and we propose key design choices that enable both performing and learning from such rollouts, yielding the \algo algorithm. Across diverse benchmarks, \algo is competitive with conservative baselines and particularly effective on low-quality and moderate-coverage datasets. More broadly, our work suggests a shift in offline RL toward the \textit{era of experience}~\citep{silver2025welcome}, where agents are designed to adapt and refine behavior as uncertainty resolves through test-time interaction.\looseness=-1 

\textbf{Future work.} We see several promising directions for future work. Improving world models, through multi-step prediction or generative models, can help push the limits of Bayesian offline RL. Better uncertainty quantification remains key for planning, suggesting deeper connections to Bayesian inference are worth exploring. Finally, reducing sensitivity to hyperparameters and dataset characteristics, as well as replacing current test-environment-based tuning with offline policy selection, would make \algo more robust and practical; since \algo already fits a Bayesian world model, a natural direction is to carry out such validation under the learned posterior~\citep{fellows2025sorel}.\looseness=-1

\section*{Acknowledgements}

We especially thank Yihao Sun for insightful discussions on the effect of layer normalization on controlling compounding errors and on offline model-based RL more broadly. We especially thank Steven Morad for his recurrent RL implementation and, together with his lab students, for detailed comments on the paper. We also thank Ziyan Luo, Lu Li, Csaba Szepesvári, Bernd Frauenknecht, Aditya Mahajan, Rasool Fakoor, and Erick Delage for valuable feedback on various aspects of this work. We are grateful to the ICLR and ICML reviewers for their constructive comments, which have greatly improved the paper.

This work was funded by Canada CIFAR AI Chairs program. This work was enabled by compute resources, software and technical help provided by Mila (\href{https://mila.quebec/en}{mila.quebec}) and the Digital Research Alliance of Canada (\href{https://www.alliancecan.ca/en}{alliancecan.ca}).

\section*{Impact Statement}

This paper presents work whose goal is to advance the field of Machine Learning. There are many potential societal consequences of our work, none which we feel must be specifically highlighted here.

{\small
\bibliography{citation}
\bibliographystyle{icml2026}
}

\newpage
\appendix
\onecolumn

\makeatletter
\def\addcontentsline#1#2#3{%
  \addtocontents{#1}{\protect\contentsline{#2}{#3}{\thepage}{\@currentHref}}%
}
\makeatother

\begingroup
  \setcounter{tocdepth}{2}
 \makeatletter
  \let\old@starttoc\@starttoc
  \def\@starttoc#1{\begingroup\parskip=2pt\old@starttoc{#1}\endgroup}
  \makeatother
  
  \tableofcontents
\endgroup

\section{Related Work}

\subsection{Offline Conservative Model-Based Methods}

Model-based methods can be broadly grouped into background planning and decision-time planning~\citep[Chapter 8.8]{sutton2018reinforcement}. 
Background planning, such as Dyna-style methods~\citep{sutton1990integrated,janner2019trust}, uses model-generated rollouts to learn a global policy or value function, which is then queried for action selection.  
Decision-time planning, such as model predictive control (MPC), performs online look-ahead at test time to select actions for the current state. 
Our work belongs to \textit{background planning}. Below, we review these paradigms within the offline setting, focusing on how they incorporate explicit conservatism.

\textbf{Conservative background planning algorithms.} In this category of work, most methods enforce conservatism by constructing a \textit{pessimistic MDP} that penalizes imagined state-action pairs, which we refer to as \textbf{uncertainty-penalized pessimism} in \autoref{app:bg_conservatism}. The earliest examples are MOReL~\citep{kidambi2020morel}  and MOPO~\citep{yu2020mopo}, both approximating the uncertainty set with world ensembles. 
MOReL adopts strong pessimism by mapping any state-action pair with ensemble disagreement above a threshold to an absorbing state with a large penalty. This aggressive design disables value bootstrapping on uncertain regions and explains why MOReL supports long-horizon planning (up to 500 steps) even with severe compounding errors.
In contrast, MOPO applies a milder penalty to imagined rewards based on the aleatoric uncertainty, retaining bootstrapping and thus limiting rollouts to very short horizons (typically 1-5 steps). 

Building on MOPO, several follow-up works redesign the uncertainty quantifier: MOBILE~\citep{sun2023model} uses inconsistencies in Bellman operators across ensemble members; \citet{kim2023model} uses the inverse frequencies of state-action pairs; MoMo~\citep{srinivasan2024offline} adopts energy-based models; SUMO~\citep{qiao2025sumo} employs k-nearest neighbors. 
Other works introduce alternative conservative mechanisms: COMBO~\citep{yu2021combo} uses CQL regularizer; LEQ~\citep{park2025model} uses lower expectile regression. 
From the model side, several works improve the model learning or rollout sampling: \citet{luo2024reward} trains discriminators on $(s,a,s')$ to resample model rollouts by fidelity;  VIPO~\citep{chen2025vipo} augments the standard MLE loss with a value-consistency objective that aligns behavior-policy values under model and true dynamics.

A different line of work jointly trains adversarial world models and policies, rather than freezing the models after pretraining, which we refer to as \textbf{adversarial model-based pessimism} in \autoref{app:bg_conservatism}.
Several adversarial objectives have been explored: RAMBO~\citep{rigter2022rambo} minimizes value estimates; \citet{yang2022unified} minimizes the divergence between real and imagined state-action distributions; ARMOR~\citep{bhardwaj2023adversarial} minimizes value differences between the current and a reference policy; \citet{rigter2023one} extends the uncertainty set to incorporate aleatoric uncertainty.

Our work differs from these prior lines of research in principle, core components, and algorithmic design. We replace conservatism with Bayesianism as the guiding principle, substituting the standard Markov actor-critic with a history-dependent one required by Bayesian principle, and derive the \algo algorithm driven by long-horizon planning, fundamentally distinct from prior methods. Moreover, prior works typically rely on tuning two critical hyperparameters: the \textit{conservatism coefficient} $\lambda$ and \textit{rollout horizon} $H$, to achieve strong performance. In contrast, \algo eliminates the need for $\lambda$ by design and replaces $H$ with an uncertainty quantile $\zeta$, which remains fixed at $1.0$ in main experiments.

\textbf{Conservative decision-time planning algorithms.} 
In this line of work, a planner is used at test time to select actions from the current state. The planner is composed of a world model, a conservative policy obtained by behavior cloning~\citep{argenson2020model}, and a conservative value function learned by fitted Q evaluation~\citep{zhan2021model} or other pessimistic estimators~\citep{janner2021offline}. Given the conservative policy as a prior, the planner samples exploratory trajectories from world models and chooses the action with the highest estimated value. Decision-time planning can discover better actions at test time by performing targeted exploration from a specific state, particularly useful in unseen states. It is distinct from classical RL because it does not learn a global policy to maximize returns; instead, it directly searches for actions.

\subsection{Offline Bayesian-Inspired and Non-Conservative RL}

\textbf{Bayesian-inspired algorithms.} 
Several offline RL works draw inspiration from Bayesian ideas, such as using model posteriors and connections to Bayes-adaptive MDPs~\citep{duff2002optimal}. However, their formulations or algorithms typically differ from optimizing the epistemic POMDP in \autoref{eq:bayes_obj}, which is the focus of our work.
For example, although \citet{ghosh2022offline} introduces the Bayesian model-based formulation in \autoref{eq:bayes_obj}, their proposed APE-V algorithm adopts a \textit{conservative model-free} approximation. Rather than maintaining a posterior over dynamics models, APE-V uses an ensemble of belief-state Q-functions as a surrogate posterior over values, trained purely via TD errors. Each Q-function is trained using SAC-N~\citep{an2021uncertainty}, which enforces conservatism by minimizing over Q-ensemble predictions.

Similar to our algorithm, \textbf{MAPLE}~\citep{chen2021offline} and \textbf{MoDAP}~\citep{choi2024diversification} learn an ensemble of models for recurrent policy training. Unlike our approach, both algorithms perform short-horizon rollouts ($\le 20$ steps). Moreover, both store the hidden states of recurrent policies in the replay buffer and reuse them to initialize the policy’s context for rollouts and updates. This design is known to induce context staleness in recurrent RL~\citep{kapturowski2018recurrent}, whereas our algorithm avoids this issue by storing and sampling full histories directly.
MAPLE also reintroduces an uncertainty penalty and terminates rollouts based on a predefined state boundary, making it a conservative method. MoDAP remains penalty-free but fine-tunes the world models during policy learning to maintain ensemble diversity; therefore, its objective departs from the epistemic POMDP, which requires freezing the posterior during policy optimization. In contrast, our work aims to stay as close as possible to the epistemic POMDP: we avoid explicit conservatism, develop design choices for long-horizon rollouts, and do not fine-tune models during policy learning. 

Other works are Bayesian-inspired in different ways. CBOP~\citep{jeong2022conservative} uses a model posterior to weight multi-step TD targets in an MVE-style update~\citep{feinberg2018model} and applies lower-confidence penalties. BA-MCTS~\citep{chen2026bayes} proposes a Bayesian Monte Carlo planning method with an uncertainty penalty, used as a policy-improvement operator.

\textbf{Offline RL algorithms without explicit conservatism.} 
Although conservatism dominates modern offline RL, classic offline (batch) RL was originally developed without any conservatism~\citep{lagoudakis2003least,ernst2005tree,riedmiller2005neural}. These methods are based on \textit{fitted Q-iteration}, which directly applies Markov Q-learning to offline data. While effective on small-scale problems with sufficient data coverage~\citep{riedmiller2005neural}, such algorithms are known to fail in high-dimensional settings~\citep{fujimoto2019off}.
More recently, \citet{agarwal2020optimistic} provides an \emph{optimistic} perspective, showing that standard off-policy RL trained on the 50M transitions from DQN’s replay buffer can outperform the behavior policy. Likewise, \citet{yarats2022don} demonstrates that off-policy RL can surpass conservative methods on diverse-coverage datasets collected by unsupervised agents. 

In the model-based setting, MBPO~\citep{janner2019trust}, using short-horizon rollouts without conservatism, is known to underperform conservative counterparts such as MOPO in offline RL benchmarks~\citep{yu2020mopo}. However, MuZero Unplugged~\citep{schrittwieser2021online} shows that MuZero, without conservatism, can achieve strong performance on the RL Unplugged suite~\citep{gulcehre2020rl}, which contains 200M Atari transitions and DMC control datasets using replay buffer from near-optimal RL agents.  
\citet{zhai2024optimistic} observes that making MOPO optimistic instead of pessimistic can achieve strong performance on halfcheetah-random-v2, but worse on other D4RL tasks. 

In summary, prior work on offline RL without explicit conservatism largely relies on standard off-policy RL, which treats offline RL as optimizing a single MDP and thus follows an \textit{optimistic} principle~\citep{agarwal2020optimistic}. Such optimism can work well when the dataset has broad state-action coverage but typically fails under limited coverage. Our approach takes a different non-conservative path: instead of being optimistic, it follows a neutral Bayesian principle that lies \textbf{between optimism and pessimism}. This allows \algo to avoid the failure modes of optimistic methods and, for the first time to our knowledge, extends the effectiveness of non-conservative offline RL to \textit{low-quality} and \textit{moderate-coverage} datasets, while still benefiting from diverse coverage when available.

\textbf{Test-time adaptation and offline-to-online RL.} 
Finally, our method performs test-time adaptation, but it differs from prior approaches in its mechanism. For example, \citet{hong2023confidenceconditioned, swazinna2023userinteractive} adapt at deployment by dynamically tuning the degree of conservatism without retraining. In contrast, our Bayes-adaptivity does not adjust hyperparameters or conservatism at test time. Instead, it arises purely from the policy's in-context memory: by conditioning on observed history, the recurrent policy maintains an implicit belief about the true MDP, adapting its behavior intra-episode without any parameter updates.

While one could view in-context memory as inducing an implicit policy improvement~\citep{moeini2025survey}, we do not treat this mechanism as explicit online learning in this work (i.e., offline-to-online RL~\citep{nair2020awac}). In offline-to-online setting, \citet{sentenac2025balancing} analyzes bandit problems and shows that optimism can be preferable to pessimism during online learning, depending on the available interaction budget. While their focus is on a different regime, this observation is conceptually related to our finding that pessimism can hinder online adaptation. 

\subsection{Model-Generated Rollouts in RL}

\textbf{Effect of rollout length on value estimation.} 
\citet{sims2024edge} identifies the \textit{edge-of-reach} problem in offline conservative model-based RL (MBRL): when short-horizon rollouts are used, bootstrapping often occurs on states whose Q-values are never directly trained, leading to substantial overestimation once uncertainty penalties vanish (e.g., under true dynamics). Closely related to our work, they emphasize that such overestimation induced by short rollouts constitutes a distinct failure mode in offline MBRL, separate from classic compounding errors. To mitigate this, they replace dynamics-based uncertainty penalties with value-based ones via pessimistic Q-ensembles~\citep{an2021uncertainty}, while still relying on short-horizon rollouts.

Our work extends this line of reasoning in two ways. (1) We show that the same overestimation mechanism arises under Bayesian Bellman backups and extend \citet[Proposition 1]{sims2024edge} to Bayesian setting in \autoref{app:proof_bootstrap}. (2) Instead of conservative short-horizon rollouts, we use adaptive long-horizon rollouts that exploit the vanishing factor $\gamma^H$ to naturally reduce the bootstrapped error and remove the need for conservatism.

\textbf{Reducing compounding errors and the scale of rollout horizon.} 
Model-generated rollouts suffer from compounding prediction errors~\citep{talvitie2014model,lambert2022investigating}, which can grow quickly with the rollout horizon. These errors cause the performance gap between the policy evaluated under the learned model and under the true MDP, as formalized by theory such as the simulation lemma (see \citet[Lemma 9]{uehara2021pessimistic} and our \autoref{lemma:simulation_lemma}). To prevent large compounding errors from harming policy learning, most online and offline MBRL methods~\citep{janner2019trust,yu2020mopo,lu2021revisiting,hafner2023mastering,hansen2024td} restrict rollouts to very short horizons (typically \textbf{1-20 steps}). These approaches often share a minimalist setup: standard MLP world models without layer normalization, pure MLE training, and rollout procedures without uncertainty awareness.

Prior works attempt to reduce compounding errors and thus enable longer rollouts along three ways: \textit{model architecture}, \textit{training objective}, and \textit{inference strategy}.
On the architectural side, improving model smoothness can enhance generalization on unseen states~\citep{asadi2018lipschitz}. 
On the training-objective side, multi-step architecture predicts future states given an initial state and an action sequence~\citep{asadi2019combating}, enabling horizons up to 500 steps with Transformers in online RL~\citep{ma2024transformer} and up to 50 steps with RNNs in offline RL~\citep{lin2025any}. 
For inference strategies, MOREC~\citep{luo2024reward} resamples model rollouts using a fidelity estimator and scales horizons to 100 steps.\looseness=-1

\algo keeps the training objective fixed to standard one-step MLE, while contributing simple and effective components along the other two dimensions. Architecturally, we apply layer normalization~\citep{ba2016layer} to stabilize prediction magnitudes (\autoref{sec:ensemble_arch}), increasing model smoothness~\citep{asadi2018lipschitz} and thus generalization; at inference time, we truncate rollouts using an uncertainty threshold as a proxy for compounding error (\autoref{sec:how_longhorizon}). Together, these components enable \algo to successfully use horizons of \textbf{64-512} steps, which is rare in MBRL literature.

\textbf{Mechanism of adaptive horizon.} 
As discussed in \autoref{sec:how_longhorizon}, several prior works have used uncertainty thresholds to adaptively truncate rollouts. In the online RL setting, M2AC~\citep{pan2020trust} truncates rollouts for states whose uncertainty ranks in the top 25\% of the current-step batch and caps the horizon at $10$. Similarly, MACURA~\citep{frauenknecht2024trust} computes a 95\% uncertainty quantile from the batch of first-step predictions in current rollouts, also with a maximum horizon of $10$. Infoprop~\citep{frauenknecht2025rollouts} extends MACURA by using accumulated uncertainty along the rollout as an additional truncation criterion and increases the maximum horizon to around $50$.
In the offline RL setting, MOPP~\citep{zhan2021model} adopts an 85\% uncertainty quantile from the offline dataset to filter rollouts and imposes a maximum horizon of $16$. 
TATU~\citep{zhang2023uncertainty} uses truncation thresholds proportional to the maximum uncertainty in the offline dataset, with a coefficient of $0.5$ and a maximum horizon of $5$. 

A key difference is the \emph{role} of truncation. Prior work uses adaptive horizons conservatively to keep rollouts short and limit compounding model error. In contrast, \algo is motivated by the observation that, once explicit conservatism is removed, \emph{long horizons become necessary} to reduce overestimation by reducing reliance on bootstrapping. Thus, we use uncertainty-based truncation to preserve the longest reliable rollouts possible, rather than to enforce short ones.

\subsection{Bayesian and Partially Observable RL}

\textbf{Bayesian RL.}
Bayesian RL~\citep{vlassis2012bayesian,ghavamzadeh2015bayesian} models epistemic uncertainty (also known as ambiguity~\citep{ellsberg1961risk}) for purposes such as exploration~\citep{osband2016deep}, robustness~\citep{rajeswaran2017epopt,derman2020bayesian,rigter2021risk}, and generalization~\citep{ghosh2021generalization,jiang2023importance}. Uncertainty can be incorporated in model-free methods (e.g., Bayesian Q-learning~\citep{dearden1998bayesian}) or in model-based methods (e.g., BAMDPs~\citep{duff2002optimal}). Depending on the design preference, Bayesian RL can be ambiguity-seeking, ambiguity-neutral, or ambiguity-averse.
Our work is grounded in the epistemic POMDP formulation~\citep{ghosh2021generalization,ghosh2022offline}, an ambiguity-neutral, model-based framework for generalization through adaptation. This formulation stems from BAMDPs, which optimally balance the exploration-exploitation tradeoff but incur test-time exploratory costs, as illustrated in \autoref{sec:bandit} and also known as the \textit{cost of exploration}~\citep{vuorio2024bayesian}. 

Epistemic POMDPs are also related to \textbf{meta-RL}~\citep{beck2023survey}, which likewise seeks to maximize expected performance over a distribution of MDPs. Each MDP is called a task in meta-RL. Bayesian meta-RL~\citep{zintgraf2020varibad,dorfman2021offline} explicitly models this MDP posterior. The key difference is that in meta-RL the true environment is itself a pre-specified distribution over MDPs (a particular POMDP), so the task uncertainty is \textit{inherent}. In contrast, the epistemic POMDP assumes a single underlying MDP, and its task uncertainty is purely epistemic. As a result, meta-RL methods cannot be applied to solve the epistemic POMDP without modification.

\textbf{Recurrent model-free RL for online POMDPs.} Model-free RL offers a simple and effective way to tackle online POMDPs without explicitly learning belief-state representations~\citep{ni2022recurrent,yang2021recurrent}. Memory is typically implemented with recurrent neural networks (RNNs) such as LSTMs~\citep{hochreiter1997long} or GRUs~\citep{cho2014learning}, but recent work shows that these architectures struggle with long-term memory~\citep{parisotto2020stabilizing,ni2023transformers}. State-space models (SSMs) with linear recurrence~\citep{orvieto2023resurrecting,gu2023mamba} have emerged as a compute-efficient alternative to Transformers~\citep{vaswani2017attention}, balancing long-term memory with parallel optimization~\citep{blelloch1990prefix}. Recent applications of SSMs to recurrent RL~\citep{lu2023structured,lu2024rethinking,morad2024recurrent,luo2024efficient,luis2024uncertainty} demonstrate strong performance on online POMDP benchmarks~\citep{morad2023popgym,zintgraf2020varibad,ni2022recurrent}.

\subsection{Concurrent Work}

ADM-v2~\citep{lin2026admv} is a concurrent work that, like ours, demonstrates the effectiveness of long-horizon rollouts for offline model-based RL. ADM-v2 pursues an extreme \textbf{full-horizon} setting, performing rollouts up to the maximum episode length (e.g., $1000$ steps in D4RL locomotion tasks) and thus treating the learned world model as a surrogate simulator. Building on ADMPO~\citep{lin2025any}, it simplifies multi-step modeling and introduces a parallelized rollout mechanism, enabling efficient any-step prediction. With full-horizon rollouts, ADM-v2 can perform off-policy evaluation directly without value bootstrapping and achieves strong performance on popular benchmarks when combined with an uncertainty penalty, representing the first successful full-horizon method.

We view ADM-v2 as highly complementary to our work. While ADM-v2 demonstrates that full-horizon rollouts can be effective under a \textit{conservative} design, our method explores the opposite extreme: adhering to a Bayesian, \textit{non-conservative} objective while using adaptive (but not full) horizons. Interestingly, ADM-v2’s sensitivity analysis~\citep[Figure 16]{lin2026admv} shows that long horizons are most beneficial under strong penalties, whereas in \textit{lightly penalized} regimes, excessively long horizons can degrade performance. This aligns with our findings: without explicit conservatism, performance is more sensitive to compounding errors, motivating adaptive truncation rather than full-horizon rollouts. Together, these two works suggest that long-horizon planning is a powerful component in offline RL, whose effective use depends on the degree of conservatism.

\section{Extended Background on Conservative and Bayesian Principles}

\subsection{Formal Connections between Conservatism and Robustness}
\label{app:bg_conservatism}

In this subsection, we place prior conservative  algorithms in the (soft) robust MDP framework. For classic model-free and adversarial model-based pessimism, we make explicit how their updates correspond to particular uncertainty sets in \textit{robust} MDPs~\citep{wiesemann2013robust}. For uncertainty-penalized pessimism, we connect it to \textit{soft robust} MDP framework~\citep{zhang2024soft}.\footnote{The term ``soft robustness'' is used differently in prior work: \citet{derman2018soft} use it for a Bayesian formalism, while \citet{zhang2024soft} define it as a risk-sensitive MDP relaxing strict worst-case robustness.} We follow the notation introduced in \autoref{sec:prelim}.

\textbf{Classic model-free pessimism.} Many model-free methods enforce an in-support~\citep{fujimoto2019off,kumar2019stabilizing} or in-sample~\citep{kostrikov2021offline,xu2023offline} constraint on the Bellman backup. This amounts to updating a pessimistic value function $Q^{\mathrm{MF}}$ such that, for a transition tuple $(s,a,r,d,s') \in \D$,
\begin{equation}
\label{eq:model-free-pess}
Q^{\mathrm{MF}}(s,a) \gets r + \gamma (1-d) \max_{a' \in \A,\, \text{s.t.}\, (s',a')\in \D}{Q^{\mathrm{MF}}(s',a')}. 
\end{equation}
This update is equivalent to assigning all out-of-dataset state–action pairs the minimal reward $-r_{\max}$, thereby enforcing a worst-case behavior. The corresponding uncertainty set $\Mf_\D$ in the robust MDP framework (\autoref{eq:robust_obj}) can be written explicitly:

\begin{proposition}
\label{prop:model-free-pess}
If the pessimistic update in \autoref{eq:model-free-pess} converges to a fixed point, then the induced uncertainty set for the corresponding robust MDP is
\begin{equation}
\label{unc:model-free-pess}
\Mf_\D^{\mathrm{MF}} = \Biggl\{m \;\Bigg|\;
\begin{aligned}
&m(r,s'\mid s,a) = m_\D(r,s'\mid s,a), && \forall (s,a) \in \D \\
&m(r,s'\mid s,a) = p(r)\,\mathbbm{1}(s'=s_{\text{absorb}}),\;\; \forall p \in \Delta([-\rmax,\rmax]), && \forall (s,a) \notin \D
\end{aligned}
\Biggr\},
\end{equation}
where $m_\D$ is the empirical model and $s_{\text{absorb}} \not \in \D$ is an artificial absorbing state.
\end{proposition}

\begin{proof}
First, let $m \in \Mf_\D^{\mathrm{MF}}$ and decompose $m(r,s'\mid s,a) = R(r\mid s,a) P(s'\mid s,a)$, and denote the empirical model as $m_\D(r,s'\mid s,a) = R_\D(r\mid s,a) P_\D(s'\mid s,a)$. 
By construction, $\Mf_\D^{\mathrm{MF}}$ places no uncertainty on transitions: for $(s,a)\in\D$, $P(s,a)$ equals the empirical transition $P_\D(s,a)$, while for $(s,a)\notin\D$, $P(s,a)$ deterministically transitions to $s_{\mathrm{absorb}}$. Thus, all epistemic uncertainty is in the reward function $R$. 

Applying the robust MDP framework~\citep{wiesemann2013robust}, the optimal value function $Q^*$ is Markov and satisfies the robust Bellman optimality equation:
{\small
\begin{align}
&Q^*(s,a) = \min_{R(s,a) \in \mathcal \Mf_\D^{\mathrm{MF}}(s,a)} \E{r \sim R(s,a)}{r} + \gamma \E{s'\sim P(s,a)}{\max_{a'\in \A}Q^*(s',a')}, \forall (s,a)\in \S\times\A, \\
\label{eq:mf_rmax}
&Q^*(s,a) -  \gamma \E{s'\sim P(s,a)}{\max_{a'\in \A} Q^*(s',a')} =  \min_{R(s,a) \in \mathcal \Mf_\D^{\mathrm{MF}}} \E{r \sim R(s,a)}{r} = 
\begin{cases}
\E{r \sim R_\D(s,a)}{r} & (s,a) \in \D, \\
-\rmax & (s,a) \not \in \D,
\end{cases}
\end{align}%
}where the last line uses the fact that $R(\cdot\mid s,a)$ may be any distribution on $[-\rmax,\rmax]$, so the worst case is attained by a Dirac mass at $-\rmax$.

Therefore, we can simplify \autoref{eq:mf_rmax} by cases.  (1) Absorbing state:  
\begin{align}
&\forall a\in \A, \quad Q^*(s_{\text{absorb}},a) -  \gamma \max_{a'\in \A} Q^*(s_{\text{absorb}},a')  = -\rmax.
\end{align}
Since $Q^*(s_{\text{absorb}},a)$ is a constant w.r.t. $a$, it follows that $Q^*(s_{\text{absorb}},a) = -\frac{\rmax}{1-\gamma}, \forall a$, reaches the minimal return.  (2) Unseen state-action pairs,
\begin{align}
\forall (s,a) \not\in \D, \quad Q^*(s,a) -  \gamma \max_{a'\in \A} Q^*(s_{\text{absorb}},a')  = -\rmax.
\end{align}
This implies $Q^*(s,a) = -\frac{\rmax}{1-\gamma}, \forall (s,a) \not\in \D$. (2) Seen state-action pairs,  
\begin{align}
\forall (s,a) \in \D, \quad Q^*(s,a)  &= \E{r \sim R_\D(s,a)}{r}+ \gamma \E{s'\sim P_\D(s,a)}{\max_{a'\in \A} Q^*(s',a') } \\
\label{eq:BCQ}
&= \E{r \sim R_\D(s,a)}{r}+ \gamma \E{s'\sim P_\D(s,a)}{\max_{a'\in\A, \text{s.t.}\, (s',a')\in \D} Q^*(s',a') }
\end{align}
The last line follows that $
Q^*(s,a_{\text{out}}) = -\frac{\rmax}{1-\gamma} \le Q^*(s,a_{\text{in}}), \forall (s,a_{\text{in}}) \in \D, \forall (s,a_{\text{out}}) \not \in \D.
$
Finally, \autoref{eq:BCQ} recovers the pessimism principle underlying \autoref{eq:model-free-pess}. This includes many model-free offline RL algorithms, such as BCQ~\citep[Equation 10]{fujimoto2019off}, BEAR~\citep[Definition 4.1]{kumar2019stabilizing}, EMaQ~\citep[Theorem 3.3]{ghasemipour2021emaq}, IQL~\citep[Corollary 2.1]{kostrikov2021offline}.
\end{proof}

\textbf{Adversarial model-based pessimism.} One class of model-based methods constructs an explicit uncertainty set around the empirical model~\citep{uehara2021pessimistic,rigter2022rambo}:
\begin{equation}
\Mf_\D^{\text{MB}} = \{ m \mid \E{(s,a)\sim \D}{\mathrm{div}(m(s,a), m_\D(s,a)) } \le \epsilon\},
\end{equation}
where popular choices of the divergence $\mathrm{div}(\cdot, \cdot)$ include total variation (TV) distance and KL divergence.

\textbf{Uncertainty-penalized pessimism: soft robust MDP.}
Another line of model-based methods incorporates explicit uncertainty penalties into value updates~\citep{yu2020mopo,kidambi2020morel,jeong2022conservative,sun2023model}. For imagined transitions $(\hat s, \hat a, \hat r, \hat d, \hat s')$, the pessimistic update takes the form:
\begin{equation}
\label{eq:unc_penalty}
Q^{\text{MB}}(\hat s,\hat a) \gets \hat r - \lambda U(\hat s, \hat a)+ \gamma (1-\hat d) \max_{\hat a'\in \A}{Q^{\text{MB}}(\hat s',\hat a')},
\end{equation}
where $U:\mathcal S\times \mathcal A \to \mathbb R^+$ is an uncertainty measure based on dataset $\D$ and learned world models, and $\lambda > 0$ controls the degree of pessimism. Similar uncertainty penalties have also been incorporated into model-free value functions~\citep{bai2022pessimistic,an2021uncertainty}.

We now provide a connection between \autoref{eq:unc_penalty} and the \textit{soft robust MDP} framework of~\cite[Section 5]{zhang2024soft}. While our derivation follows a similar line to theirs, we include it here to be self-contained. Consider a robust MDP with a \textit{policy-dependent} uncertainty set:
\begin{equation}
\max_{\pi}\min_{m}J(\pi,m)\quad\text{s.t.}\quad
\E{(s,a)\sim (m, \pi)}{\mathrm{div}(m(s,a), m^*(s,a))}\le \epsilon,
\end{equation}
where the divergence constraint describes the uncertainty set, taken in expectation under occupancy measure induced by $m$ and $\pi$. We use Lagrangian relaxation with a coefficient $\alpha \ge 0$ to transform the problem into a soft robust MDP: 
\begin{equation}
\max_{\pi}\min_{m} J(\pi, m) + \alpha \E{(s,a)\sim (m, \pi)}{\mathrm{div}(m(s,a), m^*(s,a))}.
\end{equation}
In dynamic programming form, for a given $(s,a)$ pair, the inner optimization becomes
\begin{equation}
\label{eq:inner_optim}
\min_{m(\cdot \mid s,a)}\ \E{s'\sim m(\cdot \mid s,a)}{V^{\text{MB}}(s')}\ + \alpha\,\mathrm{div}(m(\cdot \mid s,a),m^*(\cdot \mid s,a)),
\end{equation}
where $V^{\text{MB}}$ is the policy's state-value function. 
This inner problem can be transformed by duality, depending on the choice of divergence. For the KL divergence, one can apply Donsker and Varadhan's formula~\citep{donsker1975asymptotic}\footnote{For any probability distributions $x,x^*\in \Delta^k$ (the $k$-dimensional simplex), any vector $y\in \R^k$, and $\alpha > 0$, the duality formula is $-\alpha \log  (\langle x^*, \exp(-y / \alpha)  \rangle) =  \min_x \langle x, y\rangle+ \alpha\,\textsc{kl}(x \mid \mid x^*)$.} to state that \autoref{eq:inner_optim} is equivalent to
{\small
\begin{equation}
-\alpha \log \left(\E{s'\sim m^*(\cdot \mid s,a)}{\exp\left(-\frac{V^{\text{MB}}(s')}{\alpha}\right)} \right)= \E{s'\sim m^*}{V^{\text{MB}}(s')} -\frac{1}{2\alpha} \mathrm{Var}_{m^*}[V^{\text{MB}}(s')] + O\left(\frac{1}{\alpha^3}\right),
\end{equation}%
}where we use cumulant expansion.\footnote{For a random variable $X$, $\log (\E{}{\exp(t X)}) = t\E{}{X} + \frac{t^2}{2} \mathrm{Var}[X] + O(t^3)$. We substitute $t=-1/\alpha$.}

Therefore, the corresponding soft-robust Bellman optimality equation~\citep[Equation 15]{zhang2024soft}, ignoring higher-order terms, becomes
\begin{equation}
\label{eq:soft-robust_eq}
Q^{\text{MB}}(s,a) = R^*(s,a)  - \frac{\gamma}{2\alpha} \mathrm{Var}_{s'\sim P^*}[\max_{a'}Q^{\text{MB}}(s', a')] + \gamma \E{s'\sim P^*}{\max_{a'}Q^{\text{MB}}(s', a')}
\end{equation}

In practice, model-based RL methods approximate $m^* = (R^*, P^*)$ with an ensemble of learned models trained on $\mathcal D$. While the variance term in \autoref{eq:soft-robust_eq} reflects the aleatoric uncertainty of the true dynamics, ensemble-based variance also incorporates epistemic uncertainty by the law of total variance, thus blending both. Prior work has interpreted the penalty as aleatoric~\citep{yu2020mopo}, epistemic~\citep{sun2023model}, or both~\citep{rigter2023one}; here we focus on its epistemic interpretation. Accordingly, the ensemble variance provides a practical surrogate for the penalty in \autoref{eq:unc_penalty}.

\subsection{Connection between Bayesianism and Partial Observability}
\label{app:bg_bayesian}

\textbf{Epistemic POMDP as a special class of POMDP.} As noted by \citet[Appendix A]{ghosh2022offline}, the Bayesian objective (\autoref{eq:bayes_obj}) can be cast as a POMDP~\citep{cassandra1994acting}. We adapt their proof here. The POMDP's state space is $\S^+ = \S \times \Mf_\D$ with the same action space $\A$, where $\Mf_\D = \supp(\Pr_\D)$. 
The joint reward–transition function in the POMDP is 
$$\Pr (r_{t+1}, s_{t+1}^+ \mid s_t^+, a_t) = \mathbbm{1}(m_{t+1} = m_t) m_t(r_{t+1} , s_{t+1}  \mid s_t, a_t),$$ where the initial state is $s_0^+ \defeq (s_0, m_0)$ with $s_0 \sim \rho$ and $m_0 \sim \Pr_{\D}$. 
It is partially observable because the agent only observes $s_t \in s_t^+$, while the model $m_t \equiv m_0 \in \Mf_\D$ remains hidden but fixed throughout each episode.

\section{Analysis of Bootstrapped Error in Long-Horizon Rollouts}
\label{app:proof_bootstrap}

We formalize the insight in \autoref{sec:why_longhorizon} using a result adapted from \citet[Proposition 1]{sims2024edge}.

Consider a rollout generated by a deterministic world model $m$ and a deterministic policy $\pi$, $$\tau = (h_t,\hat a_t, \hat r_{t+1}, \hat s_{t+1}, \hat a_{t+1}, \dots, \hat s_{t+H}),$$ where the initial history $h_t$ is drawn from the dataset $\D$, $\hat a_{t+j} = \pi(\hat h_{t+j})$, $(\hat r_{t+j+1}, \hat s_{t+j+1}) = m(\hat s_{t+j}, \hat a_{t+j})$ and $\hat h_t = h_t$. 

Assume a tabular policy evaluation setting and let $Q^{\pi}$ denote the exact value function of policy $\pi$ under the MDP $m$. Define the Bellman evaluation operator $\mathcal T$: for a value function $Q$, 
$$
\mathcal T Q (\hat h_{t+j}, \hat a_{t+j})\defeq \hat r_{t+j+1} + \gamma Q(\hat h_{t+j+1}, \pi(\hat h_{t+j+1})).
$$
We perform the value update along the imagined rollout backward in time. Let $Q_0$ be the initial value function prior to value update and $Q_j$ be the value function at $j$-th iteration. For $j= 1, \dots, H$, define $Q_j(\hat h_{t+H-j}, \hat a_{t+H-j})\approx \mathcal T Q_{j-1} (\hat h_{t+H-j}, \hat a_{t+H-j})$, with a per-step \textbf{TD error} bounded by: 
$$
| Q_j(\hat h_{t+H-j}, \hat a_{t+H-j}) - \mathcal T Q_{j-1} (\hat h_{t+H-j}, \hat a_{t+H-j})| \le \delta_j.
$$

Following \citet{sims2024edge}, assume that $\hat h_{t+H}$ is an \textbf{edge-of-reach} history, i.e., a history used as Bellman targets but never itself updated. The \textbf{bootstrapped error} at this truncated point is
$$
\epsilon \defeq | Q_0(\hat h_{t+H},\pi(\hat h_{t+H})) - Q^\pi(\hat h_{t+H},\pi(\hat h_{t+H})) |.
$$
We can decompose the total error for $Q_j$: for $j = 1, \dots, H$,
\begin{equation}
\begin{split}
\xi_j &\defeq |Q_j(\hat h_{t+H-j}, \hat a_{t+H-j}) - Q^\pi(\hat h_{t+H-j}, \hat a_{t+H-j}) | \\
&\le |Q_j(\hat h_{t+H-j}, \hat a_{t+H-j}) -  \mathcal T Q_{j-1} (\hat h_{t+H-j}, \hat a_{t+H-j}) | \\
& + |\mathcal T Q_{j-1} (\hat h_{t+H-j}, \hat a_{t+H-j}) - Q^\pi(\hat h_{t+H-j}, \hat a_{t+H-j}) |\\
&\le \delta_j + | \gamma Q_{j-1} (\hat h_{t+H-j+1}, \hat a_{t+H-j+1}) - \gamma Q^\pi (\hat h_{t+H-j+1}, \hat a_{t+H-j+1})|\\
&= \delta_j + \gamma \xi_{j-1}.
\end{split}
\end{equation}
Unrolling the recursion and using the fact that $\xi_0 = \epsilon$, the value at the final iteration is bounded by: 
\begin{equation}
|Q_H(h_t,\hat a_t) - Q^\pi(h_t,\hat a_t) | = \xi_H  \le \sum_{j=0}^{H-1} \gamma^j \delta_j + \gamma^H \epsilon.
\end{equation}

The bound above is sign-agnostic, but under offline policy improvement the bootstrapped term at the truncation point tends to be \emph{optimistic}. Near-greedy policies evaluate the critic on out-of-distribution actions that require extrapolation, where neural critics can assign spuriously high values. Even though edge-of-reach histories are not directly updated, their values can drift through function approximation and shared parameters, allowing optimistic errors to propagate. As a result, the terminal bootstrap often dominates value overestimation for short horizons, while increasing $H$ mitigates this effect by exponentially down-weighting it.

\section{Formal Existence Proof of the Advantage of Bayesianism}
\label{app:theory}

In this section, we compare the conservative and Bayesian principles in offline RL and highlight conditions under which the Bayesian approach yields provable advantages. After restating the conservative objective and its reliance on data coverage, we introduce a relaxed notion of coverage from a Bayesian view. This allows us to lower-bound the performance gap between Bayesian and robust solutions. We begin by defining the optimal policies in their respective objectives:
\begin{align}
&\pi^*(m^*)\in\argmax_\pi J(\pi, m^*), \tag{ideal policy} \\
&\pi^*(\Mf_{\D}) \in \argmax_\pi J(\pi; \Mf_{\D}), \tag{robust-optimal policy} \\
& \pi^*(\Pr_{\D}) \in \argmax_{\pi}\E{m\sim \Pr_{\D}} {J(\pi, m)}. \tag{Bayes-optimal policy}
\end{align}

We then introduce the notion of \textit{robust sub-optimality gap}, measuring the gap of the robust-optimal policy relative to the ideal one:
\begin{align}
    S_{\Mf_{\D}}(m^*) &:= J(\pi^*(m^*), m^*) - J(\pi^*(\Mf_{\D}), m^*).
\end{align}
To assess when $\pi^*(\Mf_\D)$ is competitive, one seeks an upper bound on this gap. Its tightness depends on dataset quality: the closer the data coverage is to $\pi^*(m^*)$, the smaller the gap.
Prior work formalizes this dependence via coverage assumptions~\citep{uehara2021pessimistic,li2024settling}. We unify these notions through the following definition.

\subsection{Comparison on Concentrability Coefficients}
\label{app:concen}
In this subsection, we define the robust and Bayesian concentrability coefficients and derive the numerical relationship between the two.

\begin{definition}[Model-dependent concentrability]
Given an MDP model $m$ and a policy $\pi$, let $d_{m}^{\pi}$ be the state-action occupancy measure of $\pi$ on $m$, and let $\beta \in \Delta(\S\times \A)$ be the offline distribution induced by the dataset $\D$ on $m^*$.
The concentrability coefficient of $\pi$ in MDP $m$ under offline distribution $\beta$ is defined as:
\begin{align}
\label{eq:concen}
\mathcal C(\pi, m) \defeq \frac{\E{(s,a)\sim d_{m}^{\pi}}{\textsc{tv}(m(s,a), m^*(s,a))^2}}{\E{(s,a)\sim \beta}{\textsc{tv}(m(s,a), m^*(s,a))^2}},
\end{align}
where $\textsc{tv}$ denotes the total variation distance between distributions.
\end{definition}
Intuitively, $\mathcal C(\pi,m)$ quantifies the mismatch between $m$ and $m^*$ under the policy distribution versus the dataset distribution. 

We then extend the \textit{model-based concentrability} of \citet[Definition~1]{uehara2021pessimistic}, originally defined only for transition dynamics, to also incorporate the reward function.
We refer to this generalization as robust concentrability.

\begin{definition}[\textbf{Robust concentrability}~\citep{uehara2021pessimistic}]
Let $\Mf_{\D}$ be a realizable hypothesis class of models consistently built from the dataset $\D$, 
\begin{align}
\label{eq:robust_concen}
\mathcal C(\pi) \defeq \sup_{m\in\Mf_{\D}} \mathcal C(\pi, m).
\end{align} 
\end{definition}

If $m^*\in\Mf_{\D}$, the robust concentrability is upper bounded by   the classic density-ratio-based concentrability, as shown in \citet[Section 4]{uehara2021pessimistic}:
$$
\mathcal C(\pi) \le \sup_{(s,a)\in\S\times\A} \frac{{d^{\pi}_{m^*}(s,a)}}{{\beta(s,a)}}.
$$

We now consider an even \textit{weaker} notion of coverage by extending model-dependent concentrability with a Bayesian posterior over models.
\begin{definition}[\textbf{Bayesian concentrability}]
\begin{align}
\label{eq:bayes_concen}
    \mathcal C_{\text{Bayes}}(\pi) \defeq \frac{\E{m\sim \Pr_{\D}}{\E{(s,a)\sim d_{m}^{\pi}}{\textsc{tv}(m(s,a), m^*(s,a))^2}}}{\E{m\sim \Pr_{\D}}{\E{(s,a)\sim \beta}{\textsc{tv}(m(s,a), m^*(s,a))^2}}}.
\end{align} 
\end{definition}

We now show that Bayesian concentrability is always upper bounded by its robust counterpart.

\begin{proposition}[Bayesian concentrability is upper-bounded by robust concentrability]
\label{prop: bayes vs robust coverage}
Assume that $\E{m\sim \Pr_{\D}}{g(m)} > 0$, then
\begin{equation*}
C_{\text{Bayes}}(\pi) \le \sup_{m \in \supp(\Pr_\D)} \mathcal C(\pi, m).
\end{equation*}
\end{proposition}
\begin{proof}
Denote
$$
f(m) := \E{(s,a)\sim d_{m}^{\pi}}{\textsc{tv}(m(s,a), m^*(s,a))^2},
$$
$$
g(m) := \E{(s,a)\sim \beta}{\textsc{tv}(m(s,a), m^*(s,a))^2},
$$
so that
$$
\mathcal C(\pi, m) = \frac{f(m)}{g(m)}, 
\quad \mathcal C_{\text{Bayes}}(\pi) \defeq \frac{\E{m\sim \Pr_{\D}}{f(m)}}{\E{m\sim \Pr_{\D}}{g(m)}}.
$$

$$
\LHS =  \frac{\E{m\sim \Pr_{\D}}{g(m) \frac{f(m)}{g(m)}}}{\E{m\sim \Pr_{\D}}{g(m)}} \le \sup_{m\in \supp(\Pr_\D)} \frac{f(m)}{g(m)} = \RHS,
$$
by the elementary inequality $\sum_i w_i x_i \le \sum_i w_i \max_j x_j = (\sup_j x_j) \sum_i w_i$ for $w\geq 0$.
\end{proof}

The next example sharpens this bound to a strict inequality, making the Bayesian coverage assumption $\mathcal C_{\text{Bayes}}(\pi)  < \infty$ strictly weaker than its robust analogue $\mathcal C(\pi) < \infty$. 

\begin{example}[Strictness of Bayesian concentrability bound]
Let $\beta$ have incomplete coverage over $\S\times \A$. Following the notations from the proof of \autoref{prop: bayes vs robust coverage}, consider two models:
\begin{enumerate}
    \item Model $m_1$ is a small perturbation of  $m^*$ on all of
    $\S\times \A$. Then, the numerator $f(m_1) <\infty$, the denominator $g(m_1)>0$, making $\mathcal C(\pi,m_1)< \infty$. 
    \item Model $m_2$ is equivalent to $m^*$ on $\supp(\beta)$, i.e., $m_2(s,a) = m^*(s,a), \forall (s,a) \in \supp(\beta)$, 
    but there exists an off-support pair
    $(s^\dagger,a^\dagger) \not \in \supp(\beta)$ with $d_{m_2}^\pi(s^\dagger,a^\dagger)>0$ and
    $\textsc{tv}(m_2(s^\dagger,a^\dagger),m^*(s^\dagger,a^\dagger))>0$.
    In that case, $g(m_2)=0,f(m_2) > 0,$ so $ \mathcal C(\pi,m_2) =\infty.$
\end{enumerate}

If the posterior $\Pr_\D$ assigns weights $\Pr_\D(m_1) = 1-\varepsilon$ and $\Pr_\D(m_2) =  \varepsilon$ with any $\varepsilon \in (0,1)$, then
\[
\E{m\sim\Pr_{\D}}{f(m)} = (1-\varepsilon)f(m_1)+\varepsilon f(m_2)<\infty,
\quad
\E{m\sim\Pr_{\D}}{g(m)} = (1-\varepsilon)g(m_1) > 0,
\]
and $\mathcal C_{\text{Bayes}}(\pi)<\infty$, while 
$\sup_{m\in \supp(\Pr_\D)} \mathcal C(\pi,m)=\infty$. 
\end{example}

\subsection{\autoref{thm:final_bound} and Proof}
\label{app:proof}

To formalize the benefit of taking a Bayesian approach over a conservative one, we define the \textit{Bayesian sub-optimality gap} as:
\begin{align}
\label{eq:bayes_gap}
S_{\Pr_{\D}}(m^*) &= J(\pi^*(m^*), m^*) - J(\pi^*(\Pr_{\D}), m^*),
\end{align}
which we aim to compare with the robust gap via their difference:
\begin{align}
\Delta_{\Pr_{\D}, \Mf_{\D}}(m^*):= S_{\Mf_{\D}}(m^*) - S_{\Pr_{\D}}(m^*) = J(\pi^*(\Pr_{\D}), m^*) - J(\pi^*(\Mf_{\D}), m^*).
\end{align}
The theorem below establishes a lower bound on this difference.

\begin{theorem}
\label{thm:final_bound}
    Assume that $\gamma > 1/2$. Then, we can construct a dataset $\D$, a set of MDPs $\Mf_{\D}$, and a posterior $\Pr_{\D}$ such that $m^* \in \supp(\Pr_\D)\subseteq \Mf_\D$, and for some $\delta_0> 0$, it holds with probability at least $\delta_0$ that, 
    {\small
    \begin{align*}
        \Delta_{\Pr_{\D}, \Mf_{\D}}(m^*) >  \frac{8r_{\max}}{(1-\gamma)^2}\sqrt{\frac{\ln(|\Mf_{\D}|/\delta_0)}{|\D|}}\left(\sqrt{\mathcal C(\pi^*(\Mf_{\D}))} -\sqrt{\max(\mathcal C_{\text{Bayes}}(\pi^*(m^*)), \mathcal C_{\text{Bayes}}(\pi^*(\Pr_{\D})))}\right).
    \end{align*}
    }
\end{theorem}
\begin{proof}[Proof sketch]
The proof proceeds in two steps, each being established in \autoref{thm:bayes_subopt} and \autoref{thm:robust_subopt} respectively. (1) We upper-bound the Bayesian gap $S_{\Pr_\D}(m^*)$ (\autoref{thm:bayes_subopt}). 
This bound makes explicit the ``price of Bayesianism'': coverage adequate for the ideal policy $\pi^*(m^*)$ alone does not suffice. We also need reasonable coverage for the Bayes-optimal policy, i.e., $\mathcal C_{\text{Bayes}}(\pi^*(\Pr_{\D})) < \infty$, even though posterior averaging makes $\mathcal C_{\text{Bayes}}(\cdot)$ more likely to be finite.
(2) We lower-bound the robust gap $S_{\Mf_{\D}}(m^*)$ (\autoref{thm:robust_subopt}), by constructing an MDP with inadequate frequentist coverage relative to the behavior policy. 
\end{proof}
\begin{remark}
This theorem shows that when $\mathcal C(\pi^*(\Mf_{\D})) > \max(\mathcal C_{\text{Bayes}}(\pi^*(m^*)), \mathcal C_{\text{Bayes}}(\pi^*(\Pr_{\D})))$, the Bayes-optimal policy performs better than the robust-optimal policy in the true MDP. By construction, the lower bound in \autoref{thm:robust_subopt} holds when $\epsilon$ -- the parametric gap between two well-chosen MDPs -- is smaller than $c\cdot \mathcal C(\pi^*(\Mf_{\D}))/|\D|$. This is more likely to occur when the dataset is small or when the concentrability $\mathcal C(\pi^*(\Mf_{\D}))$ is large. In other words, adopting a frequentist (or conservative) approach rather than Bayesianism implies that the dataset does not provide sufficient statistical information when: (1) it is small, or (2) its coverage differs from that of the ground truth optimal policy. Those are the two features measuring the ``quality'' of an offline dataset. The Bayesian approach potentially yields better performance than conservatism because, for a given amount of data, the Bayesian concentrability is smaller than the robust one, making the lower bound irrelevant for the Bayesian setting.  
\end{remark}

\begin{lemma}[\textbf{Upper bound on Bayesian sub-optimality gap}]
\label{thm:bayes_subopt}
    Assume that $m^*\in \supp(\Pr_\D)\subseteq \Mf_\D$. Then, for any $\delta \in (0,1)$, with probability at least $1-\delta$, 
    \begin{align*}
        &S_{\Pr_\D}(m^*) \leq \frac{8r_{\max}}{(1-\gamma)^2}\sqrt{\frac{\ln(|\Mf_\D|/\delta)}{|\D|}}\sqrt{\max(\mathcal C_{\text{Bayes}}(\pi^*(m^*)), \mathcal C_{\text{Bayes}}(\pi^*(\Pr_{\D})))}.
    \end{align*}
\end{lemma}

\begin{proof}[Proof of \autoref{thm:bayes_subopt}]
By definition of the Bayesian sub-optimality gap (\autoref{eq:bayes_gap}), 
\begin{align*}
    &S_{\Pr_\D}(m^*) \\
    &= J(\pi^*(m^*), m^*) - J(\pi^*(\Pr_\D), m^*)\\
    &=  J(\pi^*(m^*), m^*) - \E{m\sim\Pr_\D}{J(\pi^*(m^*), m)} + \E{m\sim\Pr_\D}{J(\pi^*(m^*), m)} \\
    &\qquad- \E{m\sim\Pr_\D}{J(\pi^*(\Pr_\D), m)} +\E{m\sim\Pr_\D}{J(\pi^*(\Pr_{\D}), m)} - J(\pi^*(\Pr_\D), m^*)\\
    &\leq (J(\pi^*(m^*), m^*) - \E{m\sim\Pr_\D}{J(\pi^*(m^*), m)}) + (\E{m\sim\Pr_{\D}}{J(\pi^*(\Pr_{\D}), m)} - J(\pi^*(\Pr_{\D}), m^*)),
\end{align*}
where the last inequality holds by definition of $\pi^*(\Pr_\D)$ being a solution of $\max_{\pi}\E{m\sim\Pr_\D}{J(\pi, m)}$. 
The first term measures how the performance of the ideal policy $\pi^*(m^*)$ under the true MDP deviates from its Bayesian average.  
The second term measures the analogous deviation for the Bayes-optimal policy $\pi^*(\Pr_\D)$.  
We aim to upper-bound each of these regret terms.

\textbf{Step 1: Upper bound $J(\pi^*(m^*), m^*) - \E{m\sim\Pr_\D}{J(\pi^*(m^*), m)}$. }
\begin{align*}
    &J(\pi^*(m^*), m^*) - \E{m\sim\Pr_\D}{J(\pi^*(m^*), m)} =\E{m\sim\Pr_\D}{J(\pi^*(m^*), m^*)  - J(\pi^*(m^*), m)}\\
    &\leq \frac{2r_{\max}}{(1-\gamma)^2} \E{m\sim\Pr_\D}{\E{(s,a)\sim d^{\pi^*(m^*)}_m}{  \textsc{tv}(m(s,a), m^*(s,a))}},
\end{align*}
according to \autoref{lemma:simulation_lemma}. Therefore, by Jensen's inequality, 
{\small
\begin{align*}
    J(\pi^*(m^*), m^*) - \E{m\sim\Pr_\D}{J(\pi^*(m^*), m)}&\leq \frac{2r_{\max}}{(1-\gamma)^2}\E{m\sim\Pr_\D}{\sqrt{\E{(s,a)\sim d^{\pi^*(m^*)}_m}{  \textsc{tv}(m(s,a), m^*(s,a))^2}}}\\
    &\leq \frac{2r_{\max}}{(1-\gamma)^2} \sqrt{\E{m\sim\Pr_\D}{\E{(s,a)\sim d^{\pi^*(m^*)}_m}{  \textsc{tv}(m(s,a), m^*(s,a))^2}}},
\end{align*}%
}and by construction of the Bayesian concentrability coefficient (\autoref{eq:bayes_concen}):
\begin{align*}
    &J(\pi^*(m^*), m^*) - \E{m\sim\Pr_\D}{J(\pi^*(m^*), m)}  \\
    &\leq \frac{2r_{\max}}{(1-\gamma)^2} \sqrt{\mathcal C_{\text{Bayes}}(\pi^*(m^*))\E{m\sim\Pr_\D}{\E{(s,a)\sim \beta}{  \textsc{tv}(m(s,a), m^*(s,a))^2}}}.
\end{align*}
Finally, since by construction $\Pr_\D$ is supported only on models that achieve MLE under $\D$, every $m\in \supp(\Pr_\D)$ satisfies the PAC bound of \autoref{lemma:mle_pac}, yielding
\begin{align*}
    J(\pi^*(m^*), m^*) - \E{m\sim\Pr_\D}{J(\pi^*(m^*), m)} \leq    \frac{4r_{\max}}{(1-\gamma)^2} \sqrt{\mathcal C_{\text{Bayes}}(\pi^*(m^*))}\sqrt{\frac{\ln(|\Mf_{\D}|/\delta)}{|\D|}}.
\end{align*}

\textbf{Step 2: Upper bound $\E{m\sim\Pr}{J(\pi^*(\Pr_{\D}), m)} - J(\pi^*(\Pr_{\D}), m^*)$. }
We use the same reasoning as above, but applied this time to the Bayes-optimal policy, and deduce the following bound:
\begin{align*}
    \E{m\sim\Pr}{J(\pi^*(\Pr_{\D}), m)} - J(\pi^*(\Pr_{\D}), m^*) \leq  \frac{4r_{\max}}{(1-\gamma)^2} \sqrt{\mathcal C_{\text{Bayes}}(\pi^*(\Pr_\D))}\sqrt{\frac{\ln(|\Mf_{\D}|/\delta)}{|\D|}}.
\end{align*}
We finally obtain: 
\begin{align*}
    S_{\Pr_\D}(m^*) &\leq  \frac{4r_{\max}}{(1-\gamma)^2}\sqrt{\frac{\ln(|\Mf_{\D}|/\delta)}{|\D|}} \left( \sqrt{\mathcal C_{\text{Bayes}}(\pi^*(m^*))}+   \sqrt{\mathcal C_{\text{Bayes}}(\pi^*(\Pr_\D))}\right)\\
    &\leq \frac{8r_{\max}}{(1-\gamma)^2}\sqrt{\frac{\ln(|\Mf_{\D}|/\delta)}{|\D|}}\sqrt{\max(\mathcal C_{\text{Bayes}}(\pi^*(m^*)) , \mathcal C_{\text{Bayes}}(\pi^*(\Pr_\D))}.
\end{align*}
\end{proof}

\begin{lemma}[\textbf{Existence of a lower bound on robust sub-optimality gap}]
\label{thm:robust_subopt}
Assume that $\gamma > 1/2$. We can construct an MDP instance $m^* = m_1$ inducing an offline dataset $\D$ and a set of models $\Mf = \{m_{-1}, m_1\}$, such that the optimal robust policy $\pi^*({\Mf})$ satisfies w.p. at least $\delta_0$:
\begin{align*}
     J(\pi^*(m^*), m^*) - J(\pi^*({\Mf}), m^*) >  \frac{8r_{\max}}{(1-\gamma)^2}\sqrt{\frac{\ln(|\Mf|/\delta_0)}{|\D|}}\sqrt{\mathcal C(\pi^*(\Mf))}.
\end{align*}
\end{lemma}

\begin{proof}[Proof of \autoref{thm:robust_subopt}]

Similarly as \cite{li2024settling}, we build two MDPs (here, two sequential two-armed bandits) and a behavior policy such that the conservative value estimated from the resulting dataset is far from the optimal return. 

\textbf{The example.} Consider two sequential bandits $\Mf:= \{m_{-1}, m_1\}$ such that $m^*=m_1$ is the ground-truth model. We parameterize both by $\theta\in\{-1,1\}$. They share the same action space $\A := \{-1,1\}$, the same reward for the negative action, but different reward for the positive action, namely, $R_{\theta}(-1)\sim \mathcal{B}(1/2)$ while $R_{\theta}(1)\sim \mathcal{B}(1/2+\theta\epsilon)$ with $\epsilon>0$ that will be determined later. 
Clearly, the ideal policy for each MDP is $\pi^*_{\theta}= \theta$. With notational abuse, let the behavior policy be $\beta(1) = \beta = 1-\beta(-1)$
from which we collect $|\D|$ i.i.d. samples. As will be justified in the following, we set $\beta = \frac{(1-\gamma)^4}{64r_{\max}^2\mathcal C(\pi^*(\Mf))} \in [0,1].$

\textbf{Relation to the bandit example in \autoref{sec:bandit}.} 
Here, action $-1$ plays the role of the sampling action with Bernoulli parameter $1/2$, and action $1$, the “uncovered action” in the skewed bandit dataset of \autoref{sec:bandit}.
We deliberately avoid setting $\beta(1)=0$, since that would violate the data coverage assumption and make $\mathcal C(\pi^*(\Mf))$ infinite.\footnote{In the limit $\mathcal C(\pi^*(\Mf))\to\infty$, we recover the example in \autoref{sec:bandit}.} Yet, the dataset is still skewed because it can be generated from $m_{1}$ under a suboptimal policy (but close-to-optimal policy for $m_{-1}$). In that case, it is statistically hard to identify the correct model \textit{when $\epsilon$ is too small}, which is how we choose it for the lower bound to hold. 

\textbf{Proof.} 
With a slight abuse of notation, denote a policy by $\pi := \pi(1) = 1 - \pi(-1)$. The discounted value may be expressed according to the underlying MDP as: 
\begin{align*}
    J(\pi, \theta) &= \sum_{t=0}^{\infty}\gamma^t\left(\pi\cdot(1/2 + \theta\epsilon) + (1- \pi)\cdot 1/2)\right)
    = \sum_{t=0}^{\infty}\gamma^t \left(\frac{1}{2} +\theta \epsilon\pi\right) = \frac{1 + 2\theta\epsilon\pi}{2(1-\gamma)}.
\end{align*}
For each model $\theta\in \{-1, 1\}$, the suboptimality gap of a policy $\pi$ is thus:
\begin{align*}
    \delta^{\pi}(1) &=  \frac{1 + 2\epsilon}{2(1-\gamma)} -  \frac{1 + 2\epsilon\pi}{2(1-\gamma)}
    = \frac{2\epsilon(1 - \pi)}{2(1-\gamma)} = \frac{\epsilon(1 - \pi)}{1-\gamma},\\
    \delta^{\pi}(-1) &= \frac{1}{2(1-\gamma)} -  \frac{1 - 2\epsilon\pi}{2(1-\gamma)} = \frac{\epsilon\pi}{1-\gamma}.
\end{align*}
More synthetically, and since $\gamma\in [0,1)$, we get:
\begin{align}
\label{eq: subop theta}
    \delta^{\pi}(\theta) = \frac{\epsilon}{1-\gamma}(1 - \pi(\theta)).
\end{align}
By contradiction, suppose that we can find a policy estimate from the dataset such that: 
\begin{align*}
    \Pr_{\theta}(J(\pi^*(\theta), \theta) - J(\hat{\pi}, \theta) \leq \epsilon) = \Pr_{\theta}(\delta^{\hat\pi}(\theta) \leq \epsilon) \geq \frac{7}{8}.
\end{align*}
Then, from \autoref{eq: subop theta},  we should have with probability greater than $7/8$ that:\begin{align*}
    \frac{\epsilon}{1-\gamma}(1 - \hat{\pi}(\theta)) \leq \epsilon\iff \hat{\pi}(\theta) \frac{\epsilon}{1-\gamma}\geq \frac{\epsilon}{1-\gamma} -\epsilon \iff \hat\pi(\theta)\geq 1 - (1-\gamma) = \gamma.
\end{align*}
By assumption, $\gamma>1/2$. If the above statement were true, then we could construct the following estimator 
$\hat{\theta}$ of $\theta$ based on $\hat\pi$:
\begin{align}
\label{eq: estimate}
    \hat\theta = \argmax_{a\in\{-1,1\}} \hat\pi(a)
\end{align}
and thus, 
$\Pr_{\theta}(\hat\theta = \theta) = \Pr_{\theta}(\hat\pi(\theta) > 1/2) \geq \Pr_{\theta}(\hat\pi(\theta) \geq \gamma) \geq 7/8. $

We analyze the hypothesis testing of identifying the true MDP given the generated data. Formally, let the test $\phi(\D) = 0$ mean ``decide $m^* = m_{-1}$'' and $\phi(\D) = 1$ mean ``decide $m^* = m_1$''. 
Consider the minimax probability of error $p_e$ between $m_{-1}$ and $m_1$:
$$
p_e \;=\; \inf_{\phi}\max(\Pr_{{-1}}(\phi(\D)\neq 0),\; \Pr_{1}(\phi(\D)\neq 1)),
$$
where $\Pr_{\theta}(\D)$ denotes the sampling distribution of $\D$ under model $\theta$. 
By \citep{bickel2009springer}[Thm. 2.2], the following lower bound holds: 
\begin{align}
\label{eq: bret huber}
    p_e \geq \frac{1}{4} \exp(-\textsc{kl}(\Pr_{-1}(\D) \mid\mid \Pr_{1}(\D))).
\end{align}
Each data sample $\D_i$ is generated according to a mixture of two Bernoullis $\D_i\sim \beta\mathcal{B}(1/2+\theta\epsilon) + (1- \beta)\mathcal{B}(1/2)$. 
Let $n_{\mathrm{w}}$ be the number of $1$-reward samples, i.e., successful events $\D_i = 1$ (respectively, $n_{\mathrm{l}}$ the number of $0$-reward samples). Then:
\begin{align*}
    \Pr(\D_i = 1) &= \beta\left(\frac{1}{2} +\theta \epsilon\right) + \left(1 - \beta\right)\frac{1}{2} = \frac{1}{2} + \beta\theta\epsilon,\\
    \Pr(\D_i = 0) &= \beta\left(1 - \left(\frac{1}{2} +\theta \epsilon\right)\right) + \left(1 - \beta\right)\left(1-\frac{1}{2}  \right) = \beta\left(\frac{1}{2} -\theta \epsilon\right) + \frac{1}{2} - \frac{\beta}{2}= \frac{1}{2} -{\beta\theta\epsilon},
\end{align*}
and the likelihood for the whole dataset is: 
\begin{align}
\label{eq: likelihood dataset}
    \Pr_{\theta}(\D) = \left(\frac{1}{2} + \beta\theta\epsilon\right)^{n_{\mathrm{w}}}\left(\frac{1}{2} - \beta\theta\epsilon\right)^{n_{\mathrm{l}}}.
\end{align}
Based on Eq.~\eqref{eq: likelihood dataset}, we can compute the divergence: 
\begin{align*}
    \textsc{kl}(\Pr_{1}(\D)\,\|\,\Pr_{-1}(\D)) &= n_{\mathrm{w}}\textsc{kl}\left(\mathcal{B}\left(\frac{1}{2} +{\beta\epsilon}\right), \mathcal{B}\left(\frac{1}{2} -\beta\epsilon\right)\right) +
    n_{\mathrm{l}}\textsc{kl}\left(\mathcal{B}\left(\frac{1}{2} -\beta\epsilon\right), \mathcal{B}\left(\frac{1}{2} +\beta\epsilon\right)\right),
\end{align*}
by the additive property of KL between independent samples. Remarking that 
\begin{align*}
    &\textsc{kl}\left(\mathcal{B}\left(\frac{1}{2} +\epsilon\beta\right), \mathcal{B}\left(\frac{1}{2} -\epsilon\beta\right) \right)\\
    &= \left(\frac{1}{2} +\epsilon\beta\right)\log\left(\frac{\frac{1}{2} +\epsilon\beta}{\frac{1}{2} -\epsilon\beta}\right) + \left(1- \frac{1}{2} -\epsilon\beta\right)\log\left(\frac{1 - \frac{1}{2} -\epsilon\beta}{1 - \frac{1}{2} +\epsilon\beta}\right)\\
    &= \left(\frac{1}{2} +\epsilon\beta\right)\log\left(\frac{\frac{1}{2} +\epsilon\beta}{\frac{1}{2} -\epsilon\beta}\right) + \left(\frac{1}{2} -\epsilon\beta\right)\log\left(\frac{\frac{1}{2} -\epsilon\beta}{ \frac{1}{2} +\epsilon\beta}\right)\\
    &= \textsc{kl}\left(\mathcal{B}\left(\frac{1}{2} -\epsilon\beta\right), \mathcal{B}\left(\frac{1}{2} +\epsilon\beta\right)\right),
\end{align*}
it results that
\begin{align*}
    \textsc{kl}(\Pr_{-1}(\D)\,\|\,\Pr_{1}(\D)) &= (n_{\mathrm{w}}+ n_{\mathrm{l}})\textsc{kl}\left(\mathcal{B}\left(\frac{1}{2} +\epsilon\beta\right), \mathcal{B}\left(\frac{1}{2} -\epsilon\beta\right)\right) \\
    &= |\D|\left(\frac{1}{2} +\epsilon\beta\right)\log\left(\frac{\frac{1}{2} +\epsilon\beta}{\frac{1}{2} -\epsilon\beta}\right) +|\D|\left(\frac{1}{2} -\epsilon\beta\right)\log\left(\frac{\frac{1}{2} -\epsilon\beta}{ \frac{1}{2} +\epsilon\beta}\right)\\
    &= |\D|\left(\frac{1}{2} +\epsilon\beta - \frac{1}{2} +\epsilon\beta\right)\log\left(\frac{\frac{1}{2} +\epsilon\beta}{\frac{1}{2} -\epsilon\beta}\right)\\
    &= {2|\D|\epsilon\beta}\log\left(\frac{1 +2\epsilon\beta}{1 -2\epsilon\beta}\right).
\end{align*}
For a small enough $\epsilon$, the series expansion of the logarithm term yields:
\begin{align*}
    \log\left(\frac{1+2\epsilon\beta}{1-2\epsilon\beta}\right) = 2\cdot 2\epsilon\beta + o((2\epsilon\beta)^2) = 4\epsilon\beta + o((\epsilon\beta)^2)
\end{align*}
so that
\begin{align*}
    \textsc{kl}(\Pr_{-1}(\D)\,\|\,\Pr_{1}(\D)) = 2|\D|\epsilon\beta(4\epsilon\beta + o((\epsilon\beta)^2)) = {8|\D|(\epsilon\beta)^2} + o((\epsilon\beta)^3). 
\end{align*}
We can eventually deduce that for a small enough $\epsilon$:
\begin{align}
\label{eq: kl bound}
\textsc{kl}(\Pr_{-1}(\D)\,\|\,\Pr_{1}(\D)) &\leq {c|\D|\epsilon^2\beta}. 
\end{align}

The binary testing lower bound \eqref{eq: bret huber} together with \autoref{eq: kl bound} establishes that the misidentification probability $p_e$ is at least $1/8$ as long as 
\begin{align}
\label{eq: exp bound}
\exp\left(-c|\D|\epsilon^2\beta\right)\geq 1/2,    
\end{align}
which is equivalent to $\epsilon \leq \sqrt{\frac{\ln(2)}{c|\D|\beta}}$.

To conclude, we have assumed that there is a policy estimate $\hat{\pi}$ such that $\Pr_{\theta}(\delta^{\hat\pi}(\theta) \leq \epsilon) \geq 7/8$, namely,
$\Pr_{-1}(\delta^{\hat\pi}(-1) > \epsilon) < 1/8$ and $\Pr_{1}(\delta^{\hat\pi}(1) > \epsilon) < 1/8.$ 
Then, in view of our previous arguments, the estimator $\hat\theta$ as defined in \autoref{eq: estimate} must satisfy $\Pr_{-1}(\hat\theta \neq \theta) < 1/8$ and $\Pr_{-1}(\hat\theta \neq \theta) < 1/8$, which contradicts the misidentification lower-bound $p_e\geq 1/8$ when $\epsilon \leq \sqrt{\frac{\ln(2)}{c|\D|\beta}}$.
Remarking that in our example, $|\Mf|=2$ and $\delta_0=1/8$, we set $c=\ln(2)/\ln(16)$ to conclude that any policy estimate $\hat{\pi}$ inevitably satisfies $\Pr_{\theta}(J(\pi^*(\theta), \theta) - J(\hat{\pi}, \theta) \geq \epsilon)  \geq \frac{7}{8}$\footnote{For that specific value $c=\ln(2)/\ln(16)$, any $\epsilon < \min\{\frac{1}{\beta}\sqrt{1/4-8\beta}), 1/(2\beta)\}$ would make the inequality \eqref{eq: exp bound} valid, as long as $\beta<1/32$. Since $\mathcal C(\pi^*(\Mf)) \geq 1$ and $r_{\max}=1$, condition $\beta<1/32$ is automatically fulfilled. Additionally, for $\beta<1/32$, it holds that $\sqrt{\frac{\ln(2)}{c|\D|\beta}} \leq \frac{1}{\beta}\sqrt{1/4-8\beta}$, so $\epsilon\leq \sqrt{\frac{\ln(2)}{c|\D|\beta}}$ is a more restrictive bound on $\epsilon$.}. In particular, the statement holds for the optimal robust policy $\pi^*(\Mf)$. 
\end{proof}

\subsection{Auxiliary Lemmas for \autoref{thm:final_bound}}

We first recall a standard PAC bound for maximum likelihood estimation (MLE), adapted from \citet[Theorem 21]{agarwal2020flambe}.

\begin{lemma}[\textbf{MLE PAC-bound}]
\label{lemma:mle_pac}
    Let $\beta\in\Delta(\S\times\A)$ be the offline distribution induced by $\D$, and $$\hat m = \argmax_{m \in \Mf} \E{(s,a,r,s')\sim \D}{\log m(r, s' \mid s,a)} $$ be the MLE model within a finite uncertainty set $\Mf$. Suppose $m^* \in \Mf$. Then, for any $\delta \in (0,1)$, with probability at least $1-\delta$:
    \begin{align*}
     \E{(s,a)\sim \beta}{\textsc{tv}(\hat m(s,a), m^*(s,a))^2}\leq \frac{2\log(|\Mf|/\delta)}{|\D|}.
    \end{align*}
\end{lemma}

We next establish a \textit{general} simulation lemma below. Unlike the classic simulation lemma~\citep[Lemma 9]{uehara2021pessimistic}, which applies only to stationary policies and considers transition errors alone, our result (i) extends to history-dependent policies via general Bellman recursions, and (ii) incorporates discrepancies in both transition and reward functions.

\begin{lemma}[\textbf{General simulation lemma}]
\label{lemma:simulation_lemma}
    For any history-dependent policy $\pi: \mathcal{H}_t \to \Delta(\A)$ and any two MDP models $m = (P, R), \hat{m} = (\hat{P}, \hat{R})$, it holds that:
    \begin{align*}
        \vert J(\pi, m) - J(\pi, \hat{m})\vert  &\leq \frac{2r_{\max}}{(1-\gamma)^2}\E{(s,a)\sim d_{\hat{m}}^{\pi}}{\textsc{tv}(m(s,a), \hat{m}(s,a))}. 
    \end{align*}
\end{lemma}

\begin{proof}
    Given an MDP $m$ and a policy $\pi$, define the value function starting from a history $h_t\in\Hc_t$:
    \begin{align*}
        V^{\pi}_m(h_t) \defeq \E{\pi,m}{\sum_{k=t}^{\infty}\gamma^{k-t} r_{k+1}|h_t}.
    \end{align*}
    It satisfies the history-based Bellman recursion, where $h_{t+1} = (h_t, a_t,r_{t+1},s_{t+1})$, 
    \begin{align}
        V^{\pi}_m(h_t) &= \E{a_t\sim\pi(h_t)}{ \E{r_{t+1}\sim R(s_t,a_t), s_{t+1}\sim P(s_t,a_t)}{r_{t+1} + \gamma V^{\pi}_m(h_{t+1})}}\nonumber\\
        &= \E{a_t\sim\pi(h_t)}{\E{(r_{t+1}, s_{t+1})\sim m(s_t,a_t)}{r_{t+1} + \gamma V^{\pi}_m(h_{t+1})} }.\label{eq:general_bellman}
    \end{align}
    Define the value difference between models $m$ and $\hat m$ given a history $h_t\in\Hc_t$:
    \begin{align*}
        \Delta^{\pi}(h_t) := V^{\pi}_m(h_t) - V^{\pi}_{\hat m}(h_t). 
    \end{align*}
    Thus the return gap is
    \begin{align*}
        \vert J(\pi, m) - J(\pi, \hat{m})\vert = \vert\E{s_0\sim\rho}{\Delta^{\pi}(s_0)}\vert.
    \end{align*}
    Applying recursion \eqref{eq:general_bellman} to both models $m$ and $\hat m$, we get:
    \begin{align*}
        \Delta^{\pi}(h_t) = \E{a_t\sim\pi(h_t)}{\E{m(s_t,a_t)}{r_{t+1} + \gamma V^{\pi}_m(h_{t+1})} - \E{\hat{m}(s_t,a_t)}{r_{t+1} + \gamma V^{\pi}_{\hat{m}}(h_{t+1})}}.
    \end{align*}
    Decompose this into two parts:
    \begin{align}
        \Delta^{\pi}_1(h_t,a_t)&\defeq \E{m(s_t,a_t)}{r_{t+1} + \gamma V^{\pi}_m(h_{t+1})} - \E{\hat m(s_t,a_t)}{r_{t+1} + \gamma V^{\pi}_m(h_{t+1})},\nonumber\\
        \Delta_2^{\pi}(h_t,a_t) &\defeq \E{\hat m(s_t,a_t)}{r_{t+1} + \gamma V^{\pi}_m(h_{t+1})} - \E{\hat{m}(s_t,a_t)}{r_{t+1} + \gamma V^{\pi}_{\hat{m}}(h_{t+1})},\nonumber\\
        \Delta^{\pi}(h_t) &= \E{a_t\sim\pi(h_t)}{\Delta^{\pi}_1(h_t,a_t) + \Delta^{\pi}_2(h_t,a_t)}.   \label{eq:delta}
    \end{align}
    Applying the fundamental property of TV distance~\citep{levin2017markov}, the first term: 
    \begin{align}
         \Delta^{\pi}_1(h_t, a_t) &\leq 2 \textsc{tv}(m(s_t,a_t), \hat{m}(s_t,a_t))\cdot\lVert r_{t+1} + \gamma V^{\pi}_m(h_{t+1})\rVert_{\infty}\nonumber\\
         &\leq 2 \textsc{tv}(m(s_t,a_t), \hat{m}(s_t,a_t))\frac{r_{\max}}{1-\gamma},\label{eq:delta1}
    \end{align}
    where the second inequality stems from the fact that:
    \begin{align*}
        \lVert r_{t+1} + \gamma V^{\pi}_m(h_{t+1})\rVert_{\infty}\leq r_{\max} + \gamma \frac{r_{\max}}{1-\gamma} = \frac{r_{\max}}{1-\gamma}. 
    \end{align*}
    The second term can be recognized as:
    \begin{align}
    \label{eq:delta2}
        \Delta_2(h_t, a_t) &= \gamma \E{(r_{t+1}, s_{t+1})\sim \hat m(s_t,a_t)}{\Delta^{\pi}(h_{t+1})}.
    \end{align}
    Combining \autoref{eq:delta} with \Cref{eq:delta1,eq:delta2} yields: 
    \begin{align*}
        \Delta^{\pi}(h_t) 
        &\leq \E{a_t\sim\pi(h_t)}{\frac{2r_{\max}}{1-\gamma} \textsc{tv}(m(s_t,a_t), \hat m(s_t,a_t))}
        + \gamma \E{a_t\sim\pi(h_t)}{\E{\hat m(s_t,a_t)}{\Delta^{\pi}(h_{t+1})}}.
    \end{align*}
    For convenience, denote the one-step error difference at time $t$ by:
    \begin{align*}
        \epsilon(s_t,a_t):= \frac{2r_{\max}}{1-\gamma}  \textsc{tv}(m(s_t,a_t), \hat m(s_t,a_t)), 
    \end{align*}
    so that 
    \begin{align*}
        \Delta^{\pi}(h_t)\leq \E{a_t\sim\pi(h_t)}{\epsilon(s_t,a_t) }
        + \gamma \E{a_t\sim\pi(h_t)}{\E{(r_{t+1}, s_{t+1})\sim \hat m(s_t,a_t)}{\Delta^{\pi}(h_{t+1})}}.
    \end{align*}
Iterating from $t=0$ and taking $\mathbb E_{s_0\sim\rho}$,
\begin{align*}
    &\E{s_0\sim\rho}{\Delta^{\pi}(s_0)} \\
    &\leq \E{s_0\sim\rho, a_0\sim\pi(s_0)}{\epsilon(s_0,a_0)}  +\gamma \E{s_0\sim\rho, a_0\sim\pi(s_0), (r_1,s_1) \sim \hat{m}(s_0,a_0)}{\Delta^{\pi}(h_1)}\\
    &= \E{h_0\sim P^{\pi}_{\hat{m},0}, a_0\sim\pi(h_0)}{\epsilon(s_0,a_0)}  + \gamma\E{h_1\sim P^{\pi}_{\hat{m},1}}{\Delta^{\pi}(h_1)}\\
    &\leq \E{h_0\sim P^{\pi}_{\hat{m},0}, a_0\sim\pi(h_0)}{\epsilon(s_0,a_0) } + \gamma \left( \E{h_1\sim P^{\pi}_{\hat{m},1}, a_1\sim\pi(h_1)}{\epsilon(s_1,a_1) }
        + \gamma \E{h_2\sim P^{\pi}_{\hat{m},2}}{\Delta^{\pi}(h_{2})}\right)\\
    &= \E{h_0\sim P^{\pi}_{\hat{m},0}, a_0\sim\pi(h_0)}{\epsilon(s_0,a_0) } + \gamma  \E{h_1\sim P^{\pi}_{\hat{m},1}, a_1\sim\pi(h_1)}{\epsilon(s_1,a_1) }
        + \gamma^2 \E{h_2\sim P^{\pi}_{\hat{m},2}}{\Delta^{\pi}(h_{2})}\\
    &\vdots \\
    &\leq \sum_{t=0}^{\infty} \gamma^t\E{h_t\sim P^{\pi}_{\hat{m},t}, a_t\sim\pi(h_t)}{\epsilon(s_t,a_t)} \\
    &= \frac{1}{1-\gamma}\,\E{(s,a)\sim d^{\pi}_{\hat m}}{\epsilon(s,a)},
\end{align*}
where in the last line, we use the discounted occupancy
 $d_{\hat m}^{\pi}$ under $\hat m$:
$$
d^{\pi}_{\hat m}(s,a)=(1-\gamma)\sum_{t=0}^\infty \gamma^t \Pr(s_t=s,a_t=a\mid \pi,\hat m).
$$
Finally, by symmetry (the same argument applied to $-\E{s_0\sim\rho}{\Delta^{\pi}(s_0)}$),
\begin{align*}
    |\E{s_0\sim\rho}{\Delta^{\pi}(s_0)}| &\leq \left\vert\frac{1}{1-\gamma}\E{(s,a)\sim d^{\pi}_{\hat m}}{\epsilon(s,a)}\right\vert\\
    &= \frac{1}{1-\gamma}\E{(s,a)\sim d^{\pi}_{\hat m}}{\frac{2r_{\max}}{1-\gamma}  \textsc{tv}(m(s,a), \hat m(s,a))}\\
    &= \frac{2r_{\max}}{(1-\gamma)^2}\E{(s,a)\sim d^{\pi}_{\hat m}}{\textsc{tv}(m(s,a), \hat{m}(s,a))}.
\end{align*}
\end{proof}

\subsection{Related Work on Offline RL Theory}

Theoretical works on offline RL aim to establish performance guarantees without online exploration. The key difficulty is limited data coverage: when the dataset sufficiently covers all state-action pairs, standard PAC-style guarantees can be obtained without requiring conservatism~\citep{munos2008finite}. Under partial coverage, however, conservative algorithms that penalize policies deviating from well-supported regions are essential to ensure robust learning. Many prior results establish minimax optimality with information-theoretic bounds under \textit{worst-case} assumptions over possible MDPs~\citep{jin2021pessimism, rashidinejad2021bridging, li2024settling}. The lower-bound part in \autoref{thm:robust_subopt} draws on the same information-theoretic principle as \citet{li2024settling}, but we provide a simpler counterexample tailored to our sequential bandit problem in \autoref{sec:bandit}.

Complementary to this line of work, our work adopts a Bayesian viewpoint that focuses on the \textit{average case} over possible MDPs, an aspect underexplored in the theoretical literature. An exception is the Bayesian offline RL setting studied in \citet[Section 8]{uehara2021pessimistic} under a Markov policy, optimized via mirror descent with posterior sampling. Their analysis derives a Bayesian sub-optimality gap and interprets the resulting soft robustness as ``implicit pessimism''. However, because their analysis is restricted to Markov policies, whereas the Bayes-optimal policy is generally history-dependent, their regret bound is looser than ours in \autoref{thm:bayes_subopt}.

Recent work~\citep{fellows2025sorel} addresses the same Bayesian offline RL problem with history-dependent policies. Their main result (Theorem 1) provides a Bayesian sub-optimality bound expressed by an expected KL divergence, called the posterior information loss, and takes the worst case over all policies. Their derivation starts from a TV distance, which they upper bound by a KL divergence to leverage the product rule of logarithms. In that respect, our \autoref{thm:bayes_subopt} gives a tighter upper bound as it is based on TV distance rather than KL, which requires extending the simulation lemma to history-dependent policies (see \autoref{lemma:simulation_lemma}). Additionally, the regret bound provided in \citet[Theorem 1]{fellows2025sorel} involves a supremum over all policies. Although we believe their bound can be established directly on the Bayes-optimal policy rather than the worst case, building on their current result would require a full data-coverage assumption. In contrast, our lower bound in \autoref{thm:bayes_subopt} only requires partial data coverage. Finally, \citet{fellows2025sorel} focuses on bounding the Bayes regret. At the same time, we additionally provide a lower bound on the robust sub-optimality gap in \autoref{thm:robust_subopt}, showing when the Bayesian approach can provably outperform conservative ones.

\section{\algo Algorithm Details}
\label{app:algorithm}

\textbf{RL loss function.} In recurrent off-policy RL, optimization is typically performed on full trajectories (rather than i.i.d. transition tuples) for compute efficiency. In our setting, each training trajectory is a \textit{concatenation} of a real prefix and an imagined rollout: starting from an initial history $h_t \in \D$, a world model and the policy generate future steps until truncation at $t' \le T$.
Formally, let $$\tau = (s_{0:t} \oplus s_{t+1:t'} ,a_{0:t-1} \oplus a_{t:t'-1}, r_{1:t} \oplus \hat r_{t+1:t'}, d_{1:t} \oplus \hat d_{t+1:t'}),$$
where $\oplus$ denotes \textit{concatenation}. 
Our RL loss on $\tau$ balances contributions from real and imagined segments:
{\small
\begin{equation}
\label{eq:rl_loss}
L(Q_\omega, \pi_\nu; \tau, \kappa) \defeq \frac{\kappa}{t}\sum_{j=0}^{t-1} l(Q_\omega, \pi_\nu; h_j, a_j, r_{j+1}, d_{j+1}, s_{j+1}) + \frac{1-\kappa}{t'-t}\sum_{j=t}^{t'-1} l(Q_\omega, \pi_\nu; \hat h_j, \hat a_j, \hat r_{j+1}, \hat d_{j+1}, \hat s_{j+1}).
\end{equation}}Here, $\hat h_j$ denotes the imagined history (with $\hat h_t = h_t$) and $\kappa \in (0,1)$ is the real data ratio. The per-step loss $l(Q_\omega, \pi_\nu; h_j, a_j, r_{j+1}, d_{j+1}, s_{j+1})$ is standard off-policy loss without explicit conservatism, such as DQN~\citep{mnih2013playing} for discrete control and SAC~\citep{haarnoja2018soft} for continuous control.

\textbf{Rollout stopping criteria.} In our rollout subroutine (\autoref{algo:rollout}), a rollout finishes when any of three conditions hold:
\begin{equation}
\mathrm{done}_{t+1} \defeq  (\unc(\hat s_t, \hat a_t) > \uncthres) \lor (t+1 \ge T) \lor \term(\hat s_t,\hat a_t, \hat s_{t+1}).
\end{equation}
\begin{enumerate}[noitemsep, topsep=0pt, leftmargin=*]
    \item Uncertainty truncation: as described in \autoref{sec:how_longhorizon}, instead of enforcing a fixed horizon $H$, we truncate rollouts adaptively using an uncertainty threshold calibrated on the real dataset. This allows rollouts to extend as long as the model remains confident.
    \item Timout truncation: to remain consistent with test-time evaluation, we impose a hard cap at the environment’s maximum episode length $T$, regardless of rollout length. 
    \item Ground-truth termination: we retain the environment's rule-based terminal function to provide true terminal signals $\hat d_{t+1}$, following prior model-based RL methods~\citep{yu2020mopo}. Including this prior knowledge makes our algorithm directly comparable to model-based baselines, which are our main focus.  
\end{enumerate}
Importantly, only the terminal signal disables bootstrapping in RL, while both uncertainty- and timeout-based truncations preserve bootstrapping\footnote{\url{https://gymnasium.farama.org/tutorials/gymnasium_basics/handling_time_limits/}.}, which aligns with the Bayesian objective.

\section{Implementation Details}
\label{app:implementation}

\subsection{Reproducibility Statement}

We release our full codebase in the supplementary material, built on JAX~\citep{jax2018github} and Equinox~\citep{kidger2021equinox}. For further reproducibility, we also release pretrained world ensemble checkpoints to make agent training on top of the released checkpoints straightforward. 

\subsection{Dataset Details}
\label{app:dataset}

\textbf{The bandit dataset.}
As introduced in \autoref{sec:bandit}, we construct a skewed two-armed bandit dataset. We collect $10$ trajectories of length $T=100$, yielding $|\D|=1000$ action–reward pairs. We use one-hot encoding on actions as inputs for reward models and agents. The dataset is split into training and validation sets with a $4{:}1$ ratio. Since $\D$ only covers arm $0$ with $p^*_0=0.5$, the true parameter of arm $1$, $p^*_1$, is completely unseen. 

At test time, we vary $p^*_1 \in \{0.01, 0.3, 0.55, 0.7, 0.99\}$, where each choice defines a distinct bandit problem. Each problem is evaluated over $20$ independent episodes, and all problems are assessed in parallel, yielding $5 \times 20 = 100$ evaluation runs. 
The normalized return during test time is computed by: $\frac{1}{T}\sum_{t=0}^{T-1} r_{t+1}$, where $r_{t+1} \sim \mathcal B(p^*_{a_t})$.

\textbf{D4RL locomotion benchmark.} 
We evaluate on the standard D4RL locomotion benchmark~\citep{fu2020d4rl}, comprising 12 datasets formed by the Cartesian product of tasks (halfcheetah, hopper, walker2d) and dataset types (random-v2, medium-v2, medium-replay-v2, medium-expert-v2). This benchmark is the most widely used in offline RL research.
The underlying environments are OpenAI Gym tasks: HalfCheetah-v2 ($\mathcal{S} \subset \mathbb{R}^{17}, \mathcal{A} \subset \mathbb{R}^{6}$), Hopper-v2 ($\mathcal{S} \subset \mathbb{R}^{11}, \mathcal{A} \subset \mathbb{R}^{3}$), and Walker2d-v2 ($\mathcal{S} \subset \mathbb{R}^{17}, \mathcal{A} \subset \mathbb{R}^{6}$). The maximum episode step $T$ is $1000$. Hopper-v2 and Walker2d-v2 have termination functions, while HalfCheetah-v2 does not. 

Dataset sizes vary: 100k-200k transitions for medium-replay, 1M for random and medium, and 2M for medium-expert. These datasets also differ qualitatively. The random dataset is collected with a uniformly random policy; medium-replay corresponds to the replay buffer of an agent trained to a medium-level policy; medium itself is generated directly from a medium-level policy; and medium-expert is a mixture of trajectories from both a medium policy and an expert policy. Based on these properties, we categorize random as \textit{low-quality}, medium-replay and medium-expert as \textit{moderate coverage}, and medium as \textit{narrow coverage}.

Performance is reported in terms of normalized scores, following the D4RL and NeoRL conventions:
$$
\text{normalized score} = \frac{\text{score} - \text{random score}}{\text{expert score} - \text{random score}} \times 100.
$$
It is worth noting that the expert policies in these locomotion tasks are not strictly optimal. As a result, it is possible for algorithms to achieve normalized scores greater than $100$ (e.g., around $120$). Therefore, we consistently categorize all medium-\verb|*| datasets as medium-quality.

\textbf{NeoRL locomotion benchmark.} The locomotion benchmark in the NeoRL~\citep{qin2022neorl} has been widely used in recent model-based offline RL methods. The setup closely mirrors the D4RL locomotion benchmark, with nearly identical environments: HalfCheetah-v3 ($\mathcal{S} \subset \mathbb{R}^{18}, \mathcal{A} \subset \mathbb{R}^{6}$), Hopper-v3 ($\mathcal{S} \subset \mathbb{R}^{12}, \mathcal{A} \subset \mathbb{R}^{3}$), and Walker2d-v3 ($\mathcal{S} \subset \mathbb{R}^{18}, \mathcal{A} \subset \mathbb{R}^{6}$). Compared to D4RL, NeoRL increases the state dimensionality by one in each environment.

For each environment, NeoRL provides three datasets (Low, Medium, High), collected using policies of the corresponding performance levels. This leads to 9 datasets in total.
Compared to their D4RL counterparts, these policies are more deterministic, leading to smaller data coverage, which we categorize as \textit{narrow coverage} in this paper. 
Dataset sizes range from roughly 200k to 1M transitions. To avoid ambiguity, throughout the paper we denote NeoRL datasets with capitalized environment names and the v3 suffix (e.g., \textit{HalfCheetah-v3-Medium}), while D4RL datasets are written in lowercase with the v2 suffix (e.g., \textit{halfcheetah-medium-v2}).

\textbf{D4RL Adroit benchmark.} 
The Adroit benchmark is widely regarded as substantially more challenging than locomotion tasks. It involves controlling a 28-DoF robotic arm to perform high-dimensional manipulation tasks: Pen ($\mathcal{S} \subset \mathbb{R}^{45}, \mathcal{A} \subset \mathbb{R}^{24}$, $T=100$), Door ($\mathcal{S} \subset \mathbb{R}^{39}, \mathcal{A} \subset \mathbb{R}^{28}$, $T=200$), Hammer ($\mathcal{S} \subset \mathbb{R}^{46}, \mathcal{A} \subset \mathbb{R}^{26}$, $T=200$). Among these, Pen includes a  termination function, while Door and Hammer do not. Rewards combine a dense component based on distances with a sparse component that grants a large bonus once a distance threshold is satisfied.

In addition to the high dimensionality and sparsity of rewards, most Adroit datasets are either small or of limited quality. The human demonstration datasets (human-v1) contain only 5k–10k transitions, whereas the cloned datasets (cloned-v1), constructed by mixing behavior cloning with human demonstrations, contain 500k–1M transitions. This leads to 6 datasets in total. Accordingly, we categorize human datasets as \textit{narrow coverage} and cloned datasets as \textit{moderate coverage}.

Dataset performance levels vary significantly, as shown in the $\pi_\D$ column of \autoref{tab:d4rl_adroit}. Pen-human-v1 and pen-cloned-v1 achieve normalized scores in the range of $70$–$90$, whereas door-human-v1, door-cloned-v1, hammer-human-v1, and hammer-cloned-v1 yield scores below $10$. Accordingly, we categorize pen-\verb|*| datasets as \textit{medium-quality}, and the remaining ones as \textit{low-quality}.
This disparity explains why most algorithms fail to achieve meaningful results on the door and hammer datasets.
Finally, we exclude the relocate tasks from our experiments, as both prior baselines and our method consistently obtain near-zero scores in this setting.

\textbf{D4RL AntMaze benchmark.} 
AntMaze is a particularly challenging benchmark in D4RL due to its sparse reward structure. It involves controlling a MuJoCo Ant ($\S\subset \R^{29}, \A\subset \R^8$) to navigate in different maze layouts (umaze, medium, large). The agent receives reward $1$ only when reaching a goal position in the X–Y plane within a threshold, after which the episode terminates, so the maximum return is $1$. Goals are randomly sampled from a small region but remain unobserved to the agent, making evaluation noisy for RL algorithms. The maximum episode length is $T=700$ for umaze and $T=1000$ for medium and large mazes.

For each layout, two datasets (play and diverse) are provided, each containing $1$M transitions. The \textit{diverse} variants introduce a wider set of start positions compared to \textit{play}.
However, all AntMaze datasets are generated by a hierarchical controller: a high-level breadth-first search (BFS) planner selects waypoints, which are then executed by a low-level learned policy. As a result, trajectories are highly structured and largely restricted to narrow corridors that connect goals without collisions. This planner-driven generation induces narrow coverage, a property noted in prior work~\citep{zhang2024q}.

Performance in AntMaze is generally low relative to other D4RL domains. In particular, umaze-diverse has an average score of only $1.0$, which we categorize as \textit{low-quality}. The remaining AntMaze datasets achieve scores above $10.0$, which we categorize as \textit{medium-quality} (though they could also be viewed as low-quality when considered outside the AntMaze domain).

Following prior model-based RL work (LEQ~\citep{park2025model}\footnote{\url{https://github.com/kwanyoungpark/LEQ}.}, ADMPO~\citep{lin2025any}\footnote{\url{https://github.com/HxLyn3/ADMPO}.}), we adopt the same terminal functions in AntMaze. As in LEQ, we also shift the reward by $-1.0$. Since termination reveals information about the reward function (termination implies success), we do not compare against other model-based or model-free algorithms. Instead, our baselines in AntMaze are the methods reported in LEQ and ADMPO.

\subsection{Details on World Model Ensemble}
\label{app:ensemble}

We provide implementation details on world ensemble, as introduced in \autoref{sec:prelim} and \autoref{sec:ensemble_arch}. 

\textbf{Ensemble architecture.} Following common practice in offline RL~\citep{yu2020mopo}, our world ensemble is a set of neural networks $\ensemble = \{m_{\theta^{n}}\}_{n=1}^N$ where each model outputs a Gaussian distribution over next state $s'$ and reward $r$: $m_{\theta^{n}}(s,a) = \mathcal N(\mu_{\theta^{n}}(s,a), \sigma_{\theta^{n}}(s,a))$, with parameters $\boldsymbol{\theta} = \{\theta^n\}_{n=1}^N$ independently initialized at random.
Following MOBILE~\citep{sun2023model,offinerlkit}\footnote{\url{https://github.com/yihaosun1124/mobile}.}, each MLP $m_\theta(s,a)$ has $4$ hidden layers with a width of $200$ for the locomotion benchmarks, and $5$ layers with a width of $400$ for Adroit benchmark. As described in \autoref{sec:ensemble_arch}, we apply LayerNorm after each hidden layer; this design choice is further ablated in \autoref{sec:analysis}. The basic block consists of (Linear $\to$ LayerNorm $\to$ leaky ReLU) when LayerNorm is enabled, and (Linear $\to$ leaky ReLU) when it is disabled. For the bandit task, we use a small reward-model ensemble that has $2$ hidden layers with a width of $16$.  The weights of each MLP are initialized independently using the default \texttt{equinox.nn.Linear} scheme, i.e., $\text{Unif}[-\frac{1}{\sqrt{\text{dim}_{\text{in}}}}, \frac{1}{\sqrt{\text{dim}_{\text{in}}}}]$, where $\text{dim}_{\text{in}}$ denotes the input feature dimension.\looseness=-1

\textbf{Ensemble selection.} For each dataset, we train an initial pool of 128 MLPs and select the top $N$ models to form the world model ensemble $\ensemble$, which remains fixed during subsequent policy training. 
Model selection is based on a held-out validation set. For continuous control tasks, we rank the models by MSE on a validation set consisting of 1000 transitions from $\D$, following MOPO~\citep{yu2020mopo}. For the bandit task, we rank by negative log-likelihood (NLL) that captures uncertainty, since the true reward follows Bernoulli distribution. In our main experiments (\autoref{sec:benchmarking}), we use $N = 100$. For the sensitivity analysis (\autoref{sec:analysis}), we  vary $N \in \{5, 20\}$.

\textbf{Ensemble training.} Each model is trained using standard maximum likelihood estimation (MLE). For continuous control tasks,  we sample transition tuples $(s,a,r,s')$ from $\D$ using a batch size of $256$ to estimate the MLE loss. The inputs $(s,a)$ and outputs $(r,s')$ are standardized by subtracting the mean and dividing by the standard deviation computed over the dataset; during inference, predictions are inverse-transformed to restore the original scale.
We use AdamW optimizer~\citep{loshchilov2019decoupled} with a weight decay coefficient of \num{5e-5}, and a learning rate of \num{1e-3} for locomotion tasks and \num{3e-4} for Adroit tasks. Training is terminated early if the validation MSE fails to improve by more than $0.01$ relative within five consecutive epochs, following the early stopping procedure in MOBILE. In the bandit task, the model learning rate is \num{1e-3}, batch size is $128$, and improvement threshold is $0.001$ absolute.

\subsection{Details on Planning}

At the start of planning, following \autoref{sec:how_longhorizon}, we sample a batch of initial histories $h_t \sim \D$, drawn uniformly across time steps. The batch size is fixed at $100$, regardless of the ensemble size $N$. Since the values of $N$ used  in our experiments $\{5,20,100\}$ are divisors of $100$, we distribute the histories evenly across ensemble members. 
The recurrent policy $\pi_\nu$ initializes its hidden state $z_t$ with $h_t$ and then interacts with each world model to generate corresponding rollout until reaching the stopping criterion described in \autoref{app:algorithm}.  
The uncertainty threshold $\uncthres$ for truncation is fixed at $\zeta = 1.0$ in our main experiments and varied to $\{0.9,0.99,0.999\}$ in the sensitivity analysis. 
The entire planning process is parallelized on a GPU and thus introduces only negligible time costs.

For the bandit task, since there are no transition dynamics and hence no compounding error, we directly optimize the Bayesian objective: planning starts at $t=0$ and truncates at $t=T$.

\subsection{Details on Recurrent Off-Policy RL}

\textbf{Off-policy loss implementation.} For continuous control tasks, we follow a recent recurrent off-policy RL algorithm RESeL~\citep{luo2024efficient} to use REDQ~\citep{chen2021randomized} as the per-step loss for $l(\cdot)$, used in the overall RL loss defined by \autoref{eq:rl_loss}. REDQ builds on SAC~\citep{haarnoja2018soft,haarnoja2018soft2}\footnote{Our REDQ implementation is adapted from the SAC-N Equinox codebase: \url{https://github.com/Howuhh/sac-n-jax}.}, maintaining an ensemble of 10 critic MLPs and sampling 2 of them to form bootstrapped targets in the critic loss, while the actor maximizes the average Q-value over all ensemble members. 
REDQ uses in-target minimization to reduce overestimation, and may induce mild underestimation. In our use of REDQ, it \textit{mainly}  serves as a critic-stabilization component, because our notion of ``explicit conservatism'' refers to the deliberate injection of pessimism into offline policy learning. See \autoref{app:ablation} for an ablation replacing REDQ with standard SAC to show that REDQ plays a secondary role on \algo. 

For the bandit task, we adopt dueling DQN~\citep{wang2016dueling} following memoroid~\citep{morad2024recurrent} as the discrete control algorithm. Exploration is $\epsilon$-greedy, annealed from $1.0$ to $0.1$ over the first $10\%$ of gradient steps.

\textbf{Agent architecture implementation.} 
As introduced in \autoref{sec:recurrent_rl}, our agent consists of a recurrent actor $\pi_\nu: \mathcal H_t \to \Delta(\A)$ and a recurrent critic $Q_\omega: \mathcal H_t \times \A \to \R^{10}$. The critic outputs an ensemble of 10 Q-values, following the REDQ design adopted in RESeL. Both actor and critic maintain their own RNN encoders, $\nu_\phi: \mathcal H_t \to \mathcal Z$ and $\omega_\phi: \mathcal H_t \to \mathcal Z$, which share the same architecture but are optimized independently. As mentioned earlier, we adopt the \textit{memoroid} framework~\citep{morad2024recurrent}\footnote{\url{https://github.com/proroklab/memoroids}.} to use the linear recurrent unit (LRU)~\citep{orvieto2023resurrecting} as the backbone encoder for both $\nu_\phi$ and $\omega_\phi$. The actor and critic MLP heads have 2 or 3 hidden layers with a hidden size of $256$. For the bandit task, there is only a critic MLP head with 2 hidden layers.

The LRU begins with a nonlinear preprocessing layer that projects the raw input history $h_t$ into a 256-dimensional feature space (preserving the time dimension). This is followed by a stack of two LRU layers, each performing a linear recurrence update parameterized by a complex-valued diagonal matrix. Each layer maintains a hidden state of size $128$, resulting in a recurrent representation $z_t \in \C^{2 \times 128}$. From this recurrent state, the model produces a real-valued output vector $\tilde{z}_t \in \R^{128}$ via a nonlinear projection. During training and inference, $\tilde{z}_t$ is fed into the actor or critic MLP heads, while the complex hidden state $z_t$ is preserved for recurrent updates during policy inference.

We fix the recurrent architecture, including hidden sizes, across all experiments. However, \algo is compatible with any RNN encoder, and we leave a study of architectural variations to future work. For the ablation study with Markov agent, we remove the LRU encoder and only train the actor-critic MLP.

\textbf{Tape-based batching.} 
Classic recurrent RL relies on \emph{segment-based batching}~\citep{ni2022recurrent}, where sequences are padded to a fixed length, forming 3D tensors of shape $(\text{batch size}, \text{sequence length}, \text{dim})$ with NaN masks for shorter sequences. This wastes memory and lowers sample efficiency.
Memoroid~\citep{morad2024recurrent} introduces \emph{tape-based batching}, which exploits the monoid algebra of linear RNNs such as LRUs. Instead of padding, variable-length sequences are concatenated into a single 2D ``tape'', making the effective batch size equal to the sum of raw sequence lengths. Inline resets of hidden states prevent leakage across sequences within a tape~\citep{lu2023structured}, and computation over the tape is parallelized using associative scan in JAX~\citep{jax2018github}. To enable JIT compilation, a fixed tape length is enforced by dropping any trailing timesteps that exceed this length. We refer readers to~\citet{morad2024recurrent} for full details.

In our implementation, we adopt tape-based batching and we set the tape length to be larger than the maximal episode length $T$ in each task. To reduce computation time, we choose a relatively small tape length, similar to prior work in online POMDPs~\citep{morad2024recurrent,luo2024efficient}. 

\textbf{Training hyperparameters.} \autoref{tab:hparams_fixed} summarizes the modules and hyperparameters fixed in our experiments. 
Consistent with IQL~\citep{kostrikov2021offline} and MOBILE~\citep{sun2023model}, we apply cosine learning rate decay only to the actor network (both the RNN encoder and MLP head), but not to the critic, in order to promote stability during the later stages of training. 

The REDQ (SAC) entropy coefficient $\alpha$ is auto-tuned with a target entropy of $-\mathrm{dim}(\A)$~\citep{haarnoja2018soft2} for all datasets except for D4RL and NeoRL Hopper datasets and D4RL AntMaze domain. In the D4RL hopper domain, we find our algorithm is sensitive to $\alpha$, consistent with prior recurrent RL work~\citep{luo2024efficient}. To address this, we follow MOBILE and fix $\alpha=0.2$ across all hopper datasets (four in D4RL and three in NeoRL), without further tuning. 

In AntMaze, sparse-reward navigation makes SAC sensitive to the choice of $\alpha$, e.g., ADMPO~\citep{lin2025any} uses a fixed $\alpha=0.05$ and disables the entropy term backup in critic loss. We follow their insights but instead tune the target entropy $\in \{-\mathrm{dim}(\A), -5\,\mathrm{dim}(\A), -10\,\mathrm{dim}(\A)\}$. We find $-10\,\mathrm{dim}(\A)$ works best overall and report it in our main results.

For the bandit task, we sweep the RNN encoder learning rate $\eta_\phi \in \{\num{1e-6}, \num{3e-6}, \num{1e-5}, \num{3e-5}, \num{1e-4}\}$ in the critic network. Consistent with observations from ReSEL~\citep{luo2024efficient} and our continuous control results, a small learning rate is crucial for stable training. We use $\eta_\phi=\num{3e-6}$ in our bandit experiments.

{\renewcommand{\arraystretch}{1.3}
\begin{table}[H]
    \centering
\footnotesize
    \vspace{1em}
    \caption{Fixed hyperparameters used in our recurrent agents. The last block (actor and policy entropy) is only used in continuous control.}
    \vspace{-0.5em}
    \begin{tabular}{cc}
    \toprule
     Module or Hyperparameter    &  Value \\
     \midrule
      Actor and critic RNN encoders  & 2-layer LRUs~\citep{orvieto2023resurrecting} \\
      RNN hidden state size & 256 \\ 
      Actor and critic heads & 2-layer MLPs (3 layers in Adroit) \\ 
      Basic block of MLP head & (Linear $\to$ LayerNorm $\to$ leaky ReLU) \\
      MLP head hidden size & 256 \\
     \midrule
      Batch size (i.e., tape length) & 2048 (1024 in Adroit, 1000 in bandit) \\
      Update-to-data (UTD) ratio & 0.05 (0.02 in bandit) \\ 
      Gradient steps & 2M (3M in AntMaze, 20k in bandit)  \\ 
      Replay buffer size & \makecell{Full size, i.e., \\(60M in AntMaze, 1M in bandit, 40M otherwise)} \\
      Discount factor $\gamma$ & 0.99 \\
      Critic head's learning rate & \num{1e-4} \\
      Gradient norm clipping & 1000 (10000 in Adroit, 1 in bandit and AntMaze) \\ 
      \midrule
      Actor head's learning rate & \num{1e-4} \\
      Actor's learning rate decay & Cosine decay to 0.0 \\
      Entropy coef. $\alpha$'s learning rate & \num{1e-4} \\
      Entropy coef. $\alpha$ & \makecell{Auto-tuned with target $-\mathrm{dim}(\A)$ \\(AntMaze uses $-10\,\mathrm{dim}(\A)$ while Hopper uses fixed $\alpha=0.2$)} \\
    \bottomrule
    \end{tabular}
    \label{tab:hparams_fixed}
\end{table}
}

\begingroup
\rowcolors{2}{gray!10}{white}
\begin{table}[H]
\footnotesize
    \centering
    \vspace{1em}
    \caption{Best hyperparameters per dataset in the D4RL locomotion benchmark. We sweep $\eta_\phi \in \{\num{3e-7}, \num{1e-6}, \num{3e-6}, \num{1e-5}, \num{3e-5}\}$ and $\kappa \in \{0.05, 0.5, 0.8\}$.}
    \vspace{-0.5em}
    \begin{tabular}{ccc}
    \toprule
    Dataset & RNN encoder lr $\eta_\phi$ & Real data ratio $\kappa$\\ 
    \midrule
    halfcheetah-random-v2 & \num{3e-5} & $0.8$ \\
    hopper-random-v2 & \num{1e-5} & $0.5$ \\
    walker2d-random-v2 & \num{3e-5} & $0.5$ \\
    halfcheetah-medium-replay-v2 & \num{1e-5} & $0.05$ \\
    hopper-medium-replay-v2 & \num{3e-7} & $0.5$ \\
    walker2d-medium-replay-v2 & \num{1e-6} & $0.5$ \\
    halfcheetah-medium-v2 & \num{3e-5} & $0.8$ \\
    hopper-medium-v2 & \num{1e-6} & $0.8$ \\
    walker2d-medium-v2 & \num{3e-6} & $0.5$ \\
    halfcheetah-medium-expert-v2 & \num{3e-5} & $0.8$ \\
    hopper-medium-expert-v2 & \num{3e-6} & $0.5$ \\
    walker2d-medium-expert-v2 & \num{1e-6} & $0.8$ \\
    \bottomrule
    \end{tabular}
    \label{tab:sweep_d4rl_loco}
\end{table}

\begin{table}[H]
\footnotesize
    \centering
    \vspace{1em}
    \caption{Best hyperparameters per dataset in the NeoRL locomotion benchmark. 
    We sweep $\eta_\phi \in \{\num{3e-7}, \num{1e-6}, \num{3e-6}, \num{1e-5}, \num{3e-5}\}$ 
    and $\kappa \in \{0.5, 0.8\}$.}
    \vspace{-0.5em}
    \begin{tabular}{ccc}
    \toprule
    Dataset & RNN encoder lr $\eta_\phi$ & Real data ratio $\kappa$ \\ 
    \midrule
    HalfCheetah-v3-Low & \num{1e-5} & $0.8$ \\
    Hopper-v3-Low & \num{3e-6} & $0.8$ \\
    Walker2d-v3-Low & \num{3e-7} & $0.5$ \\
    HalfCheetah-v3-Medium & \num{3e-5} & $0.5$ \\
    Hopper-v3-Medium & \num{3e-6} & $0.5$ \\
    Walker2d-v3-Medium & \num{3e-7} & $0.8$ \\
    HalfCheetah-v3-High & \num{3e-5} & $0.5$ \\
    Hopper-v3-High & \num{3e-6} & $0.5$ \\
    Walker2d-v3-High & \num{3e-7} & $0.8$ \\
    \bottomrule
    \end{tabular}
    \label{tab:sweep_neorl}
\end{table}

\begin{table}[H]
\footnotesize
    \centering
    \vspace{1em}
    \caption{Best hyperparameters per dataset in the D4RL Adroit benchmark. 
    We sweep  $\eta_\phi \in \{\num{1e-6}, \num{3e-6}, \num{1e-5}, \num{3e-5}, \num{1e-4}\}$ 
    and $\kappa \in \{0.5, 0.8\}$. For the remaining Adroit datasets, all hyperparameter settings yield near-zero performance. }
    \vspace{-0.5em}
    \begin{tabular}{ccc}
    \toprule
    Dataset & RNN encoder lr $\eta_\phi$ & Real data ratio $\kappa$ \\ 
    \midrule
    pen-human-v1 & \num{3e-5} & $0.5$ \\
    pen-cloned-v1 & \num{1e-5} & $0.8$ \\
    hammer-cloned-v1 & \num{1e-5} & $0.5$ \\
    \bottomrule
    \end{tabular}
    \label{tab:sweep_adroit}
\end{table}

\begin{table}[H]
\footnotesize
    \centering
    \vspace{1em}
    \caption{Best hyperparameters per dataset in the D4RL AntMaze benchmark. 
    We sweep  $\eta_\phi \in \{\num{1e-6}, \num{3e-6}, \num{1e-5}, \num{3e-5}\}$ 
    and $\kappa \in \{0.5, 0.8, 0.95\}$. For the large maze datasets, all hyperparameter settings yield near-zero performance. 
    }
    \vspace{-0.5em}
    \begin{tabular}{ccc}
    \toprule
    Dataset & RNN encoder lr $\eta_\phi$ & Real data ratio $\kappa$ \\ 
    \midrule
    antmaze-umaze-v2 & \num{3e-6} & $0.95$ \\
    antmaze-umaze-diverse-v2 & \num{3e-6} & $0.95$ \\
    antmaze-medium-play-v2 & \num{1e-5} & $0.95$ \\
    antmaze-medium-diverse-v2 & \num{1e-6} & $0.95$ \\
    \bottomrule
    \end{tabular}
    \label{tab:sweep_antmaze}
\end{table}
\endgroup

\subsection{Computation Details}
\label{app:compute}

This subsection provides additional details on rollout costs, model training time, and parameter size to complement the summary in \autoref{sec:computation}.

\begin{table}[h]
\vspace{1em}
    \caption{\textbf{Time cost (seconds) of the Rollout function} (\autoref{algo:rollout}) for different ensemble sizes $N$ and numbers of parallel rollouts $K$. In our main experiments, $N = K = 100$.}
    \vspace{-0.5em}
    \centering
    \begin{tabular}{c|ccc}
              & $K=5$  & $K=20$ & $K=100$ \\
    \midrule
      $N=5$   & 2.7s  & 3.0s  & 4.7s \\ 
     $N=20$   & N/A  & 3.6s  & 5.0s \\ 
     $N=100$   & N/A  & N/A & 5.3s \\ 
    \end{tabular}
    \label{tab:rollout_cost}
\end{table}

\textbf{Rollout costs.} \autoref{tab:rollout_cost} reports the time cost of our Rollout implementation for halfcheetah-medium-expert-v2. We assign each ensemble member an equal number of rollouts, so we only benchmark configurations where $K$ is divisible by $N$. Thanks to full vectorization over ensemble members and rollouts using \texttt{jax.vmap}, rollout inference is very efficient: increasing $N$ or $K$ has only a minor effect on runtime. Consequently, rollout cost is negligible, and agent training time is dominated by gradient updates rather than rollout generation.

\begin{table}[h]
    \centering
    \vspace{1em}
    \caption{\textbf{World model training time (per seed)} for different total ensemble size $N_{\text{total}}$. In practice, we use $N_{\text{total}} = 128$.}
    \begin{tabular}{ccc}
       $N_{\text{total}} = 8$ & $N_{\text{total}} = 32$ & $N_{\text{total}} = 128$  \\
      \midrule
       1.2 hrs (1x)  & 1.7 hrs (1.42x) & 6.0 hrs (5x) \\
    \end{tabular}
    \vspace{1em}
    \label{tab:model_cost}
\end{table}

\textbf{Model training time.} For completeness, \autoref{tab:model_cost} reports the training time for several choices of $N_{\text{total}}$ on halfcheetah-medium-expert-v2 to provide additional compute context. The training time is \emph{roughly sublinear} in $N_{\text{total}}$. In practice, we train one ensemble of size $N_{\text{total}} = 128$ and select the top $N=5,20,100$.

\textbf{Parameter size.} In the locomotion benchmarks, the world model ensemble with $N=100$ contains fewer than 15M parameters in total. The recurrent encoder has fewer than 600k parameters, the actor MLP head fewer than 200k, and the critic MLP head (with an ensemble size of 10) fewer than 1.5M.

\section{Further Results and Discussion}
\label{app:results}

\subsection{Full Benchmarking Results}

\vspace{1em}
\begingroup
\setlength{\tabcolsep}{2pt}       %
\renewcommand{\arraystretch}{1.1} %
\begin{table}[H]
\caption{Comparison of offline RL methods on the \textbf{NeoRL locomotion} benchmark. We report mean normalized scores for all baselines, with {\scriptsize $\pm$std} for competitive baselines. The \textbf{best mean score} is bolded, and \texthl{hlcolor}{marked} methods are statistically similar under a $t$-test. Our results use 6 seeds, each evaluated at the final step with 20 episodes.\looseness=-1}
\vspace{-0.5em}
\label{tab:neorl_loco}
\centering
\resizebox{\columnwidth}{!}{%
\begin{tabular}{lrrrrrrrrrrrrr}
\toprule
 & \multicolumn{2}{c}{Model-free} & \multicolumn{8}{c}{Conservative model-based} & \multicolumn{2}{c}{Bayesian-inspired} & \multicolumn{1}{c}{Ours} \\
\cmidrule(lr){2-3} \cmidrule(lr){4-11} \cmidrule(lr){12-13} \cmidrule(lr){14-14}
Dataset & EDAC & CQL & MOPO & COMBO & ROMI & MOBILE & LEQ & ADMPO & ScorePen & VIPO & MAPLE & MoDAP & \algo \\
\midrule
hc-Low  & \valunc{31.3}{} & \valunc{38.2}{} & \valunc{40.1}{} & \valunc{32.9}{} & \valunc{35.8}{3.9} & \valunc{54.7}{3.0} & \valunc{33.4}{1.6} & \valunc{52.8}{1.2} & \valunc{49.6}{1.2} & \highlight{\valunc{\textbf{58.5}}{0.1}} & \valunc{33.4}{} & \valunc{53.9}{1.1} & \valunc{53.1}{1.1} \\
hp-Low  & \valunc{18.3}{} & \valunc{16.0}{} & \valunc{6.2}{}  & \valunc{17.9}{} & \valunc{22.4}{0.5} & \valunc{17.4}{3.9} & \valunc{24.2}{2.3} & \valunc{22.3}{0.1} & \valunc{21.1}{2.3} & \highlight{\valunc{\textbf{30.7}}{0.3}} & \valunc{22.7}{} & \highlight{\valunc{26.1}{4.7}} & \highlight{\valunc{30.3}{2.9}} \\
wk-Low  & \valunc{40.2}{} & \valunc{44.7}{} & \valunc{11.6}{} & \valunc{31.7}{} & \valunc{36.4}{2.3} & \valunc{37.6}{2.0} & \valunc{65.1}{2.3} & \valunc{55.9}{3.8} & \valunc{51.4}{1.4} & \highlight{\valunc{\textbf{67.6}}{0.7}} & \valunc{33.9}{} & \valunc{51.3}{7.8} & \valunc{43.9}{5.9} \\
\midrule
hc-Med  & \valunc{54.9}{} & \valunc{54.6}{} & \valunc{62.3}{} & \valunc{50.8}{} & \valunc{57.7}{1.4} & \valunc{77.8}{1.4} & \valunc{59.2}{3.9} & \valunc{69.3}{1.7} & \valunc{77.4}{1.0} & \highlight{\valunc{80.9}{0.2}} & \valunc{69.5}{} & \highlight{\valunc{81.0}{2.3}} & \highlight{\valunc{\textbf{81.1}}{0.8}} \\
hp-Med  & \valunc{44.9}{} & \valunc{64.5}{} & \valunc{1.0}{}  & \valunc{56.3}{} & \valunc{46.6}{6.8} & \valunc{51.1}{13.3} & \highlight{\valunc{\textbf{104.3}}{5.2}} & \valunc{51.5}{5.0} & \valunc{90.9}{1.3} & \valunc{66.3}{0.2} & \valunc{27.7}{} & \valunc{44.2}{15.3} & \highlight{\valunc{95.7}{11.5}} \\
wk-Med  & \valunc{57.6}{} & \valunc{57.3}{} & \valunc{39.9}{} & \valunc{53.8}{} & \valunc{54.9}{2.0} & \valunc{62.2}{1.6} & \valunc{45.2}{19.4} & \valunc{70.1}{2.4} & \valunc{65.8}{1.6} & \highlight{\valunc{\textbf{76.8}}{0.1}} & \valunc{40.7}{} & \valunc{70.8}{3.1} & \valunc{50.5}{8.8} \\
\midrule
hc-High & \valunc{81.4}{} & \valunc{77.4}{} & \valunc{65.9}{} & \valunc{62.2}{} & \valunc{77.4}{2.7} & \valunc{83.0}{4.6} & \valunc{71.8}{8.0} & \valunc{84.0}{0.8} & \valunc{81.4}{1.0} & \highlight{\valunc{\textbf{89.4}}{0.6}} & \textemdash & \highlight{\valunc{84.1}{8.3}} & \valunc{68.3}{23.6} \\
hp-High & \valunc{52.5}{} & \valunc{76.6}{} & \valunc{11.5}{} & \valunc{63.2}{} & \valunc{65.9}{4.5} & \highlight{\valunc{87.8}{26.0}} & \highlight{\valunc{95.5}{13.9}} & \valunc{87.6}{4.9} & \valunc{86.3}{1.3} & \highlight{\valunc{\textbf{107.7}}{0.5}} & \textemdash & \valunc{52.4}{3.2} & \valunc{96.8}{7.7} \\
wk-High & \valunc{75.5}{} & \valunc{75.3}{} & \valunc{18.0}{} & \valunc{71.8}{} & \valunc{75.1}{1.8} & \valunc{74.9}{3.4} & \valunc{73.7}{1.1} & \highlight{\valunc{\textbf{82.2}}{1.9}} & \valunc{78.0}{1.8} & \highlight{\valunc{81.7}{1.0}} & \textemdash & \valunc{73.6}{2.8} & \valunc{62.7}{14.5} \\
\midrule
AVG     & \valunc{50.7}{} & \valunc{56.1}{} & \valunc{28.5}{} & \valunc{49.0}{} & \valunc{52.5}{} & \valunc{60.7}{} & \valunc{63.6}{} & \valunc{64.0}{} & \valunc{66.9}{} & \valunc{\textbf{73.3}}{} & \textemdash & \valunc{59.7}{} & \valunc{64.7}{} \\
\bottomrule
\end{tabular}
}
\end{table}
\endgroup

\begingroup
\setlength{\tabcolsep}{4pt}       %
\renewcommand{\arraystretch}{1.1} %
\begin{table}[H]
\vspace{1em}
\caption{Comparison of offline RL methods on the \textbf{D4RL Adroit} benchmark. We report mean normalized scores for all baselines, with {\scriptsize $\pm$std} for competitive baselines. The \textbf{best mean score} is bolded, and \texthl{hlcolor}{marked} methods are statistically similar under a $t$-test. Our results use 6 seeds, each evaluated at the final step with 20 episodes. In addition, we include the average dataset performance, denoted as $\pi_\D$, to provide a reference for \textbf{data quality}.}
\label{tab:d4rl_adroit}
\vspace{-0.5em}
\centering
\resizebox{\columnwidth}{!}{%
\begin{tabular}{lrrrrrrrrrrrr}
\toprule
 &  & \multicolumn{4}{c}{Model-free} & \multicolumn{6}{c}{Conservative model-based} & \multicolumn{1}{c}{Ours} \\
\cmidrule(lr){3-6} \cmidrule(lr){7-12} \cmidrule(lr){13-13}
Dataset & $\pi_\D$ & BC & EDAC & IQL & ReBRAC & MOPO & MOBILE & ARMOR & ROMI & MoMo & VIPO & \algo  \\
\midrule
pen-human     & 88.7 &  \valunc{71.0}{6.3} & \valunc{52.1}{8.6} & \valunc{78.5}{8.2} & \highlight{\valunc{\textbf{103.2}}{8.5}} & \valunc{10.7}{} & \valunc{30.1}{14.6} & \valunc{72.8}{13.9} & \valunc{63.8}{8.4} & \valunc{74.9}{} & \valunc{52.6}{7.7} & \valunc{20.8}{13.2} \\
pen-cloned    & 68.7 &  \valunc{51.9}{15.2} & \valunc{68.2}{7.3} & \valunc{83.4}{8.2} & \highlight{\valunc{\textbf{102.8}}{7.8}} & \valunc{54.6}{} & \valunc{69.0}{9.3} & \valunc{51.4}{15.5} & \valunc{57.7}{3.9} & \valunc{74.1}{} & \valunc{71.1}{9.5} & \highlight{\valunc{91.3}{21.2}} \\
door-human    & 7.7 &  \valunc{2.3}{4.0}  & \highlight{\valunc{10.7}{6.8}} & \valunc{3.3}{1.8}  & \valunc{-0.1}{0.0}  & \valunc{-0.2}{} & \valunc{-0.2}{0.1} & \valunc{6.3}{6.0}  & \valunc{0.0}{0.0} & \highlight{\valunc{\textbf{11.3}}{}} & \valunc{2.0}{0.3}  & \valunc{0.0}{0.0}  \\
door-cloned   &  3.2 & \valunc{-0.1}{0.0} & \highlight{\valunc{9.6}{8.3}}  & \valunc{3.1}{1.8}  & \valunc{0.1}{0.1}   & \highlight{\valunc{15.3}{}} & \highlight{\valunc{\textbf{24.0}}{22.8}} & \valunc{-0.1}{0.0} & \highlight{\valunc{14.3}{5.8}} & \valunc{5.8}{}  & \textemdash      & \valunc{0.0}{0.0}  \\
hammer-human  & 1.5 &  \highlight{\valunc{\textbf{3.0}}{3.4}}  & \valunc{0.8}{0.4}  & \highlight{\valunc{1.8}{0.8}}  & \valunc{0.2}{0.2}   & \valunc{0.3}{}  & \valunc{0.4}{0.2}  & \highlight{\valunc{1.9}{1.6}}  & \highlight{\valunc{2.2}{0.2}} & \valunc{1.7}{}  & \valunc{1.1}{0.9}  & \valunc{0.0}{0.0}  \\
hammer-cloned & 0.7&  \valunc{0.6}{0.2}  & \valunc{0.3}{0.0}  & \valunc{1.5}{0.7}  & \valunc{5.0}{3.8}   & \valunc{0.5}{}  & \valunc{1.5}{0.4}  & \valunc{0.7}{0.6}  & \valunc{0.0}{0.0} & \valunc{0.7}{}  & \valunc{2.1}{0.2}  & \highlight{\valunc{\textbf{14.4}}{9.7}} \\
\midrule
AVG           & N/A &  \valunc{21.5}{} & \valunc{23.6}{} & \valunc{28.6}{} & \valunc{\textbf{35.2}}{}  & \valunc{13.5}{} & \valunc{20.8}{} & \valunc{22.2}{} & \valunc{23.0}{} & \valunc{28.1}{} & \textemdash      & \valunc{21.1}{} \\
\bottomrule
\end{tabular}
}
\end{table}
\endgroup

\begingroup
\setlength{\tabcolsep}{4pt}       %
\renewcommand{\arraystretch}{1.1} %
\begin{table}[H]
\vspace{1em}
\caption{Comparison of offline model-based RL methods on the \textbf{D4RL AntMaze} benchmark. We report mean normalized scores for all baselines, with {\scriptsize $\pm$std} for competitive baselines. The \textbf{best mean score} is bolded, and \texthl{hlcolor}{marked} methods are statistically similar under a $t$-test. Our results use 6 seeds, each evaluated at the final step with 100 episodes.
In addition, we include the average dataset performance, denoted as $\pi_\D$, to provide a reference for \textbf{data quality}.}
\label{tab:d4rl_antmaze}
\vspace{-0.5em}
\centering
\resizebox{0.9\columnwidth}{!}{%
\begin{tabular}{lrrrrrrrrrr}
\toprule
Dataset & $\pi_\D$ & CBOP & MOBILE & RAMBO & COMBO & MOBILE$^\dagger$ & ROMI & ADMPO & LEQ & \algo  \\
\midrule
umaze     &  28.8 & \valunc{0.0}{0.0} & \valunc{0.0}{0.0} & 25.0 & 80.3 & \valunc{77.0}{} & \valunc{70.3}{11.4} & \valunc{88.4}{1.2} & \highlight{\valunc{\textbf{94.4}}{6.3}} & \valunc{66.1}{20.1} \\
umaze-diverse    & 1.0 &  \valunc{0.0}{0.0} & \valunc{0.0}{0.0} & 0.0 &  57.3 & \valunc{20.4}{} & \valunc{32.8}{13.0} & \highlight{\valunc{\textbf{81.7}}{8.6}} & \highlight{\valunc{71.0}{12.3}} & \highlight{\valunc{74.4}{12.1}} \\
medium-play    & 19.6 &  \valunc{0.0}{0.0} & \valunc{0.0}{0.0} & 16.4 &  0.0 & \highlight{\valunc{64.6}{}} & \highlight{\valunc{51.3}{19.7}} & \valunc{23.9}{6.3} & \highlight{\valunc{\textbf{58.8}}{33.0}} & \valunc{12.8}{19.2} \\
medium-diverse   & 11.3 &  \valunc{0.0}{0.0} & \valunc{0.0}{0.0} & 23.2 & 0.0 & \valunc{1.6}{} & \highlight{\valunc{30.1}{10.0}} & \valunc{24.1}{5.7} & \highlight{\valunc{\textbf{46.2}}{23.2}} & \valunc{19.4}{12.4} \\
large-play  & 10.5 & \valunc{0.0}{0.0} & \valunc{0.0}{0.0} & 0.0 & 0.0 & \valunc{2.6}{} & \valunc{0.0}{0.0} & \valunc{8.3}{4.1} & \highlight{\valunc{\textbf{58.6}}{9.1}} & \valunc{0.0}{0.0} \\
large-diverse & 10.6 &  \valunc{0.0}{0.0} & \valunc{0.0}{0.0} & 2.4 &  0.0 & \valunc{7.2}{} & \valunc{6.4}{4.2} & \valunc{0.0}{0.0} & \highlight{\valunc{\textbf{60.2}}{18.3}} & \valunc{0.5}{1.1} \\
\midrule
AVG           & N/A &  \valunc{0.0}{} & \valunc{0.0}{} & 11.2 &  22.9 & \valunc{28.9}{} & \valunc{31.1}{} & \valunc{37.7}{}  & \valunc{\textbf{64.9}}{} & \valunc{28.9}{} \\
\bottomrule
\end{tabular}
}
\vspace{1em}
\end{table}
\endgroup

\textbf{Source of benchmarked baseline results.} 
For the D4RL locomotion benchmark in \autoref{tab:d4rl_loco}, we adopt the results of CQL, EDAC, and MOPO from the MOBILE paper~\citep[Table 1]{sun2023model}, where the MOPO results correspond to the tuned variant (denoted as MOPO$^*$ therein). Results for the remaining baselines are taken directly from their respective original publications. For MAPLE, we report the improved variant that employs an ensemble size of $142$ (instead of the default $14$) as described in~\citep{chen2021offline}, to enable a fairer comparison with our method, which uses an ensemble size of $100$. Finally, we note that prior works may differ in the total number of gradient steps used for training compared to ours (2M steps in this benchmark); for example, MOPO, MOBILE, ADMPO, and VIPO report results after 3M steps, while LEQ, SUMO, and ROMI use 1M steps. We retain these numbers as reported, since we believe this setting reflects the most faithful comparison with previously published results.

For the NeoRL locomotion benchmark in \autoref{tab:neorl_loco}, we use EDAC results from the MOBILE paper~\citep{sun2023model}, CQL and MOPO from the NeoRL paper~\citep{qin2022neorl}, and COMBO from the MoDAP paper~\citep{choi2024diversification}. All other baselines are taken from their original publications. For MoDAP on Hopper-v3-High, we report the improved result obtained with an ensemble size of $40$, as presented by the authors to demonstrate the benefit of larger ensembles.

For the D4RL Adroit benchmark in \autoref{tab:d4rl_adroit}, we report BC, IQL, and ReBRAC results from the CORL paper~\citep{tarasov2023corl}, and MOPO from the MOBILE paper~\citep{sun2023model}. Results for other baselines are taken from their original publications. Standard deviations for MOPO and MoMo were not available, so we omit them.

For the D4RL AntMaze benchmark in \autoref{tab:d4rl_antmaze}, we report CBOP, MOBILE, RAMBO, COMBO, and LEQ from the LEQ paper~\citep{park2025model}, as well as an alternative tuning of MOBILE, denoted MOBILE$^\dagger$, and ADMPO from the ADMPO paper~\citep{lin2025any}. We omit standard deviations for RAMBO, COMBO, and MOBILE$^\dagger$ since they are not reported in the corresponding source papers.

\textbf{Statistical tests for benchmarking results.} To assess statistical significance, we conduct Welch’s one-sided $t$-test ($p <0.05$). In the reported tables, we highlight all methods whose average scores are not significantly different from the best-performing method.

\subsection{Full Results on Value Overestimation and Horizon Scales}
\label{app:unc_thres}

\autoref{fig:horizon_1}--\autoref{fig:horizon_4} report the full ablation results on truncation thresholds to complement \autoref{fig:horizon_main}, with the maximum rollout horizon for each batch of training rollouts shown in the last columns.

A key observation is that using a quantile threshold of $\zeta=0.9$ corresponds \textit{roughly} to a horizon cap of $1$–$10$, which mirrors the short fixed horizons commonly adopted in prior work. This shows that such prior choices are not suitable for guiding the design of Bayesian RL, where long horizons are essential.

\textbf{Summary on the effect of $\zeta$.} We further count the number of complete failures (i.e., scores $\le 5.0$) under different thresholds using \autoref{tab:ablation_full}. With $\zeta=0.9$, 16 datasets fail; with $\zeta=0.99$, 6 datasets fail; and with $\zeta=0.999$, none fail. This suggests that a safe range for $\zeta$ lies between $0.999$ and $1.0$. 
Although $\zeta=0.999$ often performs similarly to $1.0$ (as the resulting adaptive horizons are close), we observe clear advantages of $1.0$ on tasks such as D4RL walker2d-random-v2 and pen-cloned-v1. \textbf{Thus, we recommend using $\zeta=1.0$ as a starting point for our algorithm.}

\textbf{Summary on horizon scales.} 
Although \algo uses adaptive horizons, we report the empirical horizon scales (75th-percentile and maximum) to illustrate the typical horizon length required under our Bayesian formulation.
In D4RL and NeoRL locomotion tasks ($T=1000$), \algo uses 75th-percentile horizons of $2^4$-$2^6$ in 4 tasks, $2^7$ in 5 tasks, $2^8$ in 9 tasks, and $2^9$ in 3 tasks; the corresponding maximum horizons of $2^6$-$2^8$ in 4 tasks, $2^9$ in 4 tasks, $1000$ in 11 tasks. 
In Adroit ($T=100$ or $200$), the 75th-percentile horizon is $2^6$-$2^7$, and the maximum horizon reaches the episode length $T$.
In AntMaze ($T=700$ or $1000$), the 75th-percentile horizon is $2^4$-$2^8$; the maximum horizon reaches $2^8$-$2^9$.  
Overall, these statistics show that \algo selects 75th-percentile rollout horizons of \textbf{64-512 steps} and maximum horizons of \textbf{256-1000 steps}, in 21 out of 23 tasks with $T=1000$.

\begin{figure}[h]
    \centering
    \includegraphics[width=0.8\linewidth]{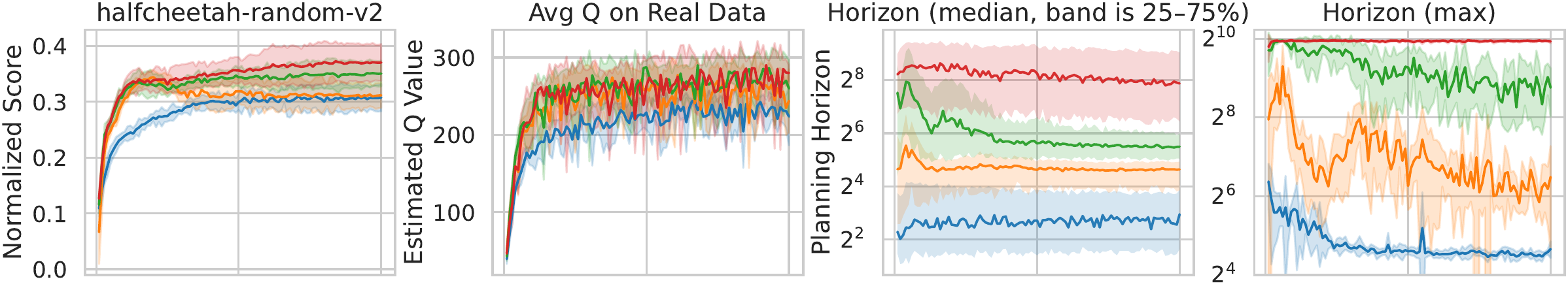}
    \includegraphics[width=0.8\linewidth]{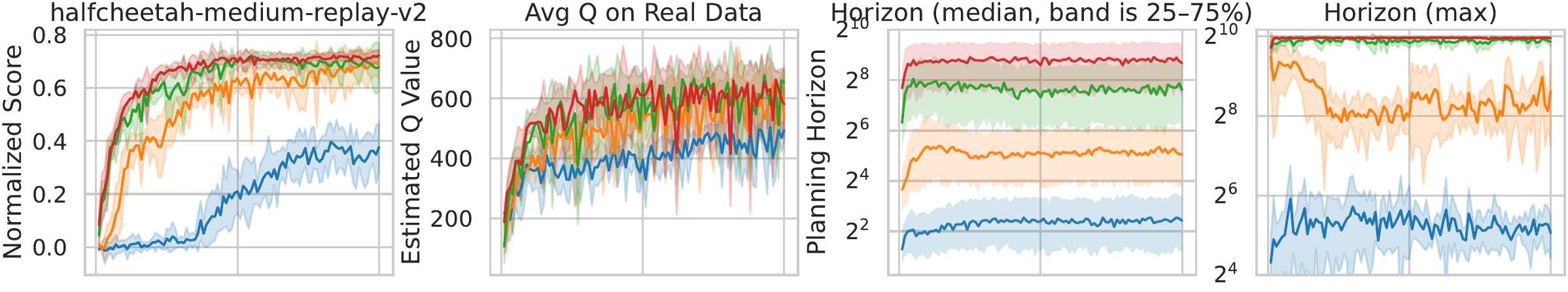}
    \includegraphics[width=0.8\linewidth]{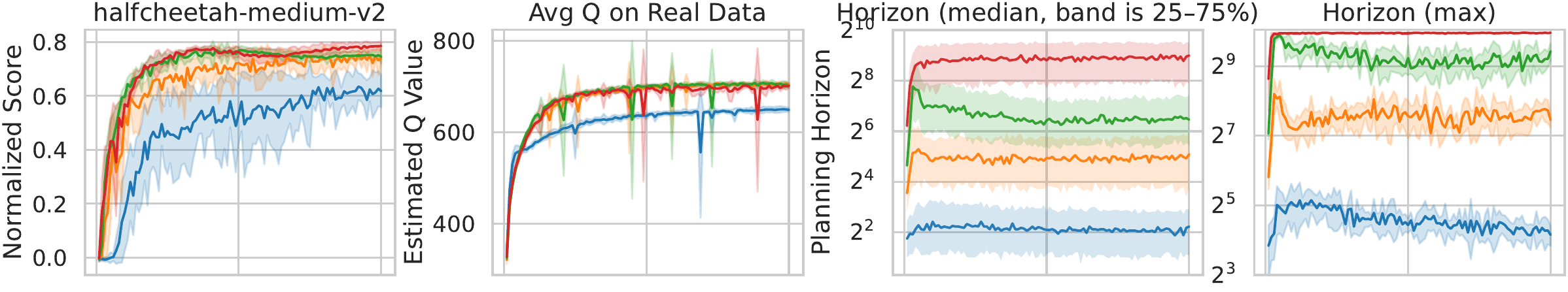}
    \includegraphics[width=0.8\linewidth]{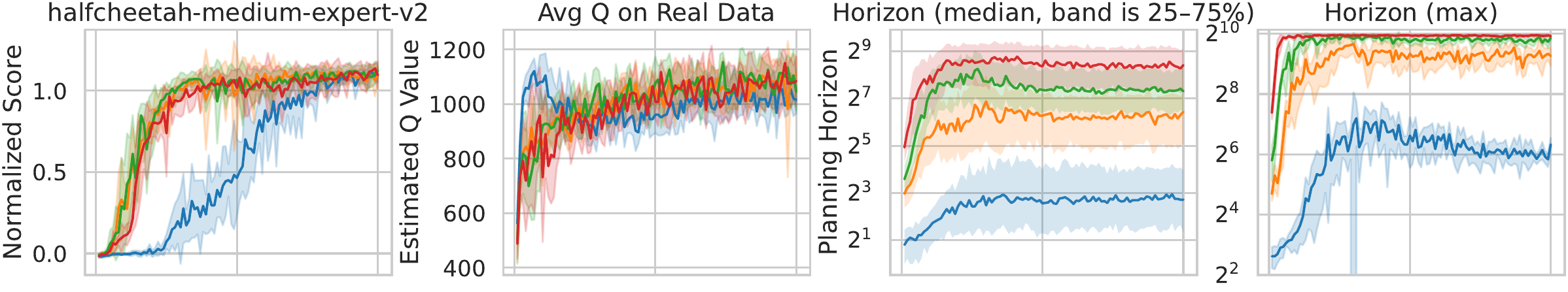}
    \\\vspace{1em}
    \includegraphics[width=0.8\linewidth]{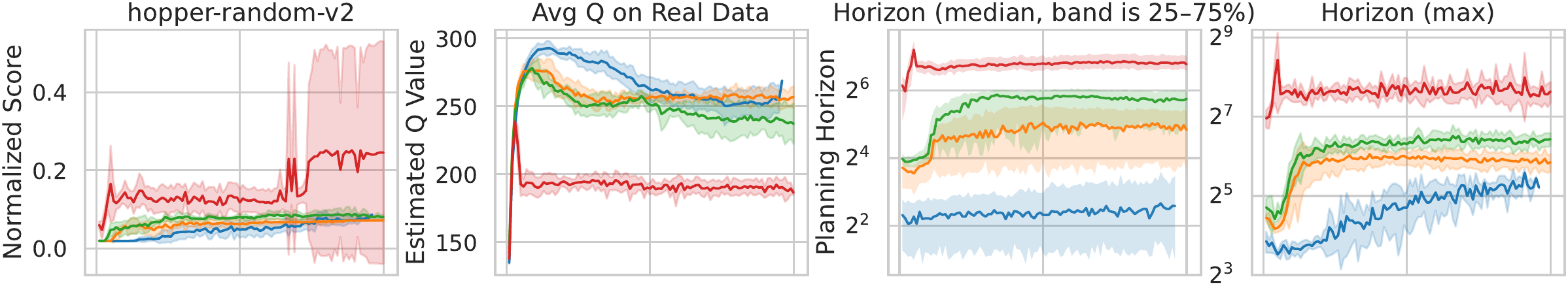}
    \includegraphics[width=0.8\linewidth]{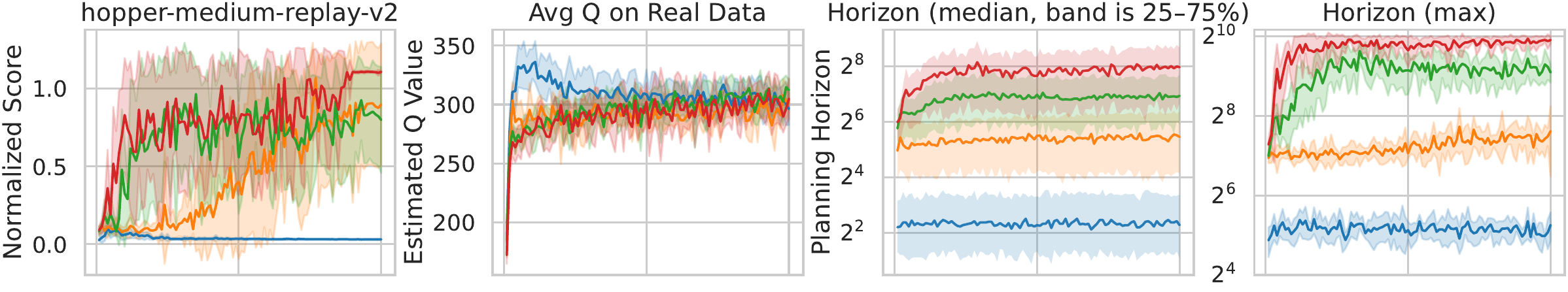}
    \includegraphics[width=0.8\linewidth]{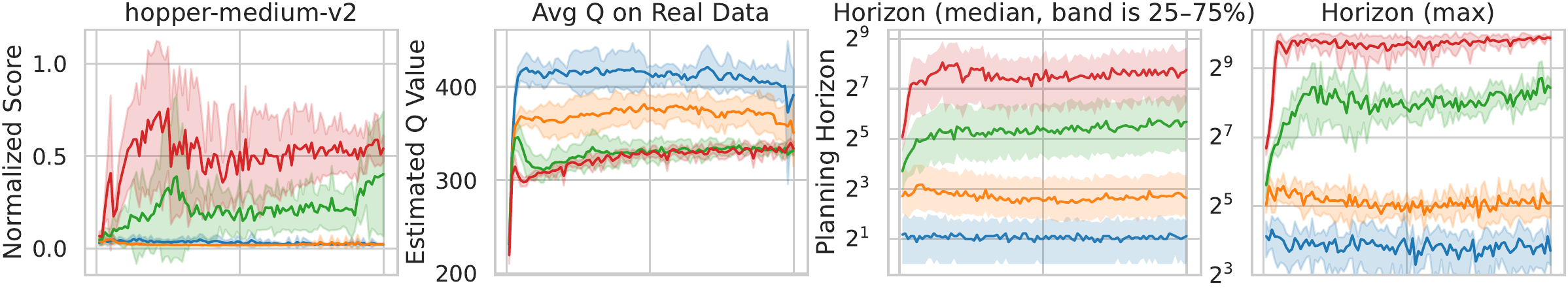}
    \includegraphics[width=0.8\linewidth]{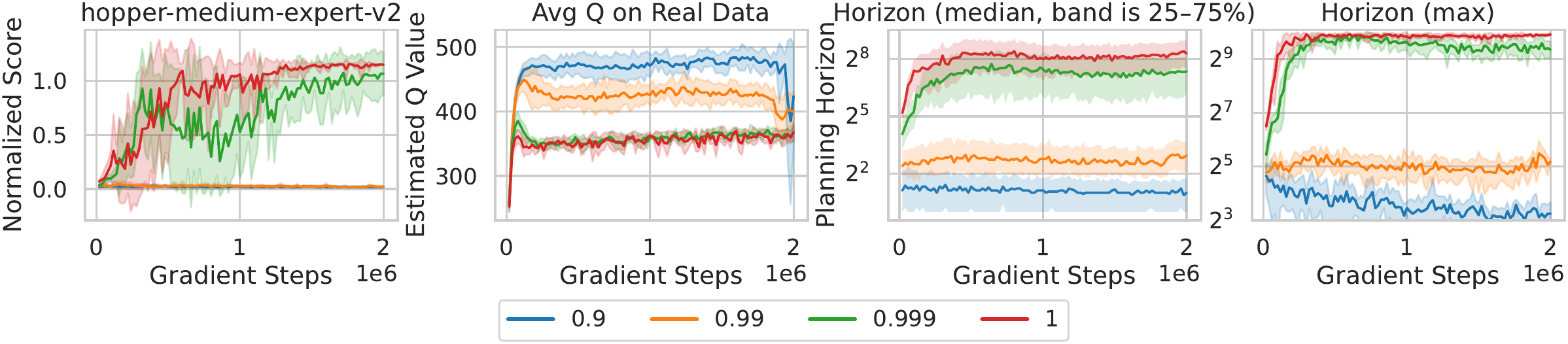}
    \caption{Ablation on the uncertainty quantile $\zeta$ for rollout truncation (part 1 of 4). The third column reports rollout horizon statistics (median with interquartile range) over 100 training-time rollouts, and the fourth column reports the maximum horizon.}
    \label{fig:horizon_1}
\end{figure}

\begin{figure}[h]
    \centering
    \includegraphics[width=0.8\linewidth]{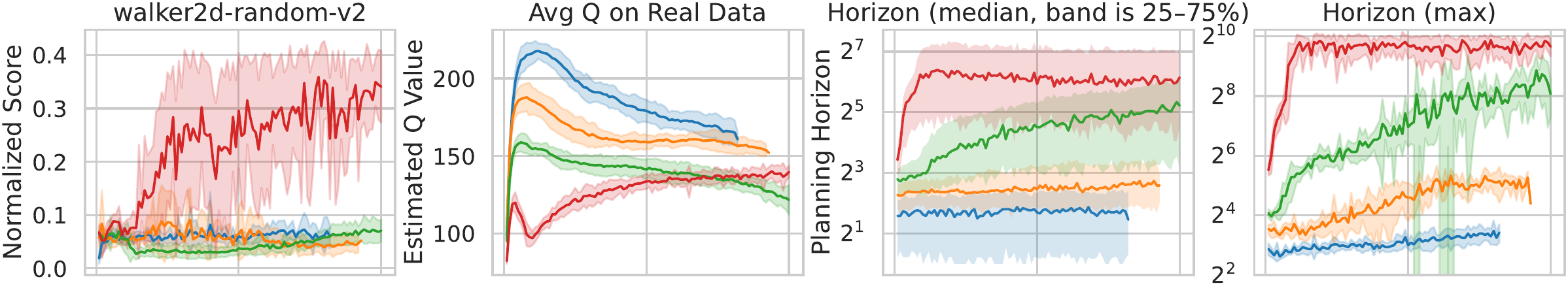}
    \includegraphics[width=0.8\linewidth]{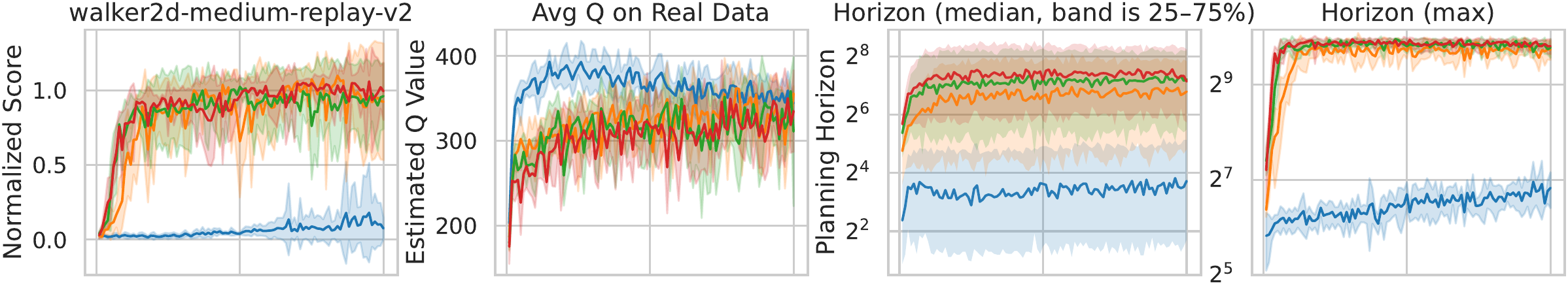}
    \includegraphics[width=0.8\linewidth]{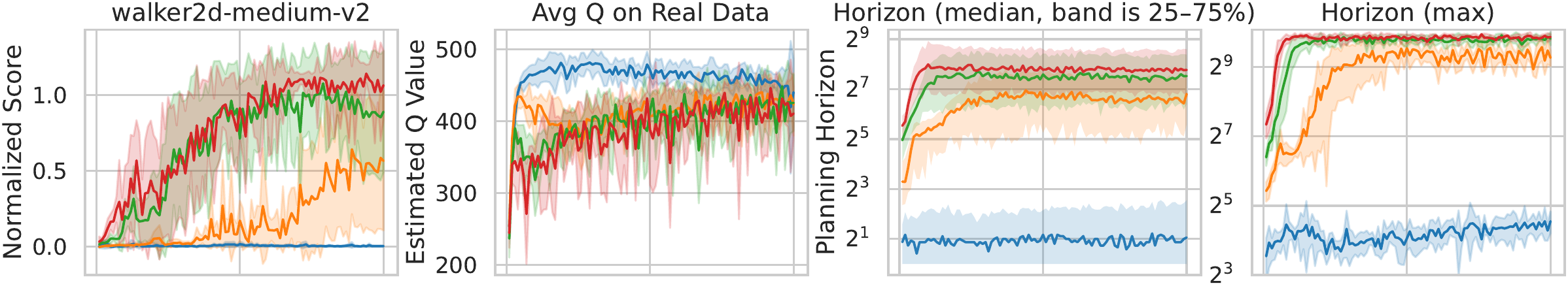}
    \includegraphics[width=0.8\linewidth]{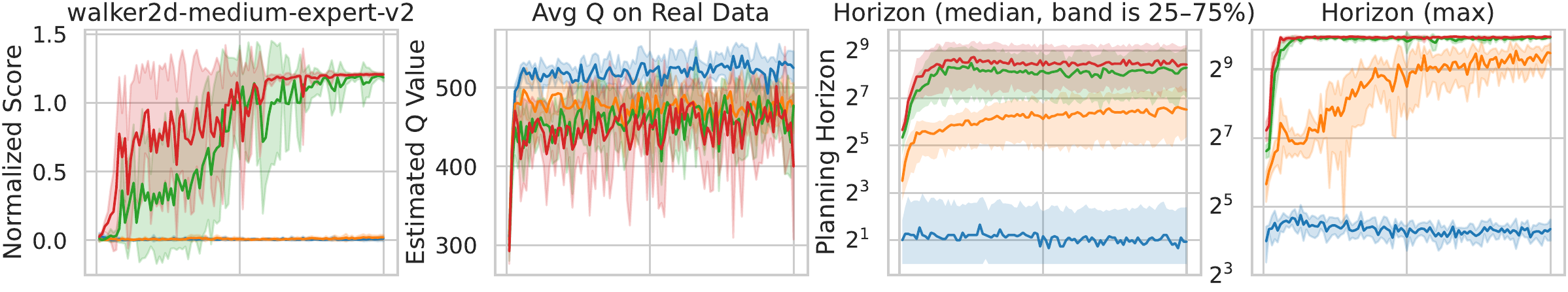}
    \\\vspace{1em}
    \includegraphics[width=0.8\linewidth]{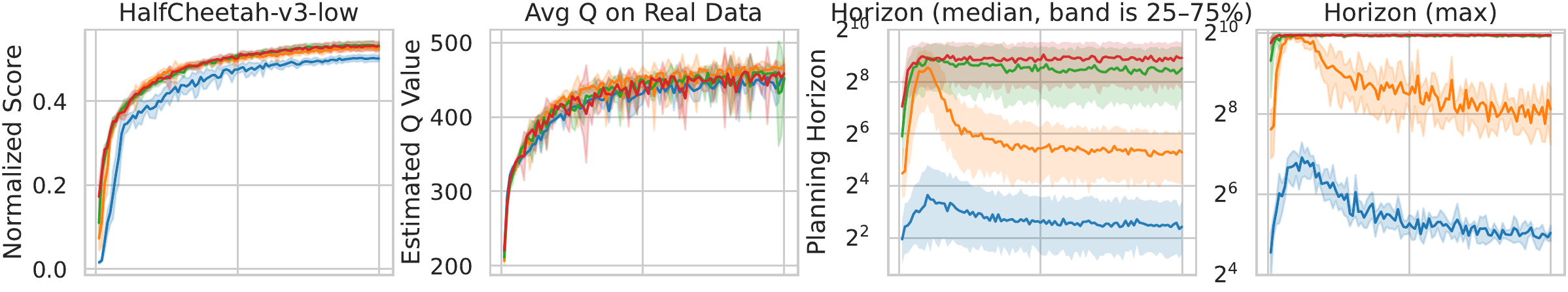}
    \includegraphics[width=0.8\linewidth]{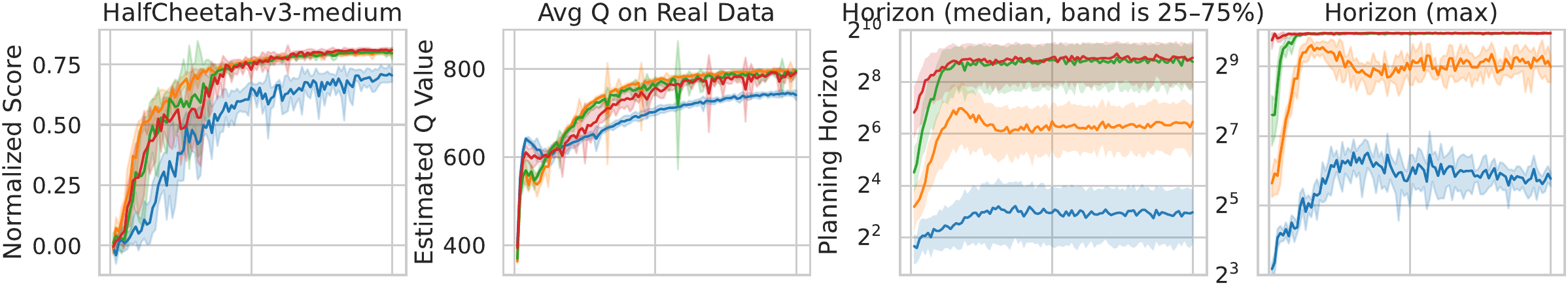}
    \includegraphics[width=0.8\linewidth]{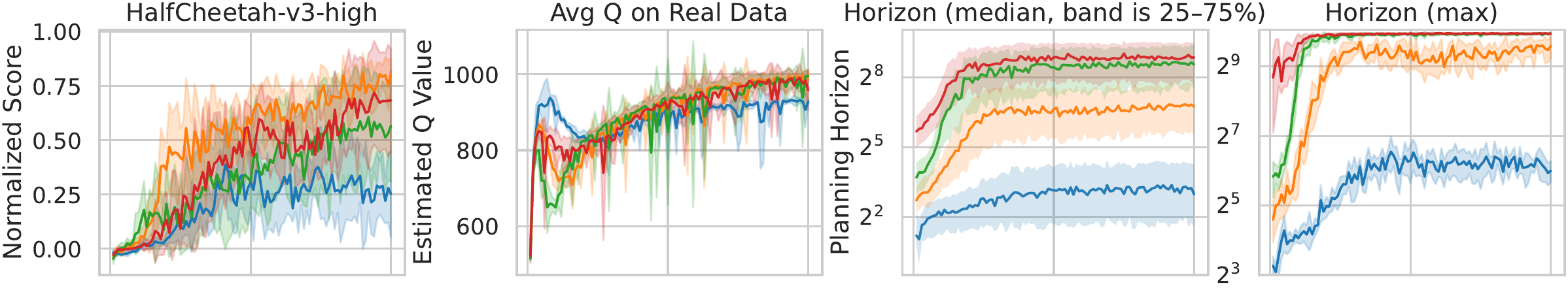}
    \\\vspace{1em}
    \includegraphics[width=0.8\linewidth]{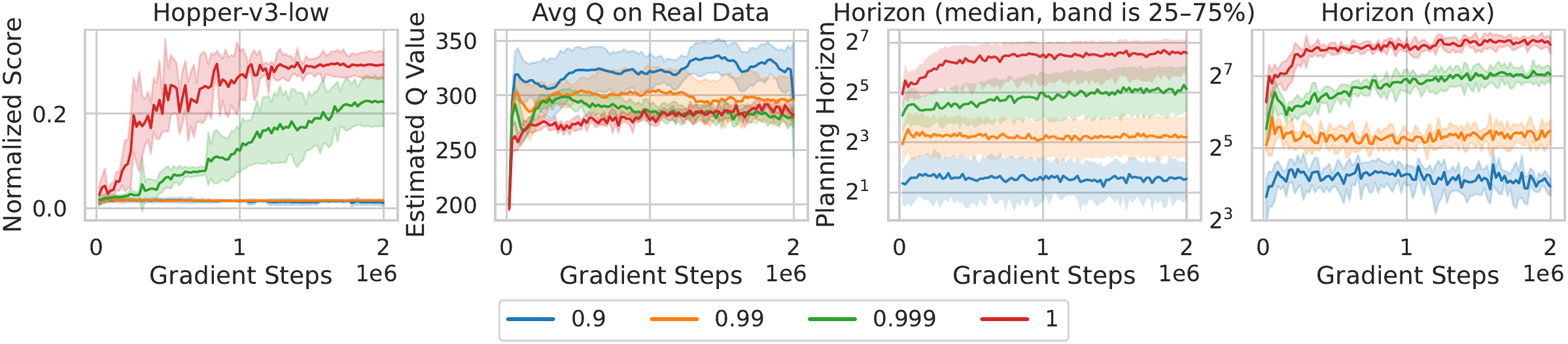}
    \caption{Ablation on the uncertainty quantile $\zeta$ for rollout truncation (part 2 of 4).}
    \label{fig:horizon_2}
\end{figure}

\begin{figure}[h]
    \centering
    \includegraphics[width=0.8\linewidth]{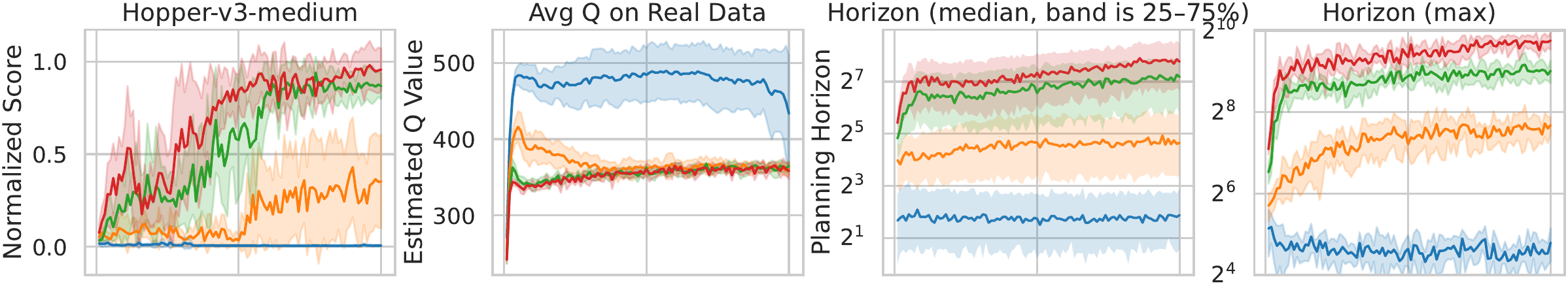}
    \includegraphics[width=0.8\linewidth]{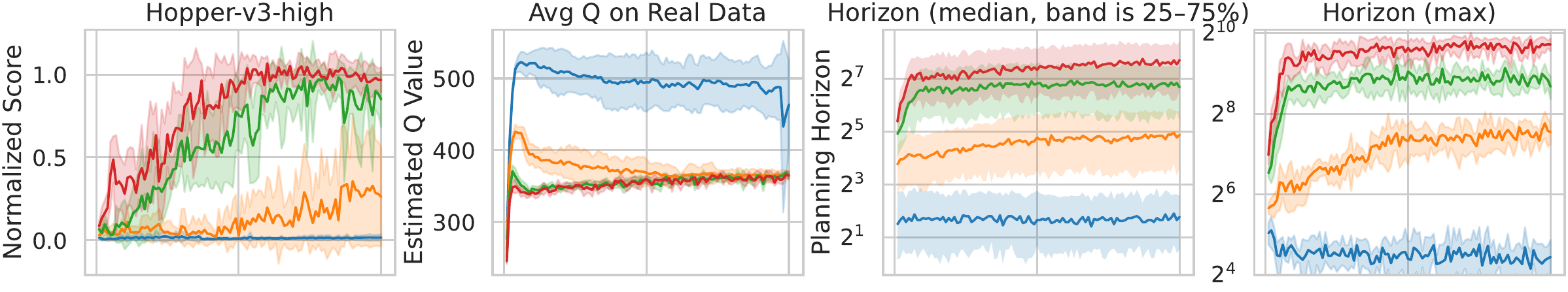}
    \\\vspace{1em}
    \includegraphics[width=0.8\linewidth]{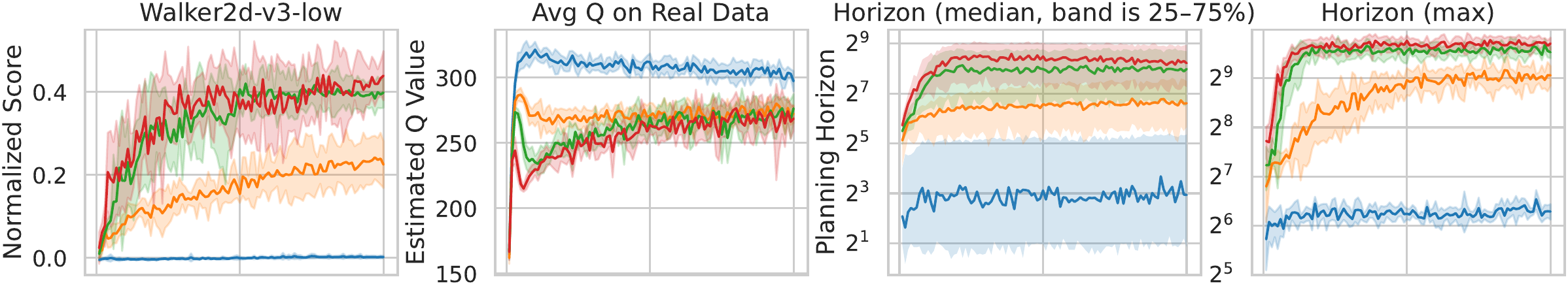}
    \includegraphics[width=0.8\linewidth]{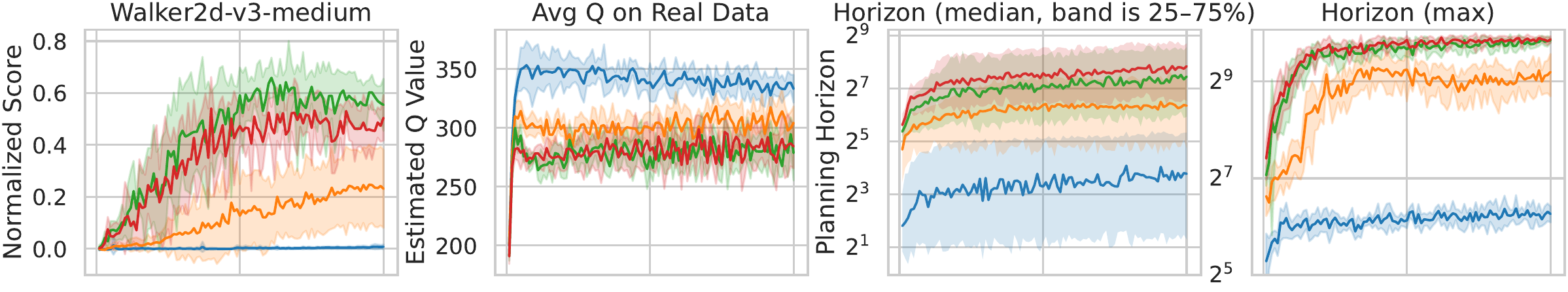}
    \includegraphics[width=0.8\linewidth]{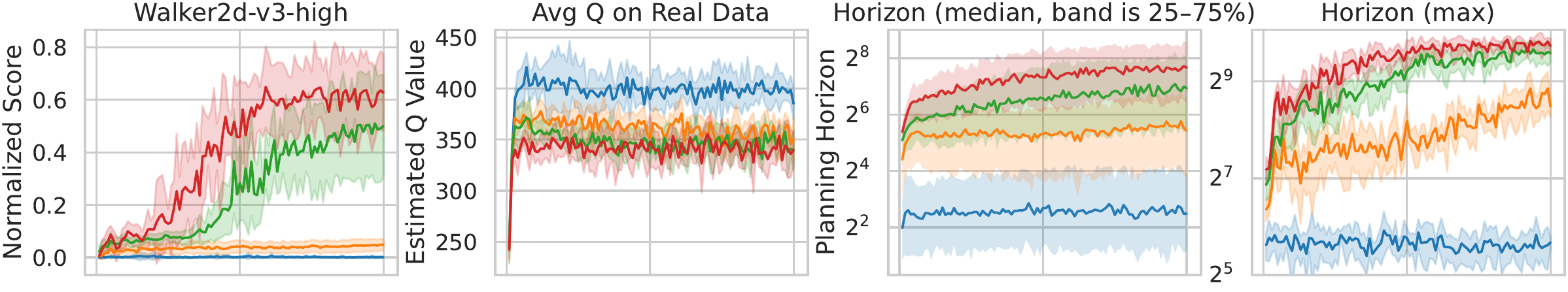}
    \\\vspace{1em}
    \includegraphics[width=0.8\linewidth]{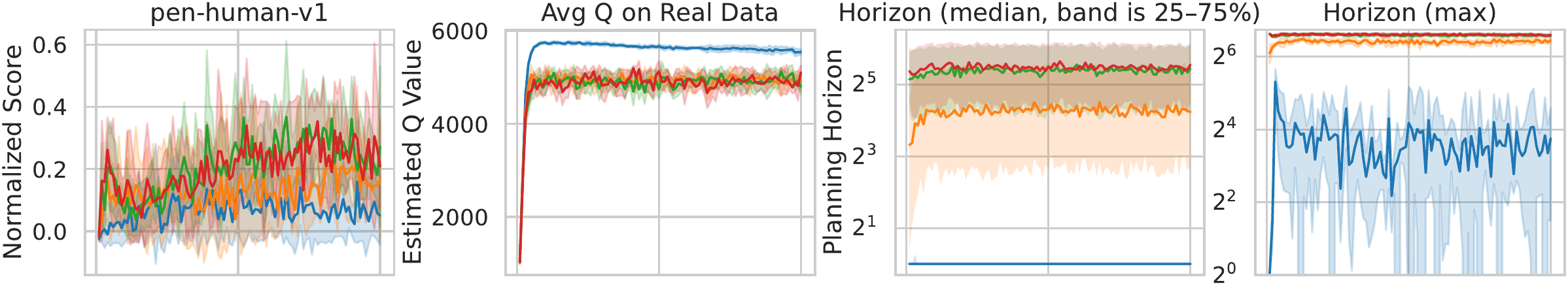}
    \includegraphics[width=0.8\linewidth]{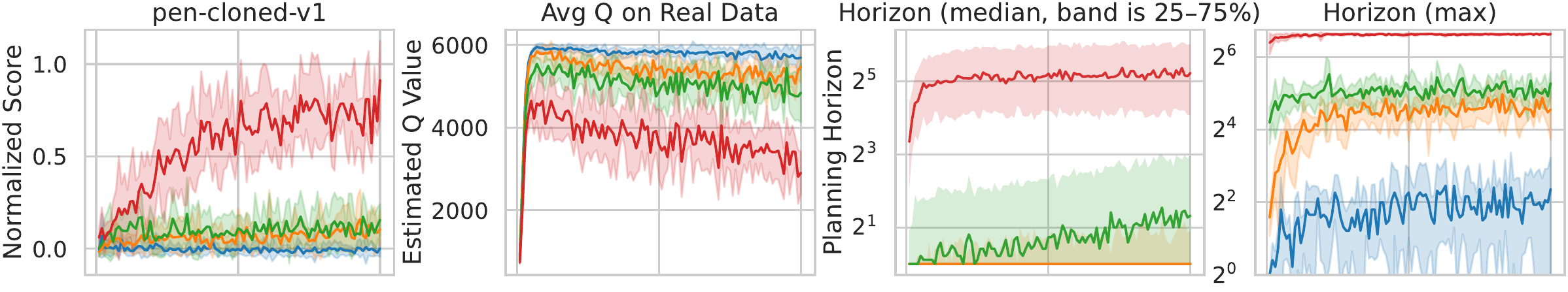}
    \includegraphics[width=0.8\linewidth]{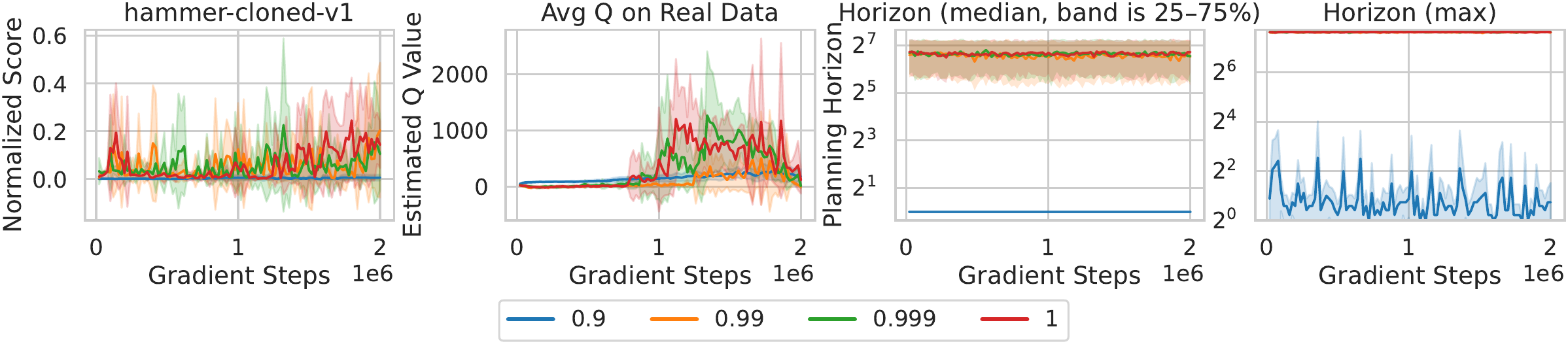}
    \caption{Ablation on the uncertainty quantile $\zeta$ for rollout truncation (part 3 of 4). Adroit benchmark has short maximum episode steps: $T=100 < 2^7$ in pen and $T=200 < 2^8$ in hammer, which limits the rollout horizon. }
    \label{fig:horizon_3}
\end{figure}

\begin{figure}[t]
\vspace{-3em}
    \centering
    \includegraphics[width=0.8\linewidth]{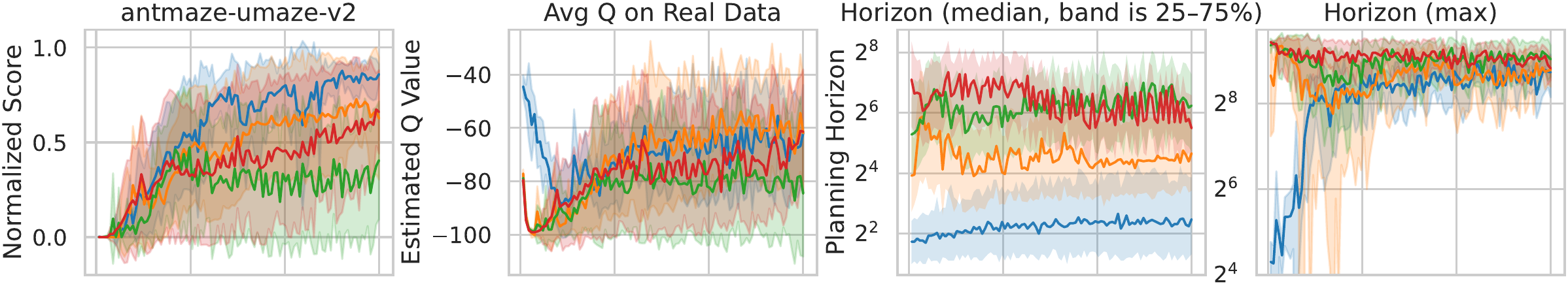}
    \includegraphics[width=0.8\linewidth]{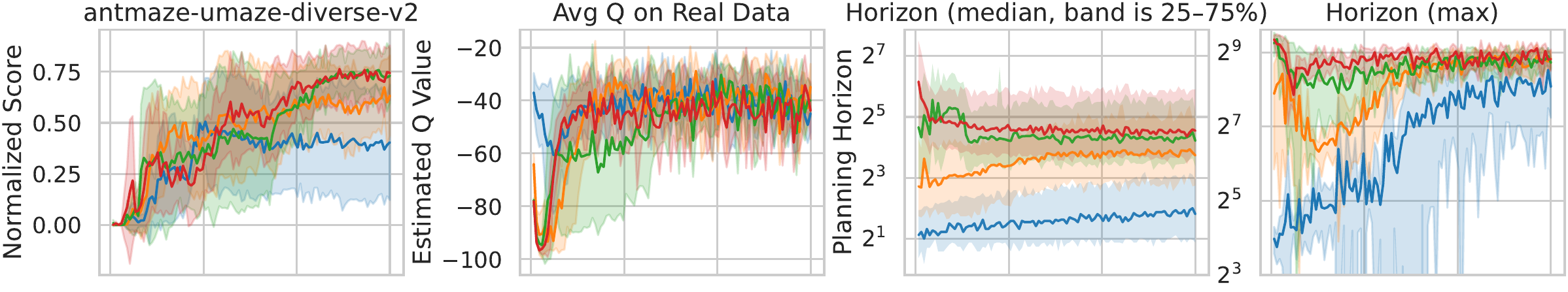}
    \includegraphics[width=0.8\linewidth]{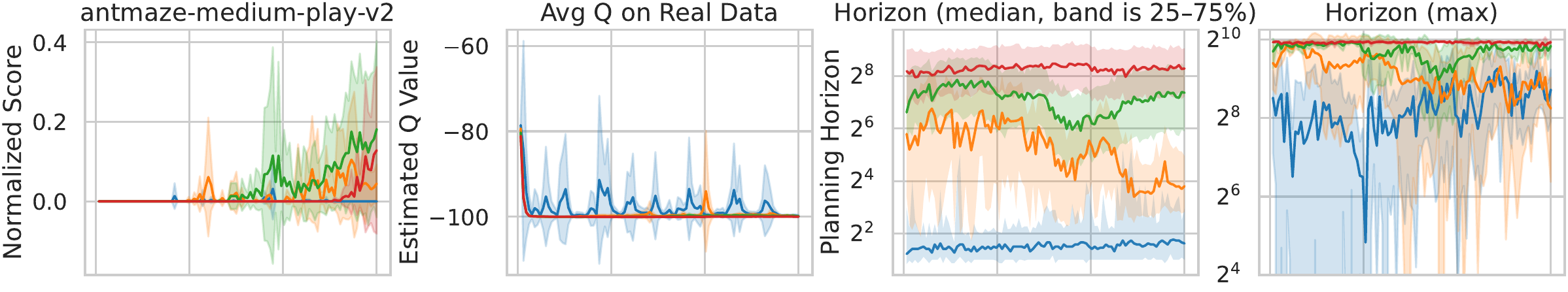}
    \includegraphics[width=0.8\linewidth]{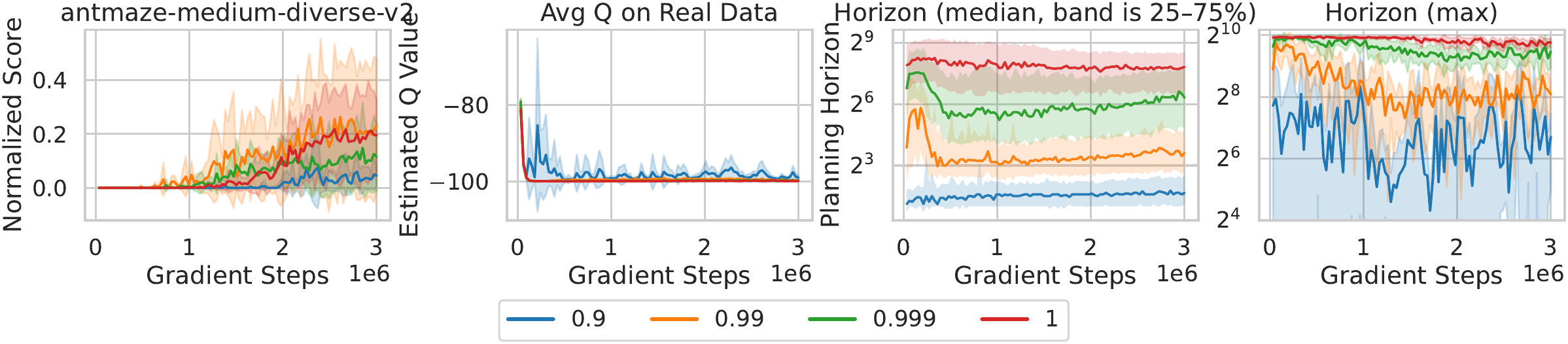}
    \caption{Ablation on the uncertainty quantile $\zeta$ for rollout truncation (part 4 of 4). Maximum episode steps are $T=700$ in umaze and $T=1000$ in medium maze. Successful episodes terminate early in AntMaze, so horizon lengths are partially confounded by this effect.\looseness=-1
    }
    \label{fig:horizon_4}
\end{figure}

\clearpage
\subsection{Full Results on Compounding Errors}
\label{app:LN}

\autoref{fig:compound_app1}--\autoref{fig:compound_app3} report the full results on compounding errors to complement \autoref{fig:compound_main} in the main paper.

\textbf{Plotting setup.} For each world ensemble trained given a dataset, we evaluate compounding errors using \textit{three} datasets that share the same underlying MDP ($\star$-random-v2, $\star$-medium-replay-v2, and $\star$-medium-expert-v2). For each evaluation dataset, we collect $200$ rollouts in total (two rollouts per ensemble member). This evaluation protocol allows us to span a broad range of exploratory behaviors, similar to \citet{zhou2025diffusion}.

We generate a synthetic rollout as follows: first we draw $m_\theta \sim \ensemble$ and an entire real trajectory $(s_{0:T}, a_{0:T-1}, r_{1:T})$ from the evaluation dataset. We then let $\hat s_0 = s_0$, and $(\hat r_{t+1}, \hat s_{t+1}) \sim m_\theta(\hat s_t, a_t), \forall t < T$.
Synthetic rollouts are truncated \textit{only} when numerical overflow occurs (\texttt{float32}); we do not apply an uncertainty threshold and we ignore the terminal function. Because rollout lengths vary in Hopper, we apply forward filling (\texttt{pandas.DataFrame.ffill}) so that medians and percentile statistics remain well-defined.
For the leftmost two columns, RMS denotes the root-mean-square, $\mathrm{RMS}(x) = \sqrt{\tfrac{1}{k}\sum_{i=1}^k x_i^2}$ for $x \in \R^k$, which normalizes the $\ell_2$ norm to be dimension-invariant. RMSE denotes the root-mean-square error, $\mathrm{RMSE}(x,y) = \mathrm{RMS}(x-y)$ for $x,y \in \R^k$. For the rightmost scatter plot, we aggregate all state-action pairs from these rollouts and show the relation between estimated uncertainty $\unc(\hat{s}_t, a_t)$ and the next-state error $\mathrm{RMSE}(\hat{s}_{t+1}, s_{t+1})$.

\textbf{LayerNorm significantly suppresses worst-case (95\%) errors across all evaluation setups.} Across the $6\times 3 = 18$ evaluation setups shown in \autoref{fig:compound_app1}--\autoref{fig:compound_app3}, LayerNorm consistently prevents compounding state-error and reward-bias explosion (1st and 3rd columns) by stabilizing the predicted state norm (2nd column). In contrast, without LayerNorm, \textbf{16 of the 18} setups exhibit clear error explosion. We further observe that models trained on medium and medium-expert datasets (with narrower coverage) tend to explode more rapidly than those trained on medium-replay datasets (with broader coverage). The only two non-exploding cases without LayerNorm occur when models trained on medium-replay are evaluated with random action sequences, likely because medium-replay offers the broadest coverage and includes random-action trajectories.

\textbf{LayerNorm also significantly suppresses medium-case errors.} LayerNorm effectively controls medium-case errors across all evaluation setups. In contrast, without LayerNorm, \textbf{12 of the 18} setups exhibit medium-case error explosion (with the remaining 6 showing comparable performance to the LayerNorm variant). Notably, although models trained on random datasets do not explode under random-action rollouts, they \textit{do} explode under higher-quality action sequences when LayerNorm is removed.

\textbf{Uncertainty threshold can safeguard with LayerNorm against compounding error.}  For world ensembles with LayerNorm, we find that the uncertainty threshold $\zeta = 1.0$ (used in our main experiments) reliably separates severe error regions, since $\zeta = 1.0$ approximates the boundary of in-distribution data. While the Spearman's rank coefficient, often used to benchmark uncertainty estimation accuracy~\citep{lu2021revisiting}, is less than $0.6$ in 5 out of the 18 setups, we argue that this metric is less critical for Bayesian RL. Unlike approaches that penalize with uncertainties at every step, we only use uncertainty as a binary cutoff for truncation. As a result, our method is inherently more robust to imperfect uncertainty ranking, requiring only that severe errors lie beyond the $\zeta=1.0$ boundary.

\begin{figure}[h]
    \centering
    \includegraphics[width=0.9\linewidth]{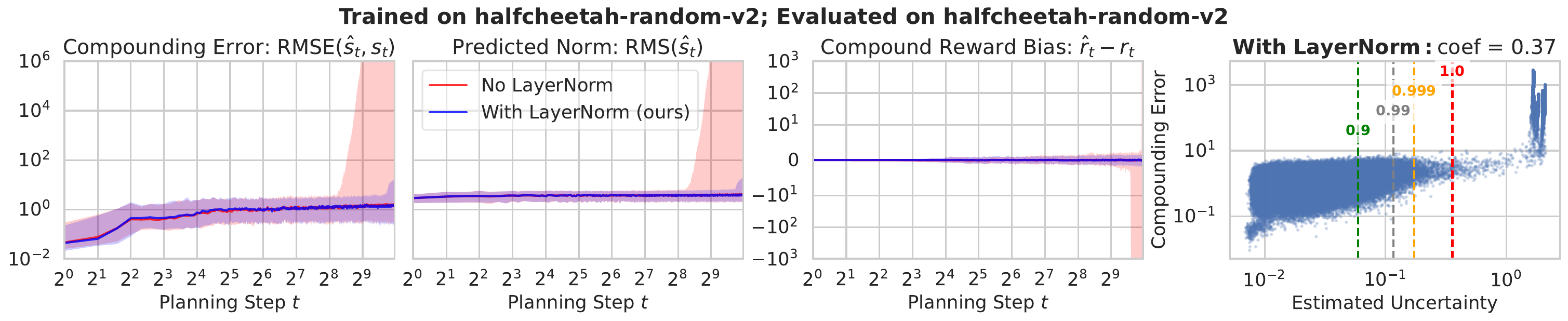}
    \includegraphics[width=0.9\linewidth]{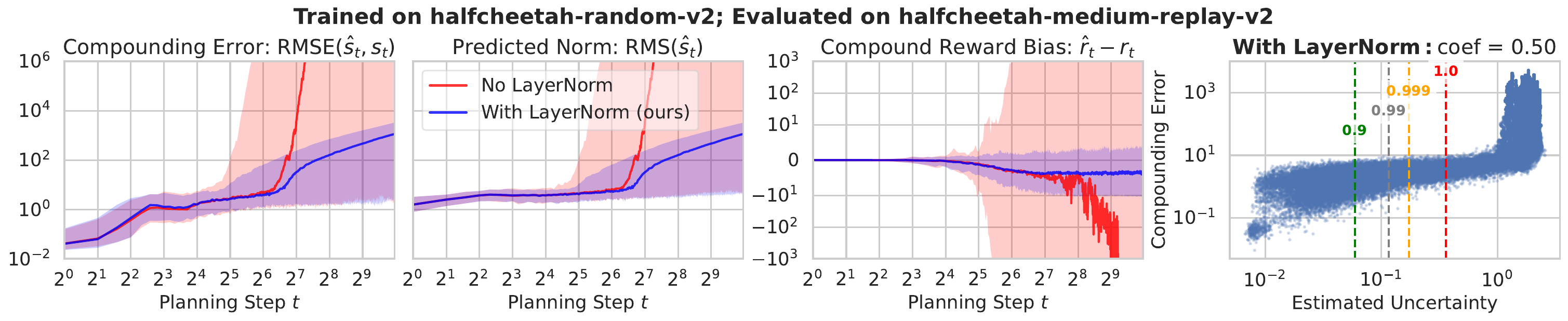}
    \includegraphics[width=0.9\linewidth]{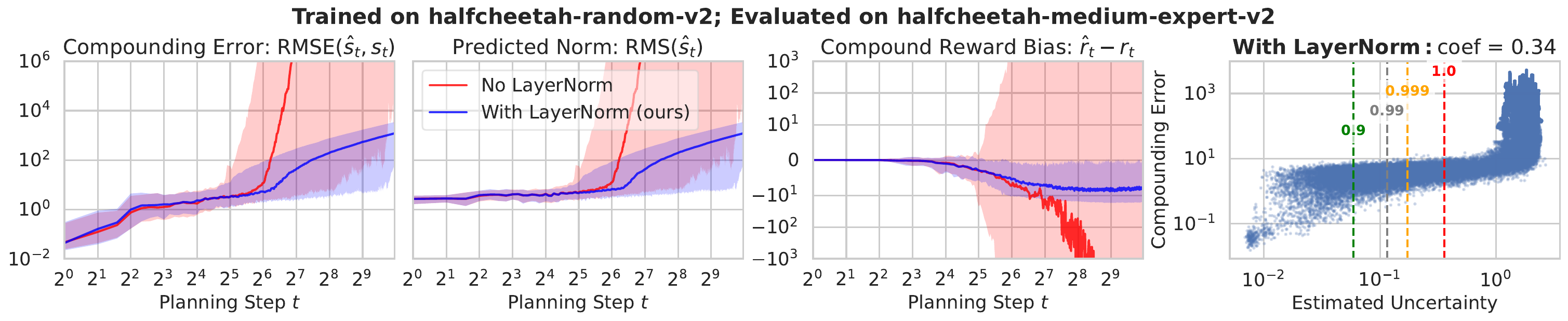}
    \\ \vspace{1em}
    \includegraphics[width=0.9\linewidth]{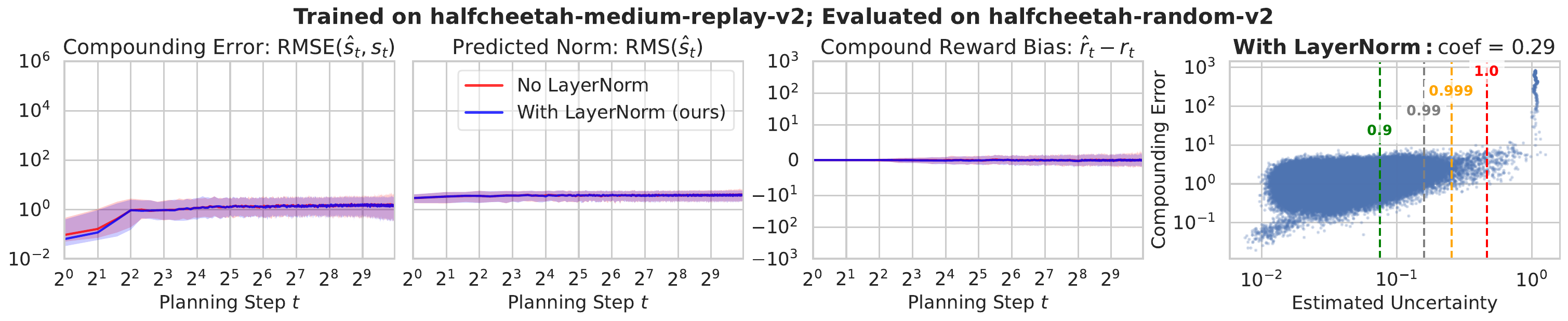}
    \includegraphics[width=0.9\linewidth]{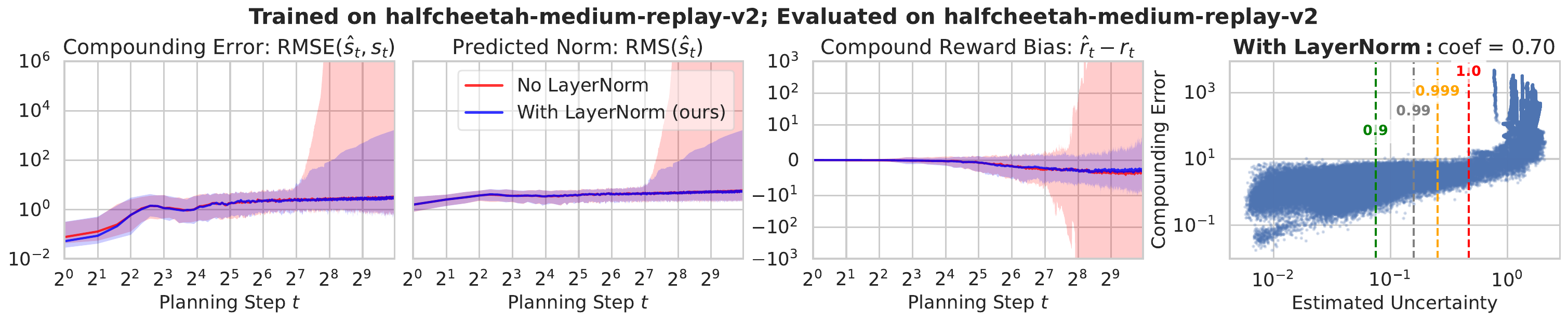}
    \includegraphics[width=0.9\linewidth]{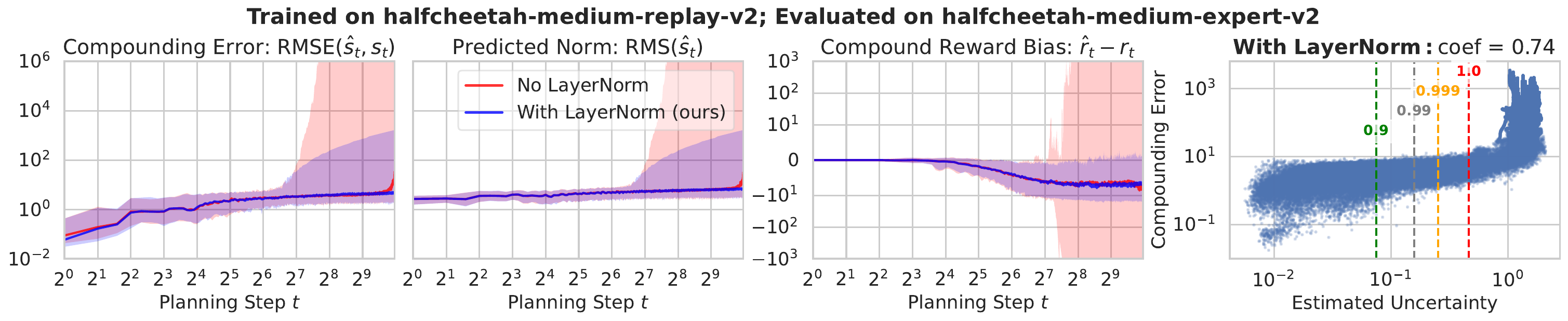}
    \caption{Effect of LayerNorm in world models (part 1 of 3). For each metric, we plot the \textbf{median} (solid line) together with the \textbf{5-95\% percentile band} across 200 rollouts. The rightmost scatter plots show the Spearman's rank coefficients in the with-LayerNorm setting; vertical lines mark uncertainty thresholds $\zeta \in\{0.9, 0.99,0.999, 1.0\}$.}
    \label{fig:compound_app1}
\end{figure}

\begin{figure}[h]
    \centering
    \includegraphics[width=0.9\linewidth]{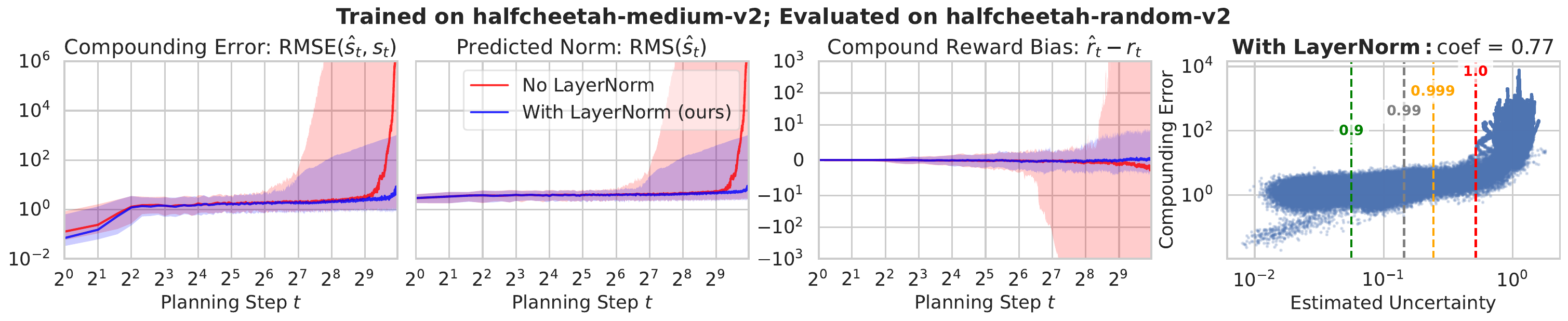}
    \includegraphics[width=0.9\linewidth]{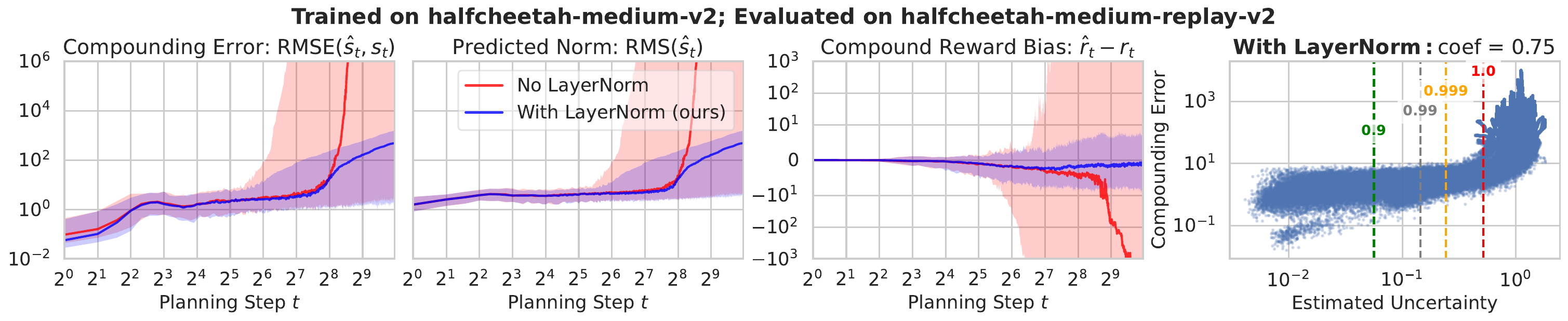}
    \includegraphics[width=0.9\linewidth]{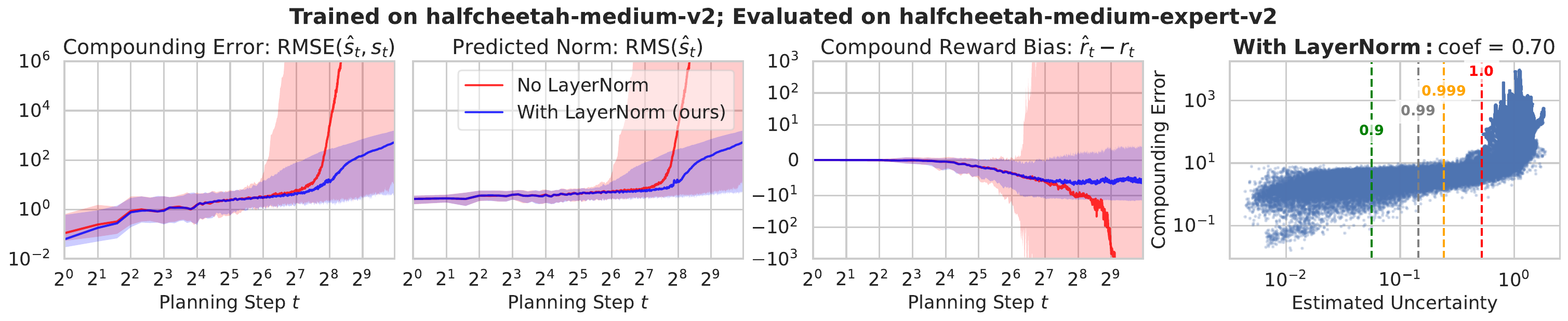}
    \\ \vspace{1em}
    \includegraphics[width=0.9\linewidth]{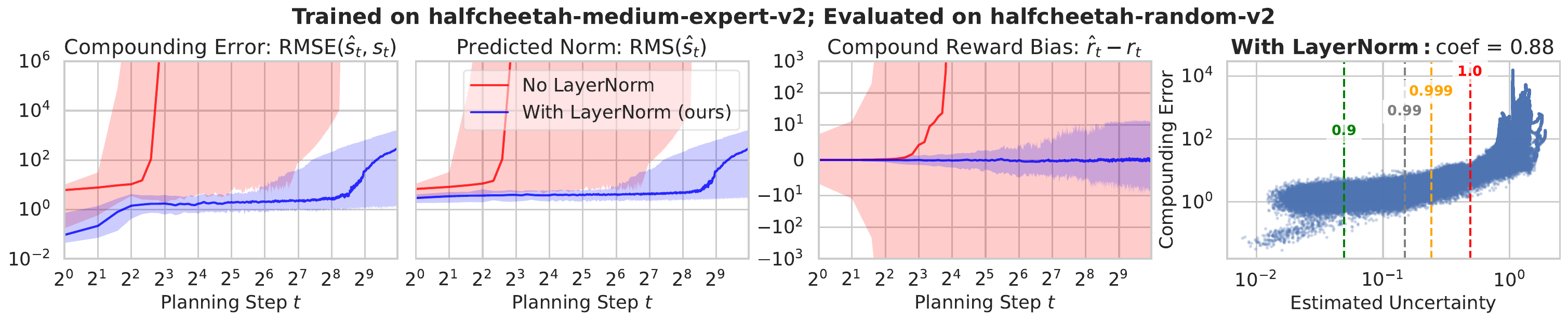}
    \includegraphics[width=0.9\linewidth]{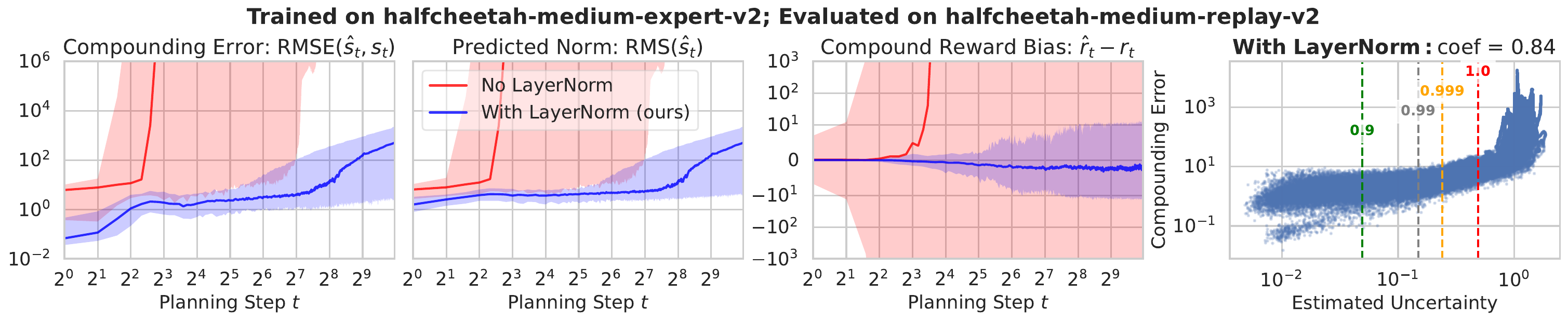}
    \includegraphics[width=0.9\linewidth]{figs/compound/halfcheetah-medium-expert-v2/halfcheetah-medium-expert-v2/main.pdf}
    \caption{Effect of LayerNorm in world models (part 2 of 3).}
    \label{fig:compound_app2}
\end{figure}

\begin{figure}[h]
    \centering
    \includegraphics[width=0.9\linewidth]{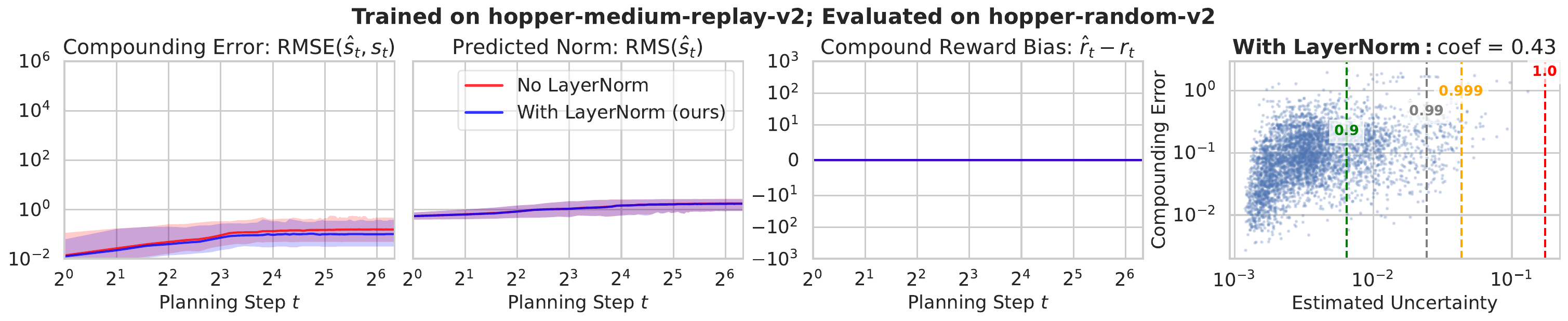}
    \includegraphics[width=0.9\linewidth]{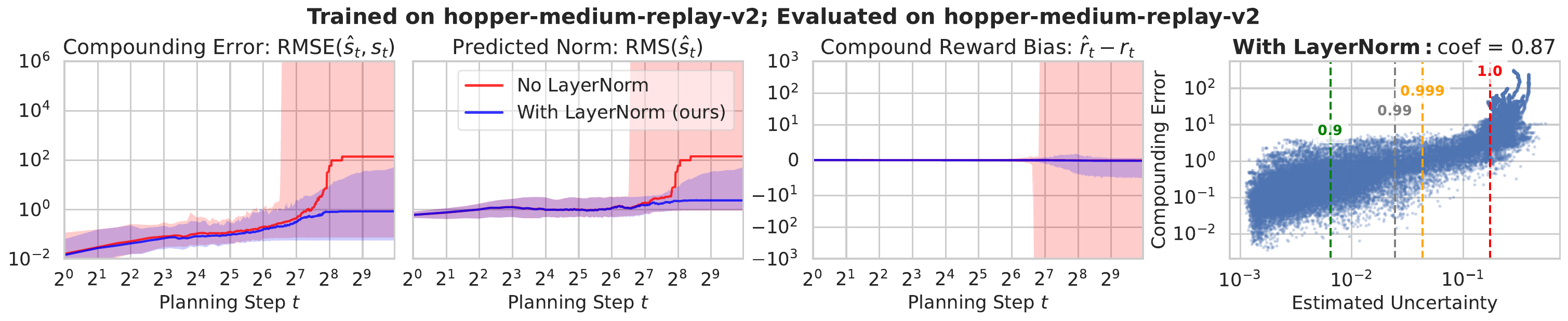}
    \includegraphics[width=0.9\linewidth]{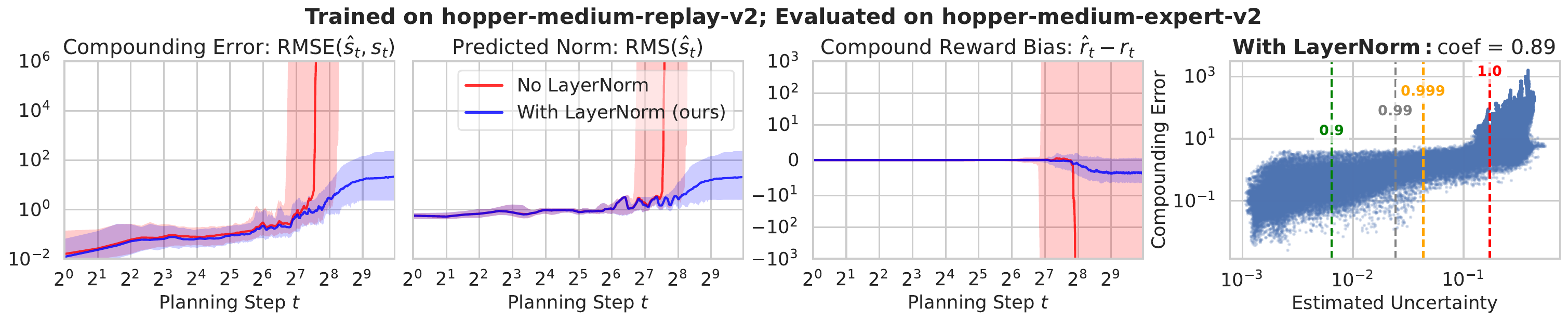}
    \\ \vspace{1em}
    \includegraphics[width=0.9\linewidth]{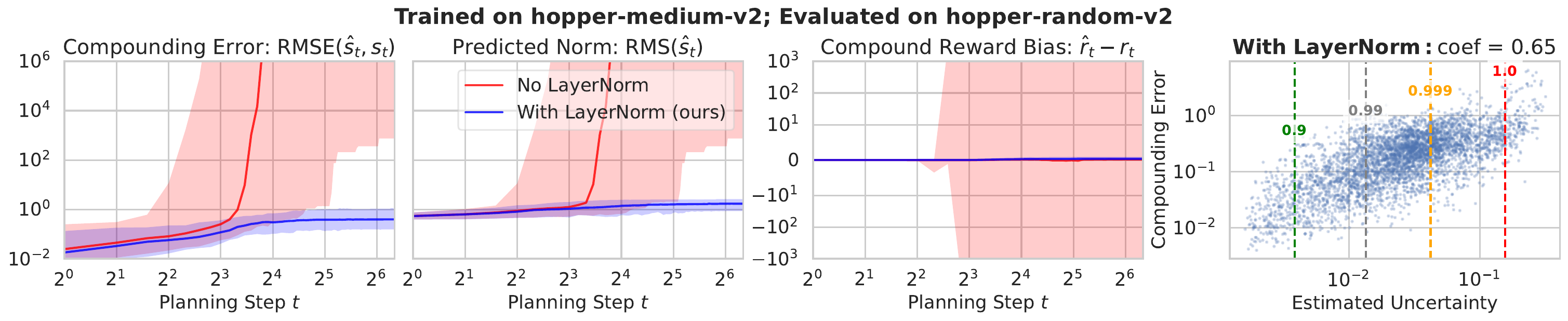}
    \includegraphics[width=0.9\linewidth]{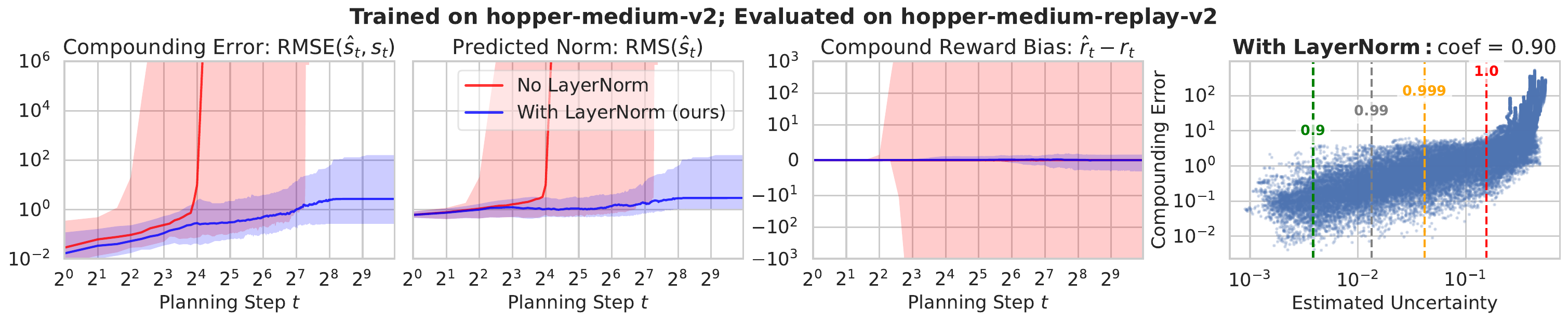}
    \includegraphics[width=0.9\linewidth]{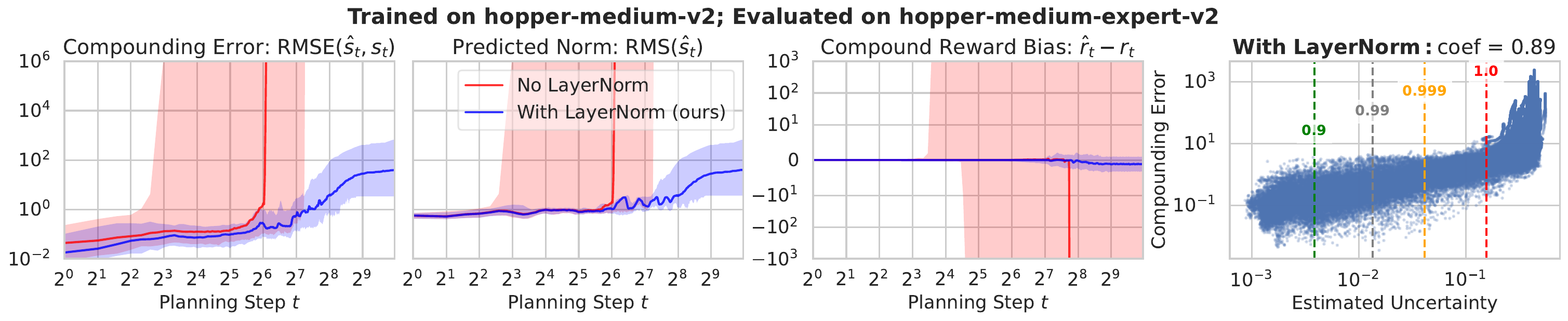}
    \caption{Effect of LayerNorm in world models (part 3 of 3).}
    \label{fig:compound_app3}
\end{figure}

\clearpage
\subsection{Sensitivity and Ablation Results}
\label{app:ablation}
\vspace{1em}

\begingroup
\setlength{\tabcolsep}{4pt}       %
\renewcommand{\arraystretch}{1.2} %
\begin{table}[H]
\caption{
\textbf{Sensitivity and ablation results per dataset.} 
The \texthl{hlcolor}{highlighted} setting ($N{=}100$, $\lambda{=}0.0$, $\zeta{=}1.0$, using the entire history as agent input) is the main result. Ablations vary one hyperparameter at a time, except for the Markov agent, where we sweep the real data ratio $\kappa$ for a fair comparison. 
Red shading shows \textbf{degradation} level: \texthl{red!10}{light} (3–10), \texthl{red!35}{medium} (10–30), \texthl{red!60}{dark} ($>$30).
Green shading shows \textbf{improvement} level: \texthl{green!10}{light} (3–10), \texthl{green!35}{medium} (10–30), \texthl{green!60}{dark} ($>$30). 
\autoref{tab:ablation} is a summary.
}
\label{tab:ablation_full}
\centering
\resizebox{\columnwidth}{!}{%
\begin{tabular}{r|rrr|rrrrr|rrrr|rr}
\toprule
& \multicolumn{3}{c|}{Ensemble size $N$} 
& \multicolumn{5}{c|}{Unc.\ penalty coef.\ $\lambda$} 
& \multicolumn{4}{c|}{Truncation threshold $\zeta$} & \multicolumn{2}{c}{Agent input} \\
\cmidrule(lr){2-4} \cmidrule(lr){5-9} \cmidrule(lr){10-13} \cmidrule(lr){14-15}
Dataset
& \highlight{100} & 20 & 5
& \highlight{0.0} & 0.04 & 0.2 & 1.0 & 5.0
& \highlight{1.0} & 0.999 & 0.99 & 0.9 
& \highlight{Hist.} & Mark. \\
\midrule
hc-random-v2
& 37.0 & 39.7 & 38.2
& 37.0 & 35.6 & \degL{32.7} & \degM{19.3} & \degD{1.2}
& 37.0 & 35.0 & \degL{31.2} & \degL{30.7}
& 37.0 & \impL{44.4} \\
hp-random-v2
& 24.5 & \degL{15.1} & \degM{13.3} 
& 24.5 & \impM{48.2} & \impM{40.1} & \impL{31.6} & \degL{18.8} 
& 24.5 & \degM{8.0} & \degM{7.1} & \degM{8.7} 
& 24.5 & 21.7 \\
wk-random-v2
& 34.1 & \degM{20.5} & \degM{8.2}
& 34.1 & 33.0 & 34.2 & \degM{23.3} & \degD{0.0}
& 34.1 & \degM{7.1} & \degM{5.3} & \degM{6.0} 
& 34.1 & 33.4 \\
hc-med-rep-v2
& 72.1 & \degL{68.6} & \degL{66.7}
& 72.1 & 72.0 & 70.1 & \degL{63.5} & \degM{52.1}
& 72.1 & \degL{67.8} & \degL{68.6} & \degD{37.7} 
& 72.1 & \impL{76.9} \\
hp-med-rep-v2
& 110.6 & \degM{81.8} & \degD{75.2}
& 110.6 & 110.9 & \degM{97.6} & \degM{95.3} & \degD{28.8}
& 110.6 & \degD{79.8} & \degM{89.7} & \degD{3.2} 
& 110.6 & \degD{47.2} \\
wk-med-rep-v2
& 99.3 & \degL{91.3} & 97.8
& 99.3 & \degM{87.4} & \degL{93.5} & \degM{86.7} & \degM{72.8}
& 99.3 & \degL{96.0} & \degL{92.3} & \degD{7.4} 
& 99.3 & \impM{112.8} \\
hc-medium-v2
& 78.6 & \degL{73.7} & \degL{74.2}
& 78.6 & 77.5 & \degL{70.7} & \degM{59.8} & \degM{54.6}
& 78.6 & \degL{74.5} & \degL{74.0} & \degM{61.8} 
& 78.6 & 78.2 \\
hp-medium-v2
& 54.2 & \degM{37.1} & \degL{48.8}
& 54.2 & 52.2 & \impD{105.8} & \impM{74.0} & \impM{64.8}
& 54.2 & \degM{40.3} & \degD{2.3} & \degD{2.2}
& 54.2 & \degL{47.0} \\
wk-medium-v2
& 106.4 & \degL{103.1} & \degD{55.8}
& 106.4 & \degL{96.6} & \degM{77.9} & \degM{93.6} & \degM{81.9}
& 106.4 & \degM{88.9} & \degD{56.5} & \degD{0.2} 
& 106.4 & \impL{112.4} \\
hc-med-exp-v2
& 109.4 & 107.3 & \degM{97.8}
& 109.4 & 107.7 & \impL{112.7} & 110.3 & \degM{96.4}
& 109.4 & 112.3 & 109.6 & 109.6 
& 109.4 & 111.9 \\
hp-med-exp-v2
& 114.8 & \degM{100.2} & \degM{96.8}
& 114.8 & 114.4 & 113.5 & \degL{110.3} & \degL{110.1} 
& 114.8 & \degL{106.5} & \degD{2.2} & \degD{1.8} 
& 114.8 & \degM{101.3} \\
wk-med-exp-v2
& 120.6 & 120.6 & 118.5
& 120.6 & 121.5 & \degL{116.2} & \degM{109.3} & \degM{107.4}
& 120.6 & 118.5 & \degD{2.0} & \degD{1.1} 
& 120.6 & 122.2 \\
\midrule
hc-v3-Low
& 53.0 & 51.4 & 52.2
& 53.0 & 52.3 & \degL{47.3} & \degM{37.8} & \degM{26.3}
& 53.0 & 53.3 & 52.7 & 50.1 
& 53.0 & 55.4 \\
hp-v3-Low
& 30.3 & \degL{26.6} & 29.3
& 30.3 & 30.0 & \degL{24.3} & \degM{15.5} & \degM{12.7}
& 30.3 & \degL{22.5} & \degM{1.6} & \degM{1.1} 
& 30.3 & 29.4 \\
wk-v3-Low
& 43.9 & 41.9 & \degL{35.4}
& 43.9 & 46.1 & 45.7 & \impL{53.0} & \impL{53.2}
& 43.9 & \degL{39.8} & \degM{22.5} & \degD{0.2} 
& 43.9 & \impM{57.4} \\
hc-v3-Med
& 81.1 & 80.5 & \degL{77.3}
& 81.1 & 81.2 & 79.3 & \degM{68.1} & \degM{56.6}
& 81.1 & 79.5 & 79.9 & \degM{70.4}
& 81.1 & 82.5 \\
hp-v3-Med
& 95.7 & \degL{88.6} & 94.1
& 95.7 & \degL{89.3} & \degM{81.6} & \degM{72.1} & \degD{33.4}
& 95.7 & \degL{86.8} & \degD{35.2} & \degD{0.6}
& 95.7 & 94.9 \\
wk-v3-Med
& 50.5 & \degM{34.4} & \degM{39.9}
& 50.5 & \degL{43.9} & 50.2 & \degL{44.8} & \degL{40.9}
& 50.5 & \impL{55.5} & \degM{23.1} & \degD{0.7} 
& 50.5 & \degL{47.1} \\
hc-v3-High
& 68.3 & \degL{59.9} & \degM{53.8}
& 68.3 & 65.9 & 67.5 & 71.1 & \degM{54.1}
& 68.3 & \degM{56.7} & \impM{80.4} & \degD{25.2}
& 68.3 & \impM{81.5} \\
hp-v3-High
& 96.8 & \impL{100.4} & 99.3
& 96.8 & \degL{89.6} & \degM{75.3} & \degL{90.2} & \degD{46.4}
& 96.8 & \degM{85.0} & \degD{26.1} & \degD{1.6} 
& 96.8 & \degM{80.9} \\
wk-v3-High
& 62.7 & \degL{59.1} & \degL{57.0}
& 62.7 & \degL{55.3} & \impL{67.2} & \impL{72.2} & \degL{58.5}
& 62.7 & \degM{50.0} & \degD{4.8} & \degD{0.0}
& 62.7 & \impL{72.4} \\
\midrule
pen-human-v1
& 20.8 & 19.3 & \degL{13.7}
& 20.8 & \impL{25.8} & \impL{24.0} & \impM{35.9} & \impM{34.8}
& 20.8 & \impL{27.2} & \degL{16.8} & \degM{5.0} 
& 20.8 & \impL{28.0} \\
pen-cloned-v1
& 91.3 & \degM{63.4} & \degM{75.6}
& 91.3 & \degM{68.2} & \degM{76.0} & \degM{77.2} & \degM{67.2}
& 91.3 & \degD{15.7} & \degD{10.6} & \degD{0.2} 
& 91.3 & \degM{75.6}  \\
hammer-cloned-v1
& 14.4 & \degL{7.4} & \degL{8.6}
& 14.4 & \degL{7.0} & \degM{2.7} & \degM{1.4} & \degM{0.2}
& 14.4 & \degL{10.5} & \impL{20.5} & \degM{0.6} 
& 14.4 & \degL{5.6} \\
\midrule
umaze-v2 
& 66.1 & \impL{71.1} & 65.3 
& 66.1 & \degD{19} & \degD{9.3} & \degD{16.7} & \degD{28.4}
& 66.1 & \degM{40.6} & \degL{62.6} & \impM{86.0} 
& 66.1 & \degD{0.0} \\
umaze-diverse-v2 
& 74.4 & \degM{52.5} & \degD{37.8} 
& 74.4 & \degD{0.8} & \degD{3.3} & \degD{0.2} & \degD{19.8}
& 74.4 & 73.0 & \degM{63.3} & \degD{40.4} 
& 74.4 & \degD{5.2} \\ 
medium-play-v2 
& 12.8 & \degM{2.2} & \degM{0.8} 
& 12.8 & \degM{0.0} & \degM{0.0} & \degM{0.0} & \degM{0.0}
& 12.8 & \impL{18.2} & \degL{4.5} & \degM{0.0} 
& 12.8 & \degM{0.0} \\
medium-diverse-v2 
& 19.4 & \degM{6.8} & \degL{11.3} 
& 19.4 &  \degM{0.0} & \degM{0.0} & \degM{0.0} & \degM{0.0}
& 19.4 & \degL{11.4} & \impL{23.1} & \degM{4.5}
& 19.4 & \degM{0.0} \\
\bottomrule
\end{tabular}
}
\end{table}
\endgroup

\textbf{Ablation: Using a Markov agent in \algo works well on most locomotion tasks, but fails in AntMaze.}
Popular offline RL methods employ Markov agents, while \algo uses history-dependent agents due to its epistemic POMDP formulation. To assess the practical importance of memory, we replace the history-dependent agent in \algo with a Markov one, sweep over the same real-data-ratio range, and report the best results in \autoref{tab:ablation} and \autoref{tab:ablation_full}.
The Markov version performs similarly to the history-dependent one on most \textit{locomotion} tasks, especially on the NeoRL benchmark (64.7 $\to$ 66.8). However, it suffers severe degradation on hopper-medium-replay-v2 (110.6 $\to$ 47.2) and on all AntMaze tasks (43.2 $\to$ 1.3). 
This pattern suggests that in locomotion domains, epistemic uncertainty is relatively mild, so using a single observation may often infer the model index, making memory less critical. As a result, \algo with a Markov agent remains a strong baseline in these settings. In contrast, AntMaze has high epistemic uncertainty, especially about the maze layout which is crucial to navigation, so memory is needed to accumulate information over time.

\begingroup
\setlength{\tabcolsep}{4pt}
\renewcommand{\arraystretch}{1.2}
\begin{table}[H]
\vspace{1em}
\caption{Ablation results on using \textbf{fixed} horizon lengths $H \in \{256, 10, 3\}$ instead of adaptive long horizon (\texthl{hlcolor}{highlighted}). For the Adroit domain which has maximum episode steps $\le 200$, we use $H=32$ to replace $H=256$. 
Red shading shows \textbf{degradation} level: \texthl{red!10}{light} (3–10), \texthl{red!35}{medium} (10–30), \texthl{red!60}{dark} ($>$30).
Green shading shows \textbf{improvement} level: \texthl{green!10}{light} (3–10), \texthl{green!35}{medium} (10–30), \texthl{green!60}{dark} ($>$30).}
\label{tab:horizon_ablation}
\centering
\resizebox{0.5\columnwidth}{!}{%
\begin{tabular}{r|rrrr}
\toprule
Dataset & \highlight{Adaptive $H$} & $H=256$ & $H=10$ & $H=3$ \\
\midrule
hc-random-v2    & 37.0 & \impL{40.6} & \degM{27.8} & \degM{24.4} \\
hp-random-v2    & 24.5 & \degM{11.7} & \degM{8.5}  & \degM{8.2}  \\
wk-random-v2    & 34.1 & \degM{16.9} & \degD{4.1}  & \degM{8.3}  \\
hc-med-rep-v2   & 72.1 & 73.1 & \degM{54.9} & \degD{4.0}  \\
hp-med-rep-v2   & 110.6& \degM{95.1} & \degD{3.8}  & \degD{3.8}  \\
wk-med-rep-v2   & 99.3 & \degL{92.6} & \degD{15.7} & \degD{7.4}  \\
hc-medium-v2    & 78.6 & 75.8 & \degM{66.9} & \degD{41.9} \\
hp-medium-v2    & 54.2& \degM{39.4} & \degD{4.8}  & \degD{2.2}  \\
wk-medium-v2    & 106.4& \degD{60.2} & \degD{4.5}  & \degD{2.4}  \\
hc-med-exp-v2   & 109.5& 108.0& \degM{98.7} & \degD{71.2} \\
hp-med-exp-v2   & 114.8& 115.8& \degD{3.2}  & \degD{1.7}  \\
wk-med-exp-v2   & 120.7& 118.4& \degD{4.5}  & \degD{2.2}  \\
\midrule
hc-v3-Low       & 53.1  & 52.5        & 50.7 & \degL{44.0} \\
hp-v3-Low       & 30.3  & \impL{39.4} & \degM{2.6}  & \degM{1.5}  \\
wk-v3-Low       & 43.9  & \degL{35.6} & \degD{4.6}  & \degD{1.0}  \\
hc-v3-Med       & 81.1  & 79.8        & \degL{71.3} & \degD{51.1} \\
hp-v3-Med       & 95.7  & 96.0 & \degD{6.2}  & \degD{0.8}  \\
wk-v3-Med       & 50.5  & 49.7        & \degD{6.0}  & \degD{0.0}  \\
hc-v3-High      & 68.3  & 68.6 & \degL{59.5} & \degD{28.7} \\
hp-v3-High      & 96.8  & 98.4 & \degD{3.2}  & \degD{1.1}  \\
wk-v3-High      & 62.7  & \degM{45.3} & \degD{4.5}  & \degD{2.2}  \\
\midrule
pen-human-v1    & 20.8  & \degM{3.5}  & \degM{8.6}  & \degM{0.2} \\
pen-cloned-v1   & 91.3  & \degL{86.2} & \degD{52.1} & \degD{6.7}  \\
hammer-cloned-v1& 14.4  & \impL{17.8} & \degM{1.1}  & \degM{0.2} \\
\midrule
umaze-v2        & 66.1  & \impL{69.9} & \impL{71.8} & \degD{33.4} \\
umaze-diverse-v2& 74.4  & \degD{39.0} & 73.3        & \degD{24.9} \\
medium-play-v2  & 12.8  & 10.7 & \degM{0.0}  & \degM{0.0}  \\
medium-diverse-v2& 19.4 & \degL{13.2} & \degM{0.8}  & \degM{0.0}  \\
\bottomrule
\end{tabular}
}
\end{table}
\endgroup

\textbf{Ablation: Fixed short horizons collapse performance, and fixed long horizons underperform adaptive horizons.}
As shown in \autoref{tab:horizon_ablation}, fixed short horizons ($H=3,10$), which are common in offline MBRL, fail dramatically in our setting. Their performance drops across most datasets, consistent with our analysis that short rollouts increase reliance on bootstrapping and thus exacerbate value overestimation. Using a fixed long horizon ($H=256$; or $H=32$ in Adroit) alleviates this issue and is often competitive, but still falls short of adaptive truncation overall. These results confirm that fixed short horizons are harmful to non-conservative RL, and adaptive long horizons are better than fixed long horizons.

\begingroup
\setlength{\tabcolsep}{4pt}
\renewcommand{\arraystretch}{1.2}
\begin{table}[H]
\vspace{1em}
\caption{Ablation results on using \textbf{SAC}~\citep{haarnoja2018soft} instead of REDQ~\citep{chen2021randomized} as the backbone actor-critic.
The \texthl{hlcolor}{highlighted} setting (REDQ, $\zeta{=}1.0$) is the main result.
Red shading shows \textbf{degradation} level: \texthl{red!10}{light} (3--10), \texthl{red!35}{medium} (10--30), \texthl{red!60}{dark} ($>$30).
Green shading shows \textbf{improvement} level: \texthl{green!10}{light} (3--10), \texthl{green!35}{medium} (10--30), \texthl{green!60}{dark} ($>$30).}
\label{tab:sac_ablation}
\centering
\resizebox{0.5\columnwidth}{!}{%
\begin{tabular}{r|rrr}
\toprule
Dataset & \highlight{REDQ ($\zeta{=}1.0$)} & SAC ($\zeta{=}1.0$) & SAC ($\zeta{=}0.9$) \\
\midrule
hc-random-v2  & 37.0  & 36.8          & \degL{31.7}  \\
hp-random-v2  & 24.5  & \degL{15.7}   & \degM{8.0}   \\
wk-random-v2  & 34.1  & \degL{25.0}   & \degM{4.3}   \\
hc-med-rep-v2 & 72.1  & 70.4   & \degD{34.2}  \\
hp-med-rep-v2 & 110.6 & \degM{98.9}   & \degD{4.9}   \\
wk-med-rep-v2 & 99.3 & \degD{61.5}   & \degD{20.6}  \\
hc-medium-v2  & 78.6  & 75.9   & \degM{62.9}  \\
hp-medium-v2  & 54.2  & \degM{42.4}   & \degD{2.1}   \\
wk-medium-v2  & 106.4 & \degM{80.2}   & \degD{0.9}   \\
hc-med-exp-v2 & 109.5 & 108.8         & 111.0 \\
hp-med-exp-v2 & 114.8 & \degL{109.3}  & \degD{2.0}   \\
wk-med-exp-v2 & 120.7 & 118.9  & \degD{0.7}   \\
\bottomrule
\end{tabular}
}
\end{table}
\endgroup

\textbf{Ablation: The dominant factor is horizon length, rather than REDQ vs. SAC.}
\autoref{tab:sac_ablation} shows that replacing REDQ with standard SAC only causes a moderate drop in performance when long horizons are retained ($\zeta{=}1.0$), whereas enforcing short horizons with SAC ($\zeta{=}0.9$) leads to a dramatic collapse on most datasets. This indicates that the main gain of \algo does not come from the specific choice of REDQ, but from enabling sufficiently long rollouts. REDQ still provides a useful stabilization benefit, but it is secondary compared with the effect of horizon length.

\begin{figure}[h]
\vspace{1em}
    \centering
    \includegraphics[width=0.28\linewidth]{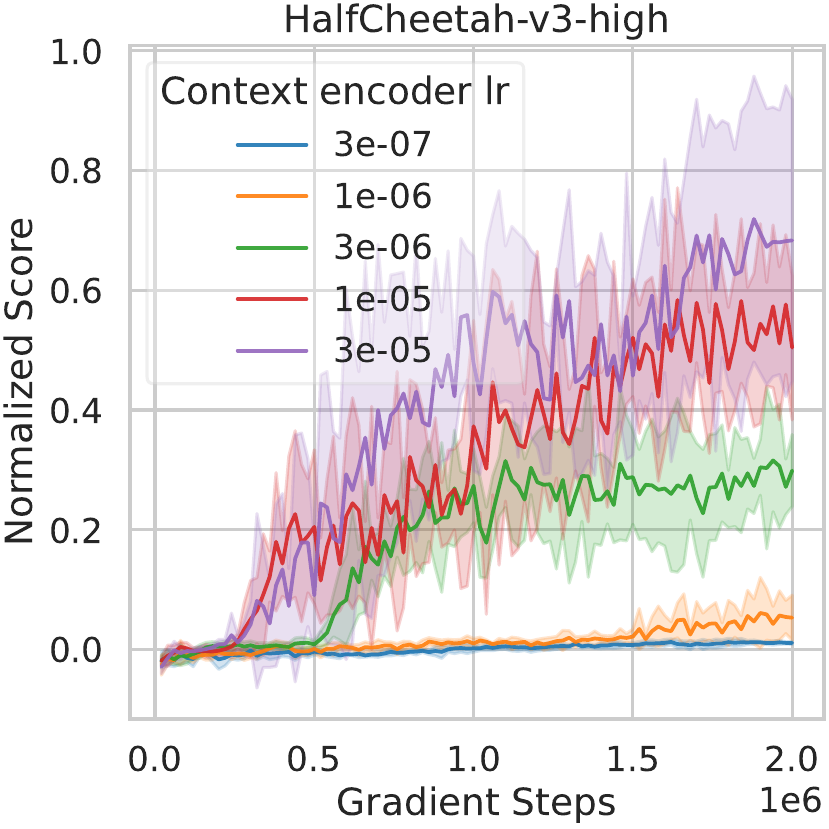}
    \includegraphics[width=0.28\linewidth]{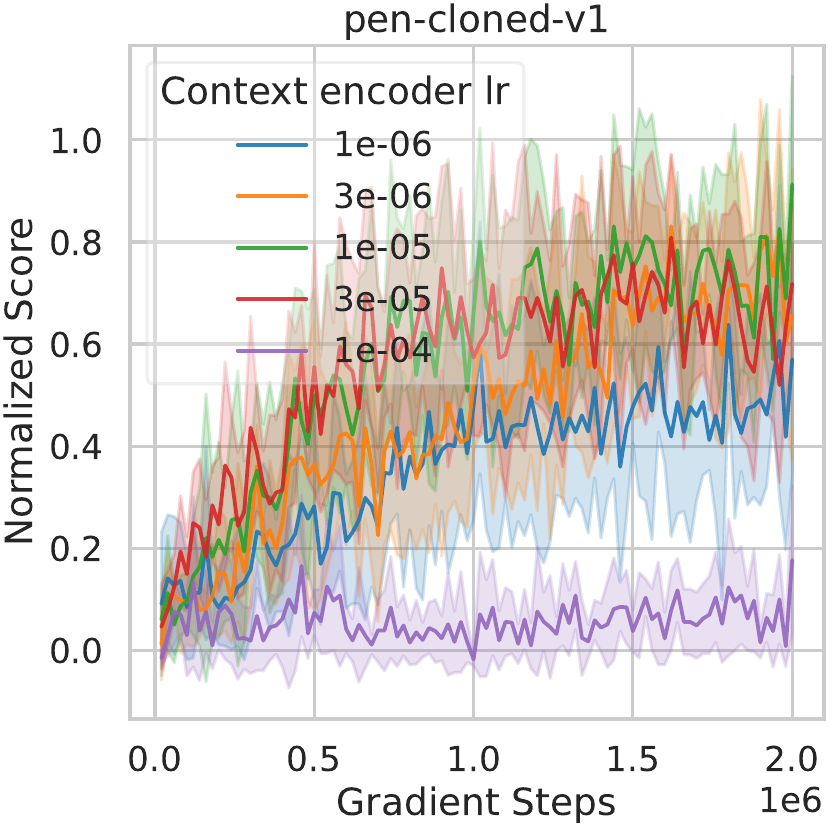}
    \includegraphics[width=0.28\linewidth]{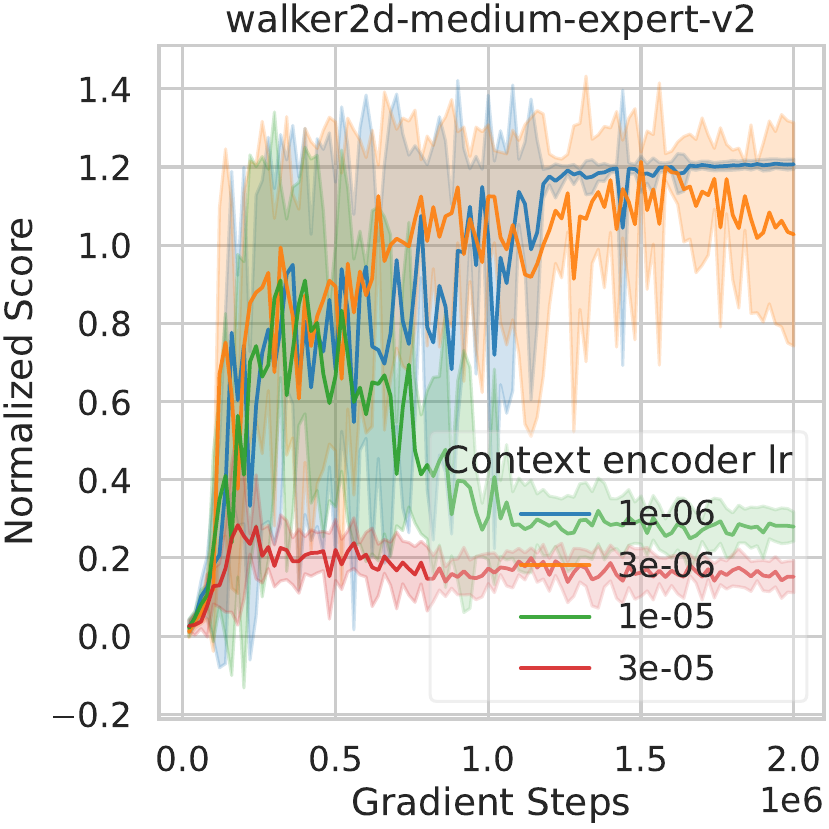}
    \caption{Selective learning curves on datasets where performance is \emph{sensitive} to the \textbf{context encoder learning rate}, favoring high (left), medium (middle), and low (right) values.}
    \vspace{1em}
    \label{fig:enclr}
\end{figure}

\begin{figure}[h]
\vspace{1em}
    \centering
    \includegraphics[width=0.28\linewidth]{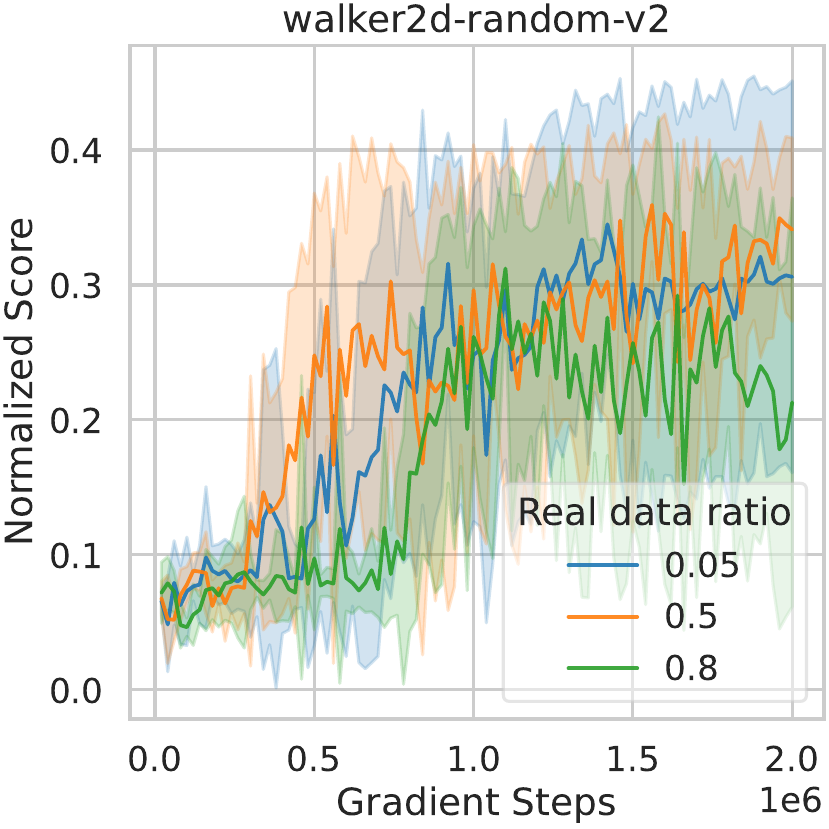}
    \includegraphics[width=0.28\linewidth]{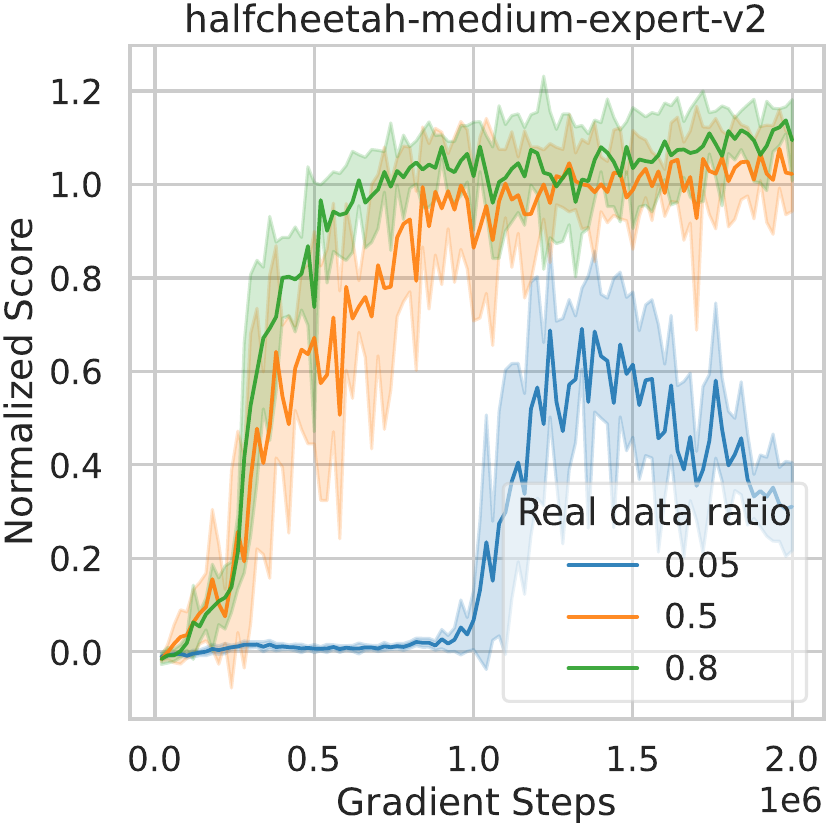}
    \includegraphics[width=0.28\linewidth]{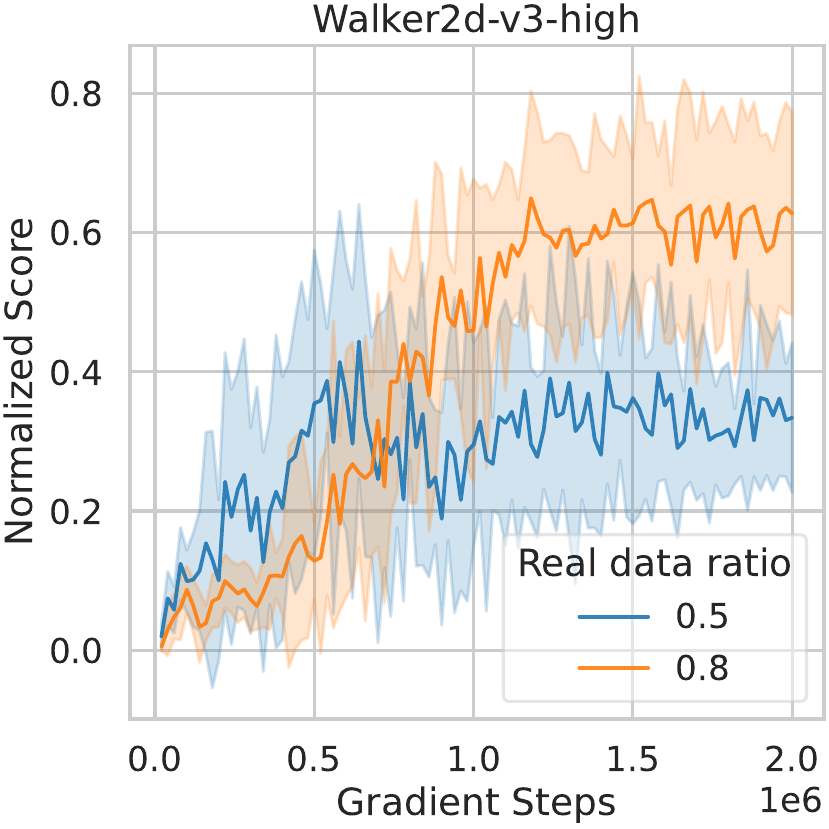}
    \caption{Selective learning curves on datasets where performance is \emph{sensitive} to the \textbf{real data ratio}.}
    \vspace{1em}
    \label{fig:realw}
\end{figure}

\subsection{Failure-Case Analysis in AntMaze}
\label{app:antmaze}
\vspace{1em}

\begin{figure}[H]
    \centering
    \includegraphics[width=0.17\linewidth]{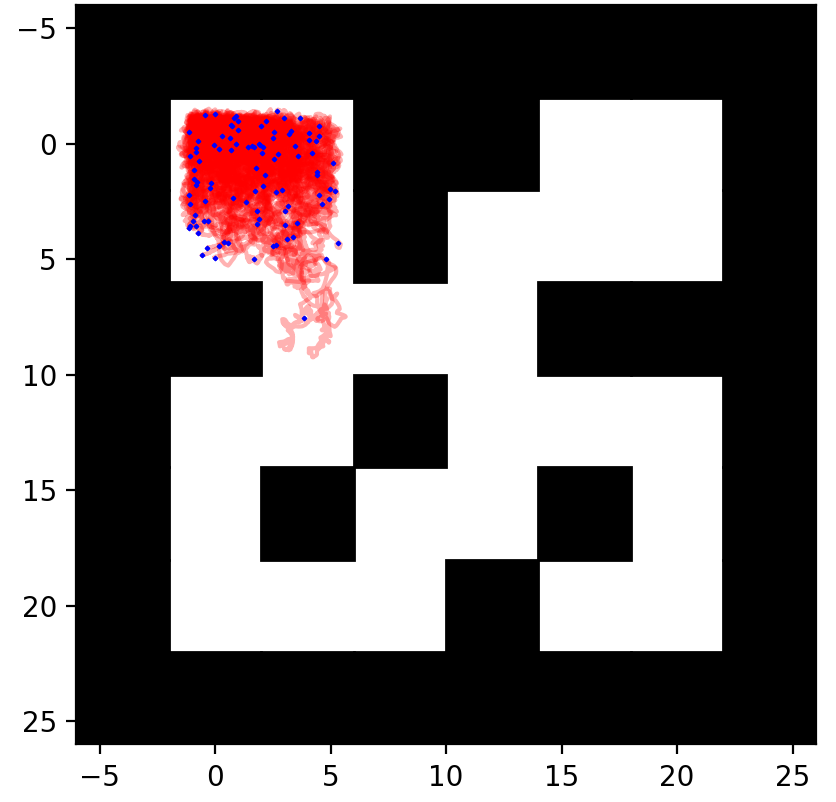}
    \includegraphics[width=0.17\linewidth]{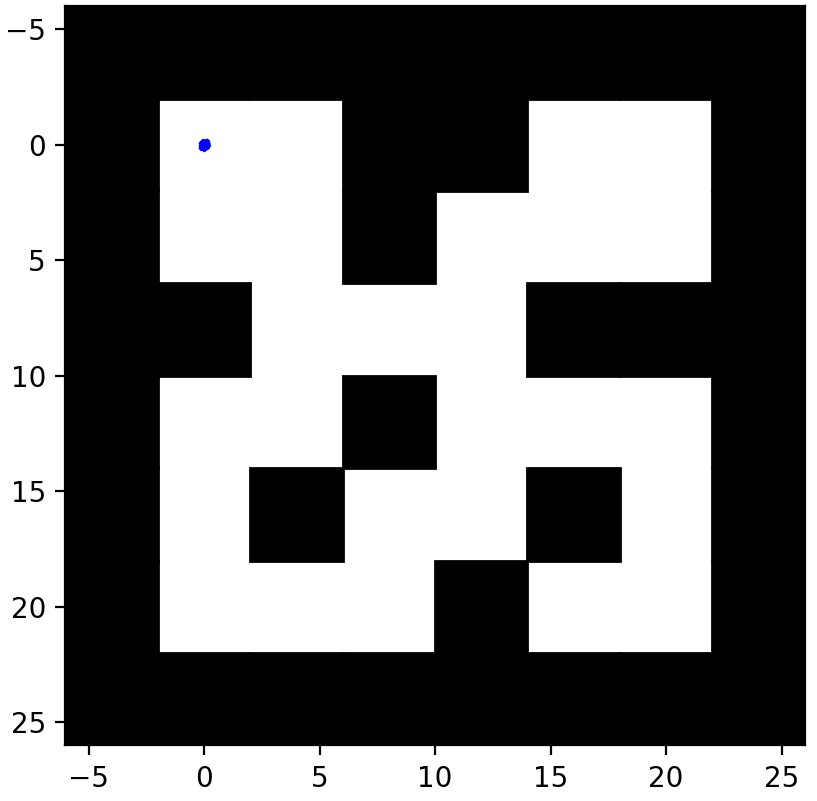}
   \includegraphics[width=0.17\linewidth]{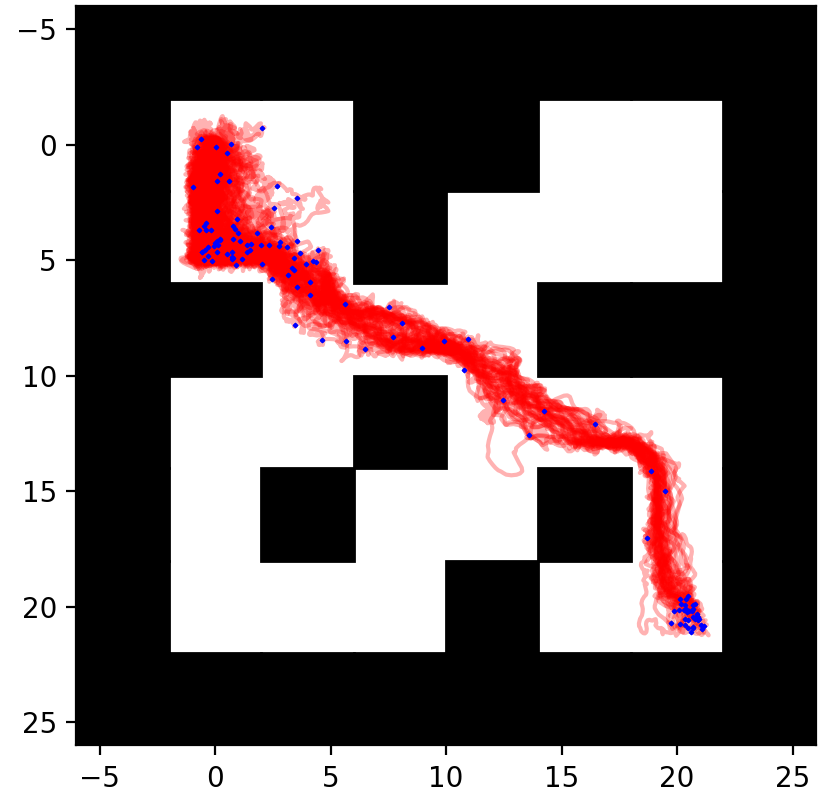}
    \includegraphics[width=0.22\linewidth]{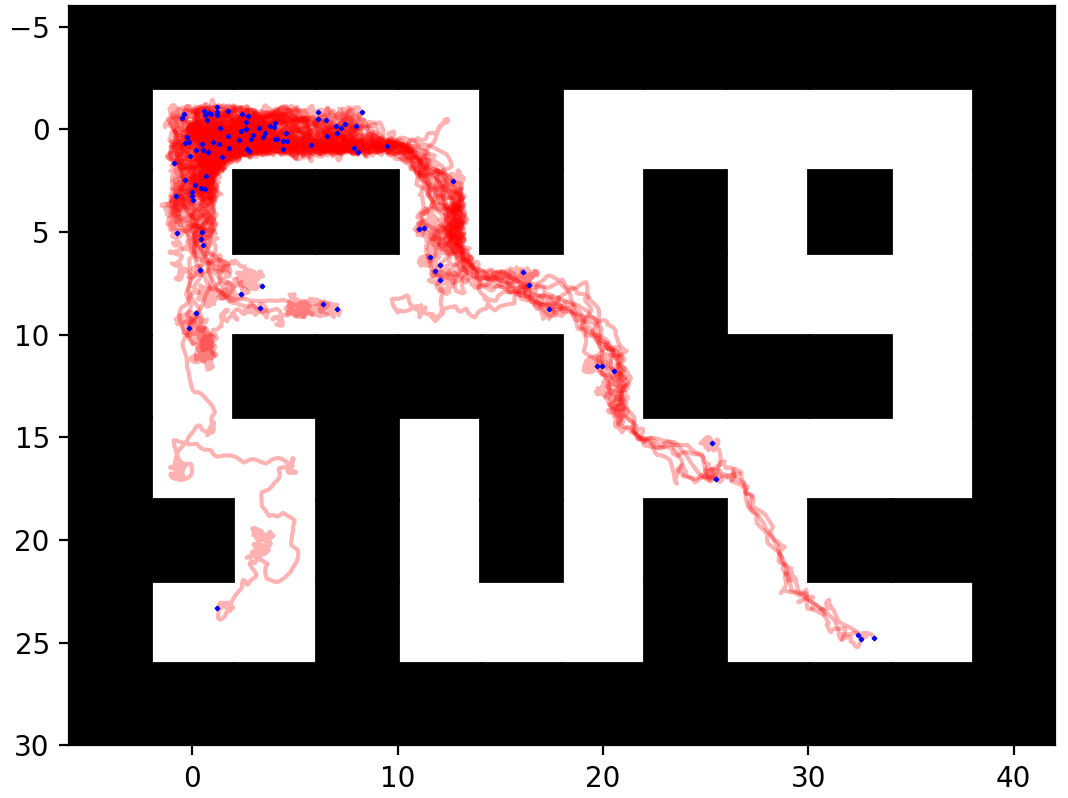}
    \includegraphics[width=0.22\linewidth]{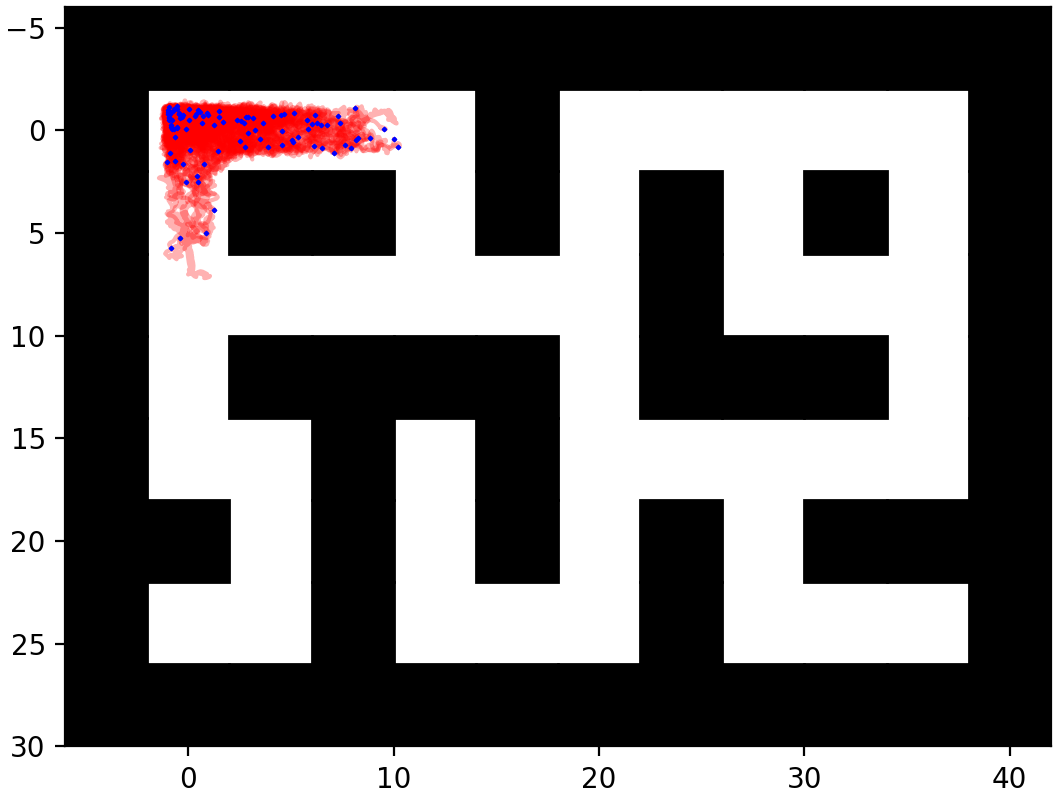}
    \caption{Failure cases in antmaze-medium-play-v2 (left three; different seeds) and antmaze-large-diverse-v2 (right two). The agent starts at the top-left corner and the goal is at the bottom-right. We show 100 evaluation trajectories in red and mark their endpoints in blue.}
    \label{fig:antmaze}
\end{figure}

\autoref{fig:antmaze} illustrates typical failure cases of \algo in AntMaze. Due to sparse rewards, the agent receives no learning signal before reaching the goal, often remaining near the initial region or colliding with nearby walls. Introducing explicit conservatism, such as uncertainty penalties, further degrades performance (\autoref{tab:ablation_full}).
We attribute these failures primarily to limited exploration in large mazes and long-horizon modeling challenges in contact-rich dynamics. Addressing these issues likely requires stronger planning and world modeling components.

\end{document}